\documentclass[letterpaper]{article} % DO NOT CHANGE THIS
\usepackage{aaai2027}  % DO NOT CHANGE THIS
\nocopyright           % suppresses the AAAI copyright slug for the arXiv version
\usepackage[hyphens]{url}  % DO NOT CHANGE THIS
\usepackage{graphicx} % DO NOT CHANGE THIS
\usepackage{natbib}  % DO NOT CHANGE THIS AND DO NOT ADD ANY OPTIONS TO IT
\usepackage{caption} % DO NOT CHANGE THIS AND DO NOT ADD ANY OPTIONS TO IT
\usepackage{algorithm}
\usepackage{algorithmic}

\usepackage{newfloat}
\usepackage{listings}
\DeclareCaptionStyle{ruled}{labelfont=normalfont,labelsep=colon,strut=off} % DO NOT CHANGE THIS
\floatstyle{ruled}
\newfloat{listing}{tb}{lst}{}
\floatname{listing}{Listing}

\usepackage{booktabs}

\usepackage{amsmath,amssymb,amsfonts,mathtools}
\usepackage{bm}
\usepackage{booktabs}
\usepackage{tabularx}
\usepackage{amsthm}
\newtheorem{proposition}{Proposition}
\newtheorem{corollary}{Corollary}
\newtheorem{assumption}{Assumption}
\theoremstyle{remark}

\newtheorem{lemma}{Lemma}

\usepackage{dsfont}

\usepackage{xcolor}

\newcommand{\zb}{\mathbf{z}}
\newcommand{\xb}{\mathbf{x}}

\newcommand{\Dec}{\mathcal{D}_\theta}
\newcommand{\cyc}[1]{\mathcal{C}_{#1}}     % round-trip consistency error
\newcommand{\err}[1]{\mathcal{E}_{#1}}     % true (unobservable) rollout error
\newcommand{\fwd}{\Phi_{+}}
\newcommand{\bwd}{\Phi_{-}}
\newcommand{\cd}{c_d}
\newcommand{\MSE}{\mathrm{MSE}}
\newcommand{\sv}{\mathbf{s}}                          % pair state
\newcommand{\pnorm}[1]{\left\Vert #1 \right\Vert}     % normalized RMS norm
\newcommand{\perr}[1]{\mathcal{E}^{\mathrm{p}}_{#1}}  % pair-level rollout error

\title{Round-Trip Consistency: Bidirectional Diffusion Models \\ Can Predict Their Own Rollout Errors}

\author {
    Alexander Scheinker
}
\affiliations{
    Los Alamos National Laboratory\\
    Los Alamos, NM, 87544, USA\\
    ascheink@lanl.gov
}

\begin{document}
%%%%%%%%%%%%%%%%%%%%%%%%%%%%%%%%%%%%%%%%%%%%%%%%%%%%%%
%%%%%%%%%%%%%%%%%%%%%%%%%%%%%%%%%%%%%%%%%%%%%%%%%%%%%%
%%%%%%%%%%%%%%%%%%%%%%%%%%%%%%%%%%%%%%%%%%%%%%%%%%%%%%

\maketitle

\begin{abstract}
Autoregressive models accumulate error over long rollouts, yet at deployment there is no ground truth to measure it against. We train a single conditional latent diffusion model that steps a dynamical system forward or backward in time via a direction
flag, and show that this bidirectionality supplies a measurement-free test-time error signal: rolling forward $i$ steps and then backward $i$ steps must return the model to its start, so the round-trip discrepancy $\mathcal{C}_i$ is a self-supervised proxy for the unobservable rollout error: no ensembles, no held-out data, no governing equations, for one extra rollout. We validate on compressible magnetohydrodynamics (MHD), an astrophysical turbulent radiative mixing layer, and natural face videos (CelebV-HQ). On held-out MHD trajectories, $\mathcal{C}_i$ ranks rollout error (Spearman $0.91$-$0.98$ at fixed depth; $0.69 \pm 0.16$ within trajectories), and a simple calibrator fit on training
rollouts predicts its magnitude to within $1.14\times$ ($68\%$) and $1.29\times$ ($95\%$) with near-nominal coverage - one nat beyond a depth-only predictor, transferring to all six decoded physical fields. The same signal flags the out-of-distribution Orszag-Tang vortex (AUROC $0.98$; $1.0$ by depth $10$) exactly where sampling-dispersion baselines invert, and it cuts incurred error by $15\%$ at $80\%$ coverage - three times the depth-only baseline. Bidirectional training comes at negative cost, beating direction specialists in both directions, and the backward direction doubles as a fast inverse solver. On LE-PDE-UQ's turbulent Navier-Stokes benchmark, a single bidirectional model reaches accuracy within $1.3\times$ of their ten-model ensemble at a tenth of the training cost, with the best training-free pixel-level calibration. Round-trip consistency turns reversibility into a practical trust signal for generative models.
\end{abstract}

\begin{links}
     \link{Code}{https://github.com/alexscheinker/round-trip-consistency}
\end{links}

% =========================== Introduction =============================
%%%%%%%%%%%%%%%%%%%%%%%%%%%%%%%%%%%%%%%%%%%%%%%%%%%%%%
%%%%%%%%%%%%%%%%%%%%%%%%%%%%%%%%%%%%%%%%%%%%%%%%%%%%%%
%%%%%%%%%%%%%%%%%%%%%%%%%%%%%%%%%%%%%%%%%%%%%%%%%%%%%%
\section{Introduction}
%%%%%%%%%%%%%%%%%%%%%%%%%%%%%%%%%%%%%%%%%%%%%%%%%%%%%%
%%%%%%%%%%%%%%%%%%%%%%%%%%%%%%%%%%%%%%%%%%%%%%%%%%%%%%
%%%%%%%%%%%%%%%%%%%%%%%%%%%%%%%%%%%%%%%%%%%%%%%%%%%%%%
\label{sec:intro}

\begin{figure}[h]
\centering
   \includegraphics[width=1.0\columnwidth]{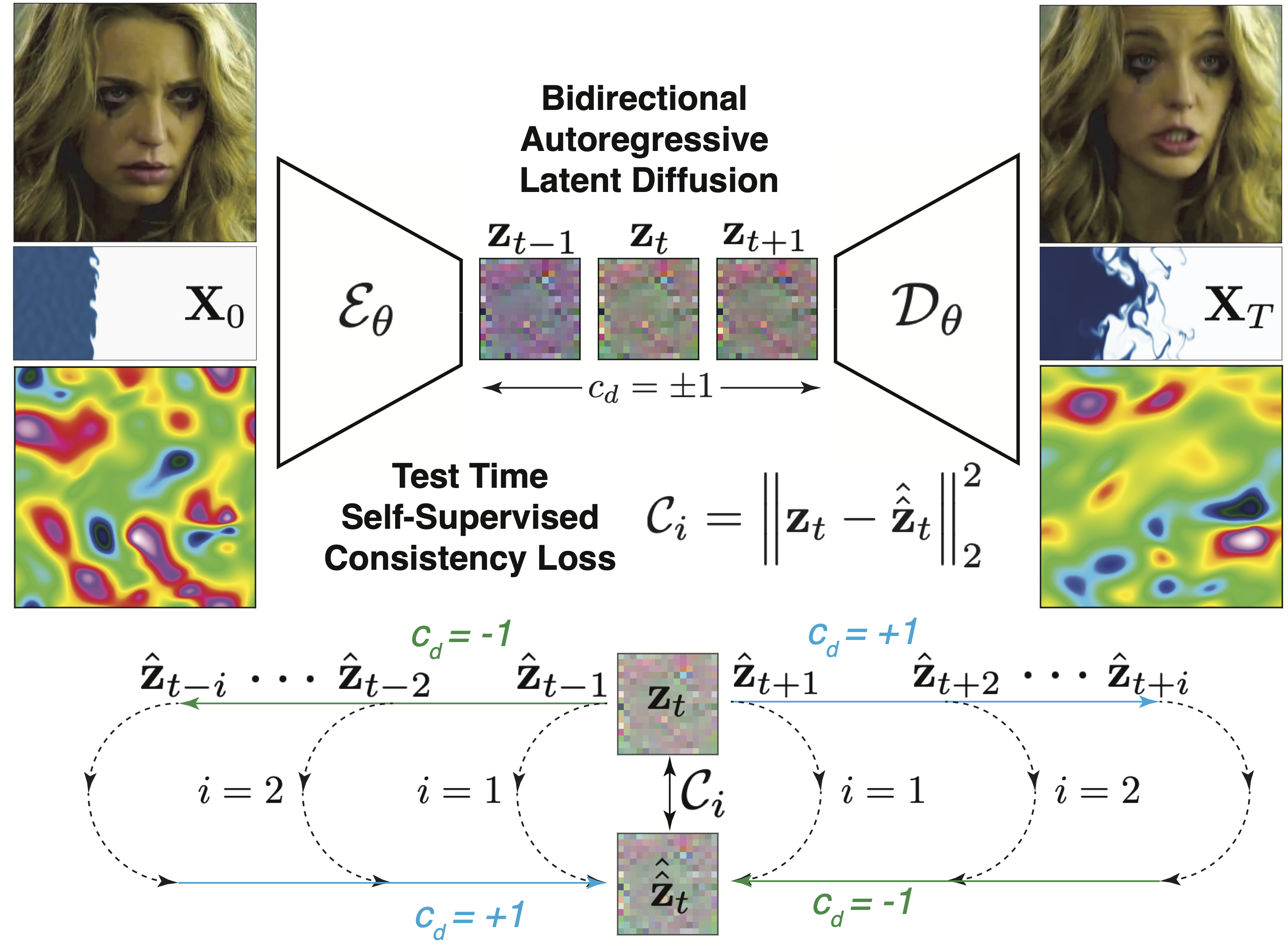}
  \caption{Consistency-based, test-time self-supervised error estimation with a
bidirectional diffusion model: a directional flag $c_d=\pm1$ selects
forward or backward rollout; reversing it rolls back to an estimate
$\hat{\hat{\mathbf{z}}}_t$ of the starting point.}
  \label{fig:Closs}
\end{figure}

Learned surrogates increasingly stand in for numerical solvers across
scientific computing \citep{sanchez2020learning, brandstetter2022message,
li2021fourier, lu2021learning, um2020solver, kochkov2021machine,
raissi2019physics, karniadakis2021physics}, and diffusion models now act
as \emph{probabilistic} simulators, from turbulence to operational
weather prediction \citep{kohl2024benchmarking, cachay2023dyffusion,
price2025gencast, lam2023graphcast}. Deployed autoregressively, these
models feed their own predictions back as inputs: small one-step errors
compound, the input distribution drifts from the training data
\citep{brandstetter2022message, lippe2023pderefiner}, and accuracy
degrades at a rate that varies unpredictably across initial conditions.
In deployment: plasma control \citep{degrave2022magnetic},
accelerator diagnostics \citep{scheinker2024cdvae,scheinker2025latent, scheinker2026phaseflow}, and weather prediction \citep{price2025gencast}, the true state is unknown at test time, so a model cannot say \emph{how far into the future it can still be
trusted}.

Standard remedies quantify predictive \emph{uncertainty} rather
than realized \emph{error}: deep ensembles \citep{lakshminarayanan2017simple,
price2025gencast}, MC dropout \citep{gal2016dropout}, learned variance \citep{wu2024lepdeuq}, spread of repeated samples
\citep{lippe2023pderefiner, shu2024zeroshot}, and conformal intervals
\citep{angelopoulos2023conformal} all measure how much the model
disagrees with itself, not whether the learned dynamics are being
applied accurately to the trajectory at hand. Dispersion-based
signals degrade under distribution shift that
autoregressive rollouts induce \citep{ovadia2019can}.

We take a different route, built on an old and simple idea: check the
consistency of a round trip (Fig.~\ref{fig:Closs}). Forward--backward
agreement flags unreliable optical-flow correspondences
\citep{sundaram2010dense, meister2018unflow}, cycle losses supervise
unpaired translation and temporal correspondence \citep{zhu2017unpaired,
wang2019learning}, and cycle defects quantify uncertainty in inverse
imaging when an analytic forward model is available
\citep{huang2023cycle}. We port this principle to learned dynamics by
training a \emph{single} conditional latent diffusion model with a
direction flag $\cd\in\{+1,-1\}$, so that one network represents both a system's forward evolution and its temporal inverse; the denoiser conditions on the two most recent frames --- the least context that determines a velocity (the check applies verbatim for any $n \ge 1$). Reversibility then becomes a checkable invariant: an accurate model composed with its own inverse is the identity, so the discrepancy left after rolling forward $i$ steps and backward $i$ steps, the round-trip consistency error $\cyc{i}$, is measurable at test time with no ground truth, no ensemble, and no access to the governing equations, for the price of one extra rollout.  We emphasize that no reversibility of the underlying physics is
assumed. All three systems studied are dissipative; the backward
direction is a learned statistical inverse, and the only assumption our
theory requires is placed on that learned map (Sec.~5.1), never on the
dynamics. The check is necessary rather than sufficient because forward and backward errors could in principle cancel, which is precisely why quantifying how faithfully $\cyc{i}$ tracks the true error is the central contribution of this paper.

We validate on three dissipative, nonlinear systems: high-resolution
compressible magnetohydrodynamics (MHD), with the canonical
Orszag--Tang vortex \citep{orszag1979} as a deliberate
out-of-distribution stress test; an astrophysical turbulent radiative
mixing layer from the Well \citep{ohana2024well}; and natural face
videos. Our aim is explicitly \emph{not} state-of-the-art rollout
accuracy. Our context length, architecture scale, and sampling schedules are held fixed and modest throughout, but the question such tuning never answers: whether a model can \emph{know}, at test time and without ground truth, how wrong its rollout is. A trust signal matters most precisely when the surrogate is imperfect, which every deployed surrogate is.

We find that $\cyc{i}$ tracks the true rollout error across depths and trajectories: a simple calibrator fit on training rollouts predicts its magnitude on held-out MHD trajectories to within a factor of $1.14$ (68\%): one nat better than a depth-only predictor, with the gain transferring to all six decoded physical fields --- supports selective prediction ($15\%$ lower incurred error at $80\%$ coverage, three times the depth-only-deferral baseline \citep{geifman2017selective}), and flags the out-of-distribution Orszag-Tang vortex immediately (AUROC $1.0$ by depth $10$; $0.98$ for the trajectory-mean score).
Our contributions are:
\begin{itemize}
  \item \textbf{Bidirectional diffusion dynamics model.} A single
  conditional latent diffusion with direction flag $\cd\in\{+1,-1\}$
  rolls a system forward (surrogate solver) or backward (inverse
  solver) in time with one set of weights.
  \item \textbf{Round-trip consistency as self-supervised UQ.} A
  test-time error proxy $\cyc{i}$, computed by cycling a rollout
  forward and back, requiring no ground truth, no ensembles, and no
  architectural changes: one extra rollout yields a trust /
  early-stop criterion for autoregressive generation.
  \item \textbf{Quantitative validation and theory.} Correlation with
  true error across depths and systems, out-of-distribution detection, selective prediction, comparisons against sampling-based UQ baselines, and a Lipschitz-style certificate relating $\cyc{i}$ to $\err{i}$ (Sec.~\ref{sec:theory}).
\end{itemize}

% ========================== Related Work ==============================
%%%%%%%%%%%%%%%%%%%%%%%%%%%%%%%%%%%%%%%%%%%%%%%%%%%%%%
%%%%%%%%%%%%%%%%%%%%%%%%%%%%%%%%%%%%%%%%%%%%%%%%%%%%%%
%%%%%%%%%%%%%%%%%%%%%%%%%%%%%%%%%%%%%%%%%%%%%%%%%%%%%%
\section{Related Work}
\label{sec:related}

\begin{table*}[t]
\centering
\caption{Test-time trust signals for learned predictors. ``Both dir.\
learned'' asks whether forward and backward maps are both neural (no
analytic operator); ``comp.\ rollout'' whether the signal is defined
over a multi-step autoregressive horizon; ``test-time data'' whether
external observations are required at deployment; the last column
states what the method \emph{does} with its signal.}
\label{tab:positioning}
\footnotesize
\setlength{\tabcolsep}{4pt}
\begin{tabularx}{\textwidth}{llcccX}
\toprule
Method & Signal & \shortstack[c]{Both dir.\\learned} &
\shortstack[c]{Comp.\\rollout} &
\shortstack[c]{Test-time\\data} & Uses signal for \\
\midrule
Cycle-consistency UQ \citep{huang2023cycle} & cycle defect
  & --$^{a}$ & -- & -- & error/OOD flags \\
PDE-Refiner \citep{lippe2023pderefiner} & sample spread
  & -- & \checkmark & -- & implicit confidence \\
LE-PDE-UQ \citep{wu2024lepdeuq} & learned variance
  & -- & \checkmark & -- & error bars \\
DiffusionRollout \citep{yoo2026diffusionrollout} & sample spread
  & -- & \checkmark & -- & step-size scheduling \\
DiffusionPDE \citep{huang2024diffusionpde} & obs.\ guidance
  & --$^{b}$ & -- & \checkmark & constraining sampling \\
Cycle loss \citep{chakraborty2022cycle} & training loss
  & \checkmark & -- & n/a & training regularizer \\
\midrule
Round-trip $\cyc{i}$ (ours) & cycle defect
  & \checkmark & \checkmark & -- & flags, selective prediction \\
\bottomrule
\end{tabularx}
%\smallskip
{\footnotesize $^{a}$\,Forward leg is a known physical operator; only
the inverse is learned. \quad $^{b}$\,One diffusion prior guided in
both directions by observations.}
\end{table*}

\textbf{Surrogates and diffusion models for dynamics.}
Learned simulators range from graph networks and message-passing
solvers to neural operators and solver-in-the-loop hybrids
\citep{sanchez2020learning, brandstetter2022message, li2021fourier,
lu2021learning, um2020solver, kochkov2021machine}; diffusion models
entered as probabilistic forecasters, from turbulence benchmarking and
forecasting frameworks \citep{cachay2023dyffusion, kohl2024benchmarking}
to operational ensemble weather prediction \citep{price2025gencast},
with diffusion-style refinement stabilizing long rollouts
\citep{lippe2023pderefiner}. Latent generation with transformer
denoisers \citep{rombach2022high, peebles2023dit} makes
high-resolution, multi-field states tractable; we operate in this
latent autoregressive regime, by 2026 an established class
\citep{shysheya2024conditional, li2025latentflow,
chen2024diffusionforcing}, and add bidirectionality as the enabling
structure. Although our multi-field MHD surrogate is, to our
knowledge, among the first for that system, we deliberately do not
rest the paper's claims on it: the bidirectional model is
infrastructure; the contribution is what its structure makes
measurable.

\textbf{Uncertainty quantification for rollouts.}
Ensembles, MC dropout, and conformal calibration are the workhorses of deep UQ \citep{lakshminarayanan2017simple, gal2016dropout,
angelopoulos2023conformal}, and their reliability erodes under exactly the distribution shift that autoregressive rollouts induce
\citep{ovadia2019can}. For PDE surrogates, LE-PDE-UQ
\citep{wu2024lepdeuq} evolves latent uncertainty alongside the state
and reports gains over ensembles, Bayesian layers, and dropout, making
it the strongest published baseline for our setting (evaluated
head-to-head on its own benchmark in Sec.~\ref{sec:exp-lepdeuq});
sample spread supplies zero-shot uncertainty \citep{shu2024zeroshot,
lippe2023pderefiner}. All of these measure the model's
\emph{dispersion}, that is how much it disagrees with itself, rather
than whether the learned dynamics are applied accurately to the
trajectory at hand; we compare against dispersion signals for error
detection and selective prediction \citep{geifman2017selective,
hendrycks2017baseline}. Relatedly, diffusion posterior sampling,
score-based data assimilation, and DiffusionPDE \citep{chung2023dps,
rozet2023sda, huang2024diffusionpde} solve inverse and
state-estimation problems by guiding a prior with test-time
\emph{observations}; our backward direction is instead learned, and
the error signal we study requires no observations at all.

\textbf{Cycle consistency and nearest neighbors.}
Forward--backward agreement has long flagged unreliable optical-flow
correspondences \citep{sundaram2010dense, meister2018unflow}, and
cycle losses supervise unpaired translation and temporal
correspondence \citep{zhu2017unpaired, wang2019learning}.
\citet{huh2020trsoden} impose time-reversal symmetry as a
\emph{training} regularizer, \citet{arik2022saf} attach an auxiliary
backcasting head, and \citet{chakraborty2022cycle} reverse predicted
trajectories as a training-time loss; our signal needs no auxiliary
head, applies at test time, and directly checks the generative
dynamics. (\emph{Consistency models} \citep{song2023consistency} are a diffusion distillation technique, unrelated to the cycle consistency studied here.) The direct precedent is \citet{huang2023cycle}, who quantify uncertainty in inverse \emph{imaging} by cycling a \emph{known} forward operator with a learned inverse over a single, static problem, with supporting theory and an OOD application. Two differences define our setting: no analytic model exists in either direction. Both legs of the cycle are the same learned bidirectional generator, and the cycle closes over a $2i$-step autoregressive rollout in which both legs accumulate error with depth, which is why our central claim is quantitative: that $\cyc{i}$ tracks the depth-dependent $\err{i}$ (Sec.~\ref{sec:exp-corr}). Among rollout-time signals, DiffusionRollout \citep{yoo2026diffusionrollout} uses the \emph{dispersion} of repeated samples to schedule step sizes,
whereas $\cyc{i}$ is a \emph{functional check} computable from a
single deterministic cycle (compared head-to-head in
Table~\ref{tab:uq}). Table~\ref{tab:positioning} situates these
signals; no existing method converts its trust signal into a
correction of the prediction itself, and such adaptation is
deliberately out of scope here (Sec.~\ref{sec:discussion}), keeping
the detection claims cleanly falsifiable.

In related work, bidirectional diffusion \emph{bridges} share our one-network,
direction-flag design: BDBM \citep{kieu2025bidirectional} extends
denoising diffusion bridge models \citep{zhou2024denoising} so a single
network transports between paired domains in both directions, and
Bi-Bridge \citep{hua2026bibridge} exploits the symmetry of the DDBM
bridge as a training regularizer for low-light enhancement. Both of those works
independently observe that bidirectional training outperforms
direction specialists. In a bridge, however, both legs run between
\emph{given} endpoints, and arrival is guaranteed by construction
(Doob's $h$-transform), so composing forward and backward yields no
measurable defect; our backward leg is seeded by the model's own
predicted terminal pair and must return unanchored, which is precisely
what makes $\mathcal{C}_i$ informative.

\textbf{Generative models as scientific diagnostics.}
Conditional generative models increasingly serve as virtual
diagnostics for large experimental facilities, e.g., reconstructing
charged-particle beam phase space from sparse measurements
\citep{scheinker2024cdvae, scheinker2026phaseflow}. The bidirectional MHD application is developed in depth in a preliminary, application-focused preprint by the authors \citet{scheinker2026bidirectional}; the present paper isolates the round-trip consistency metric, quantifies its fidelity as an error estimator, and evaluates it beyond a single system.

% ============================ Background ==============================
%%%%%%%%%%%%%%%%%%%%%%%%%%%%%%%%%%%%%%%%%%%%%%%%%%%%%%
%%%%%%%%%%%%%%%%%%%%%%%%%%%%%%%%%%%%%%%%%%%%%%%%%%%%%%
%%%%%%%%%%%%%%%%%%%%%%%%%%%%%%%%%%%%%%%%%%%%%%%%%%%%%%
\section{Background}
%%%%%%%%%%%%%%%%%%%%%%%%%%%%%%%%%%%%%%%%%%%%%%%%%%%%%%
%%%%%%%%%%%%%%%%%%%%%%%%%%%%%%%%%%%%%%%%%%%%%%%%%%%%%%
%%%%%%%%%%%%%%%%%%%%%%%%%%%%%%%%%%%%%%%%%%%%%%%%%%%%%%
\label{sec:background}

\textbf{Problem setup.}
We observe trajectories $\xb_0,\dots,\xb_T$ of a dynamical system, with each state $\xb_t\in\mathbb{R}^{d\times d\times c}$ a multi-channel field snapshot (for MHD: density, pressure, velocity, and magnetic-field components). A surrogate learns the transition map and is deployed autoregressively: seeded with observed states, it feeds its own predictions back as inputs to reach depth $i$. The true rollout error at depth $i$, $\err{i}=\MSE(\zb_{t+i},\hat\zb_{t+i})$ (in latent space; decoded-field metrics in the supplementary), is what is needed and cannot be computed at deployment, because $\zb_{t+i}$ is unknown. Our goal is a test-time estimate of $\err{i}$ using only the
model itself.

\textbf{Latent encoding and diffusion.}
Following latent generative modeling \citep{rombach2022high}, each physical field is encoded independently by a $\beta$-VAE \citep{kingma2014vae, higgins2017beta} with a weak KL penalty
($\beta=10^{-3}$), mapping a $512\times512$ field to a $16\times16\times4$ latent ($256\times$ compression) that stays compact without collapsing detail; dynamics are learned in this latent space. On the latents we use standard denoising diffusion \citep{ho2020denoising, song2021score}: a network $\epsilon_\theta$ is trained to predict the noise injected by a fixed variance schedule, which defines the reverse-time sampler \citep{anderson1982}, with conditioning information entering $\epsilon_\theta$ alongside the noised latent. Our contribution lies in \emph{what} is conditioned on and \emph{how} the resulting bidirectionality is exploited.

% ============================= Method =================================
%%%%%%%%%%%%%%%%%%%%%%%%%%%%%%%%%%%%%%%%%%%%%%%%%%%%%%
%%%%%%%%%%%%%%%%%%%%%%%%%%%%%%%%%%%%%%%%%%%%%%%%%%%%%%
%%%%%%%%%%%%%%%%%%%%%%%%%%%%%%%%%%%%%%%%%%%%%%%%%%%%%%
\section{Bidirectional Latent Diffusion Dynamics}
%%%%%%%%%%%%%%%%%%%%%%%%%%%%%%%%%%%%%%%%%%%%%%%%%%%%%%
%%%%%%%%%%%%%%%%%%%%%%%%%%%%%%%%%%%%%%%%%%%%%%%%%%%%%%
%%%%%%%%%%%%%%%%%%%%%%%%%%%%%%%%%%%%%%%%%%%%%%%%%%%%%%
\label{sec:model}

A single denoiser $\epsilon_\theta$ learns the bidirectional transition
density
\begin{equation}
  \hat\zb_{t+\cd}\sim p_\theta\!\left(\zb_{t+\cd}\mid \zb_t,\,
  \zb_{t-\cd},\,\cd\right),\quad \cd\in\{+1,-1\},
  \label{eq:bidir}
\end{equation}
trained with the conditional noise-prediction loss
\begin{equation}
  \mathcal{L}(\theta)=\mathbb{E}_{t,\cd,k,\epsilon}
  \big\|\epsilon-\epsilon_\theta\big(\zb^{(k)}_{t+\cd},k,\zb_t,
  \zb_{t-\cd},\cd\big)\big\|_2^2,
  \label{eq:loss}
\end{equation}
with physical time $t$ and direction $\cd$ drawn uniformly over training trajectories and $k$ the denoising step, $\zb^{(k)}$ denoting the correspondingly noised target. Drawing both signs of $\cd$ forces the shared weights to denoise consistently toward the future ($\cd=+1$, surrogate solver) and the past ($\cd=-1$, inverse solver). Rollouts compose \eqref{eq:bidir} autoregressively; decoded fields are $\hat\xb_t=\Dec(\hat\zb_t)$. The denoiser is a diffusion transformer \citep{peebles2023dit}
throughout ($n{=}2$ context frames unless noted): target, context, and
anchor frames are patchified into one token sequence, so conditioning
on past states acts through attention, while the scalar conditions ---
diffusion step, simulation-time index, and the direction flag $\cd$
--- modulate every block via adaLN-Zero (full architecture in
Supp.~Sec.~A). Throughout, ``forward'' and ``backward'' refer to the direction of
physical time $t$, selected by $c_d$, not to the diffusion process's
own noising/denoising steps, indexed by $k$; every physical step, in
either temporal direction, is produced by a complete denoising pass.

%%%%%%%%%%%%%%%%%%%%%%%%%%%%%%%%%%%%%%%%%%%%%%%%%%%%%%
%%%%%%%%%%%%%%%%%%%%%%%%%%%%%%%%%%%%%%%%%%%%%%%%%%%%%%
%%%%%%%%%%%%%%%%%%%%%%%%%%%%%%%%%%%%%%%%%%%%%%%%%%%%%%
\section{Round-Trip Consistency}
%%%%%%%%%%%%%%%%%%%%%%%%%%%%%%%%%%%%%%%%%%%%%%%%%%%%%%
%%%%%%%%%%%%%%%%%%%%%%%%%%%%%%%%%%%%%%%%%%%%%%%%%%%%%%
%%%%%%%%%%%%%%%%%%%%%%%%%%%%%%%%%%%%%%%%%%%%%%%%%%%%%%
\label{sec:consistency}

Let $\fwd^{\,i}$ denote an $i$-step forward rollout seeded by the true
pair $(\zb_{t-1},\zb_t)$ and $\bwd^{\,i}$ the $i$-step backward rollout
seeded by the terminal predicted pair. An error-free model satisfies
$\bwd^{\,i}\circ\fwd^{\,i}=\mathrm{Id}$. Round-trip consistency
error at depth $i$ is
\begin{equation}
  \cyc{i}=\tfrac12\Big[\MSE\big(\zb_{t-1},\tilde\zb^{(i)}_{t-1}\big)
  +\MSE\big(\zb_{t},\tilde\zb^{(i)}_{t}\big)\Big],
  \label{eq:cyc}
\end{equation}
where $\tilde\zb^{(i)}$ are the returned seeds. Every quantity in \eqref{eq:cyc} is available at test time, the
anchor pair is encoded measured data, the returned pair is produced
entirely by the model, in contrast to $\err{i}$. Given a tolerance $\tau$, a rollout is trusted to the largest depth with $\cyc{i}\le\tau$: an early-stop criterion for autoregressive generation.

\textbf{Properties and cost.}
$\cyc{i}$ costs one backward rollout per checked depth --- a $2\times$
inference overhead, with no change to
training. We use deterministic DDIM sampling throughout, which makes
the cycle well defined; averaging $S$ stochastic cycles is a variant
whose spread we compare as a dispersion baseline
(Table~\ref{tab:uq}). The check is necessary, not sufficient: Sec.~\ref{sec:theory} bounds
when small $\cyc{i}$ certifies small $\err{i}$, and the experiments
measure how often cancellation bites in practice.

\subsection{When Does the Proxy Certify the Error?}
\label{sec:theory}

Because the model is second-order Markov, the natural state is the pair $\sv_k:=(\zb_{t+k-1},\zb_{t+k})\in\mathbb{R}^{2n}$, with the normalized norm $\pnorm{\sv}:=\big(\tfrac{1}{2n}\|\sv\|_2^2\big)^{1/2}$ chosen so that $\cyc{i}=\pnorm{\sv_0-\tilde\sv_0}^2$ matches Eq.~\eqref{eq:cyc} exactly. With a deterministic sampler one model step is a map on pairs: we write $\fwd,\bwd:\mathbb{R}^{2n}\!\to\!\mathbb{R}^{2n}$ for the forward and backward steps, so the forward rollout is $\hat\sv_i=\fwd^{\,i}(\sv_0)$, the returned seed is $\tilde\sv_0=\bwd^{\,i}(\hat\sv_i)$, and the pair-level rollout error
is $\perr{i}:=\pnorm{\sv_i-\hat\sv_i}^2$, which dominates the
terminal-state error via $\err{i}\le 2\,\perr{i}$.

\begin{assumption}\label{ass:lip}
Let $\mathcal{D}\subset\mathbb{R}^{2n}$ contain the two backward
trajectories $\{\bwd^{\,k}(\sv_i)\}_{k=0}^{i}$ and
$\{\bwd^{\,k}(\hat\sv_i)\}_{k=0}^{i}$. On $\mathcal{D}$, $\forall a,b,$ and some $0<\mu\le L<\infty$, the backward step satisfies
\begin{equation}
  \mu\,\pnorm{a-b}\;\le\;\pnorm{\bwd(a)-\bwd(b)}\;\le\;L\,\pnorm{a-b}.
  \label{eq:bilip}
\end{equation}
Define the backward residual on true data,
$\delta_i:=\pnorm{\bwd^{\,i}(\sv_i)-\sv_0}$, i.e., the error of the
backward model when seeded with the \emph{true} terminal pair.
\end{assumption}

The upper inequality in \eqref{eq:bilip} is ordinary Lipschitz
continuity; the lower one, co-Lipschitz continuity, is precisely
the anti-cancellation condition: it forbids the backward map from
collapsing distinct terminal states onto the same returned seed
(unpacked in Supp.~Sec.~C).

\begin{proposition}\label{prop:sandwich}
Under Assumption~\ref{ass:lip},
\begin{equation}
  \Big(\max\big\{\mu^{\,i}\sqrt{\perr{i}}-\delta_i,\,0\big\}\Big)^{2}
  \le \cyc{i} \le
  \Big(L^{\,i}\sqrt{\perr{i}}+\delta_i\Big)^{2}.
  \label{eq:sandwich}
\end{equation}
\end{proposition}

\begin{corollary}\label{cor:certificate}
Under Assumption~\ref{ass:lip},
\begin{equation}
  \sqrt{\err{i}}\;\le\;\sqrt{2}\;\mu^{-i}
  \Big(\sqrt{\cyc{i}}+\delta_i\Big).
  \label{eq:certificate}
\end{equation}
Small $\cyc{i}$ certifies small true error when the learned backward map is well conditioned ($\mu$ not too small) and accurate on clean data ($\delta_i$ small) -- both properties of the model alone, estimable offline without test-time ground truth.
\end{corollary}

The proof combines \eqref{eq:bilip} with the triangle inequality
around $\bwd^{\,i}(\sv_i)$; the full derivation is in Supp.~Sec.~C.
Both constants are measurable offline: the residual $\delta_i$ is the detector's noise floor,
estimated on validation data (informative when $\cyc{i}\gg\delta_i^2$), and the factor $\mu^{-i}$ grows geometrically, so the certificate is meaningful at moderate turnaround depths and for well-conditioned backward maps: as $\mu\to 0$ the bound loosens as a collapsing
inverse can hide errors. Condition \eqref{eq:bilip} constrains the \emph{learned}
backward map on rollout-visited pairs only; no invertibility of the
underlying physical dynamics is assumed. Because \eqref{eq:sandwich}
pins $\sqrt{\cyc{i}}$ between two affine functions of
$\sqrt{\perr{i}}$, it justifies fitting monotone or heteroscedastic
Gaussian calibrators of the observed error on $\cyc{i}$
\citep{nix1994estimating, kuleshov2018accurate}, in the spirit of
\citet{huang2023cycle}; Sec.~\ref{sec:exp-calibration} fits exactly
such a calibrator. Estimation of $(\mu,L)$ along rollouts, certified
bi-Lipschitz architectures, and stochastic sampling are discussed in
Supp.~Sec.~C.

% =========================== Experiments ==============================
%%%%%%%%%%%%%%%%%%%%%%%%%%%%%%%%%%%%%%%%%%%%%%%%%%%%%%
%%%%%%%%%%%%%%%%%%%%%%%%%%%%%%%%%%%%%%%%%%%%%%%%%%%%%%
%%%%%%%%%%%%%%%%%%%%%%%%%%%%%%%%%%%%%%%%%%%%%%%%%%%%%%
\section{Experiments}
%%%%%%%%%%%%%%%%%%%%%%%%%%%%%%%%%%%%%%%%%%%%%%%%%%%%%%
%%%%%%%%%%%%%%%%%%%%%%%%%%%%%%%%%%%%%%%%%%%%%%%%%%%%%%
%%%%%%%%%%%%%%%%%%%%%%%%%%%%%%%%%%%%%%%%%%%%%%%%%%%%%%
\label{sec:experiments}

%\subsection{Setup}
\textbf{Data.}
Synthetic 2D compressible MHD trajectories with randomized initial
density, pressure, and velocity, and magnetic fields derived from a
random vector potential, simulated with a constrained-transport scheme
\citep{mocz2014constrained} at $512\times512$; 100 steps over 1\,s;
500 training and 50 held-out test trajectories (generation details in
Supp.~Sec.~A). Out-of-distribution evaluation uses the Orszag--Tang vortex \citep{orszag1979} with the model trained only on the synthetic distribution. \textbf{Baselines.} Dispersion-based UQ signals computed from the \emph{same} trained model: $S$-sample rollout spread \citep{lippe2023pderefiner, shu2024zeroshot}. MC-dropout \citep{gal2016dropout} and deep ensembles \citep{lakshminarayanan2017simple} are inapplicable to our single trained, dropout-free model without
retraining (Table~\ref{tab:uq}). LE-PDE-UQ \citep{wu2024lepdeuq} is evaluated head-to-head on its own Navier--Stokes benchmark in Sec.~\ref{sec:exp-lepdeuq}. \textbf{Metrics.} Per-trajectory and per-depth Spearman/Pearson correlation between
$\cyc{i}$ and $\err{i}$; AUROC for flagging high-error and OOD
rollouts; risk--coverage curves for selective prediction and
calibration of the $\tau$ rule; rollout MSE per field vs.\ depth for fidelity.

\subsection{Does the Proxy Track the Truth?}
\label{sec:exp-corr}
% THE key experiment -- the quantitative claim the arXiv version
% asserts but does not yet measure. [inference-only -- days]
The leftmost panel of Fig.~\ref{fig:fit_MHD} shows the raw relationship on the
held-out MHD trajectories; Sec.~\ref{sec:exp-calibration} quantifies
it as a calibrated predictor. Table~\ref{tab:uq} compares $\cyc{}$
against the natural sampling-dispersion alternative, and the two turn
out to be \emph{complementary}.
In-distribution, the $S{=}5$ spread is a slightly better magnitude
predictor ($-0.99$ vs.\ $-0.62$ nats): dispersion measures the
aleatoric width of the realized noise draw, at $5\times$
inference cost and only under stochastic sampling. Out of
distribution the two signals \emph{invert}: on the Orszag--Tang
vortex the model is confidently wrong, at depth $5$ its dispersion
is \emph{lower} than every in-distribution trajectory's (AUROC
$0.00$), ranking the OOD case as the safest in the batch, but
$\cyc{}$ flags it above all fifty (AUROC $1.00$). Conditional width
has no reason to grow under distribution shift; realized round-trip
drift must. $\cyc{}$ thus offers near-parity in-distribution at
$2\times$ cost, compatibility with deterministic deployment, and the
failure sensitivity a trust signal exists to provide; Sec.~\ref{sec:exp-lepdeuq} composes the two.

%\begin{figure}[t]
%\centering
%\includegraphics[width=0.7\columnwidth]{fig_corr_scatter.png}
%\caption{Round-trip consistency $\cyc{i}$ vs.\ true latent rollout
%error $\err{i}$ on the 50 held-out MHD trajectories, colored by depth $i$ (turnaround-anchored cycles; %deterministic sampler). Fixed-depth Spearman is $0.91$--$0.98$ (Sec.~\ref{sec:exp-calibration}).}
%\label{fig:corr-scatter}
%\end{figure}

\begin{table}[t]
\centering
\caption{Test-time signals vs.\ true rollout error on held-out MHD
trajectories. $\rho_{20}$: Spearman across trajectories at fixed depth
$i{=}20$; NLL / $\times_{68}$: identically-fit calibrator on the
held-out set; OOD: AUROC separating the Orszag--Tang vortex (single
canonical trajectory; percentile rank) from the 50 held-out
trajectories; cost: inference overhead. }
\label{tab:uq}
\footnotesize
\setlength{\tabcolsep}{3pt}
\resizebox{\columnwidth}{!}{%
\begin{tabular}{lccccc}
\toprule
Signal & $\rho_{20}$ $\uparrow$ & NLL $\downarrow$ & $\times_{68}$ &
OOD AUROC $\uparrow$ & Cost \\
\midrule
Round-trip $\cyc{}$ (ours) & $0.97$ & $-0.62$ & $1.14$ &
$\mathbf{1.00}$ ($i \le 10$) & $\mathbf{2\times}$ \\
$S{=}5$ rollout spread & $\mathbf{0.98}$ & $\mathbf{-0.99}$ &
$\mathbf{1.09}$ & $0.00$ ($i{=}5$) -- $0.50$ & $5\times$ \\
MC-dropout variance & \multicolumn{4}{c}{n/a --- no dropout layers} & --- \\
Deep ensemble & \multicolumn{4}{c}{n/a --- single trained model} &
$3\times$ train \\
\bottomrule
\end{tabular}}
\end{table}

\subsection{Calibrated Error Prediction}
\label{sec:exp-calibration}

Rank correlation shows that $\cyc{i}$ orders errors correctly; a
deployed trust signal should go further and \emph{predict} the error's
magnitude. We therefore fit a heteroscedastic Gaussian calibrator
\citep{nix1994estimating}, modeling $\log \err{i} \sim
\mathcal{N}\!\big(\mu(\log \cyc{i}),\,\sigma^2(\log \cyc{i})\big)$ with
polynomial $\mu$ and $\log\sigma$ on the training-trajectory rollouts;
degrees $(4,2)$ are selected by trajectory-level 5-fold
cross-validation with the one-standard-error rule (further safeguards in Supp.~Sec.~E), and the 50 held-out trajectories are touched
exactly once, for the evaluation reported here (4{,}900
depth--trajectory pairs). Two baselines are fit with the
\emph{identical} recipe: a \emph{depth-only} calibrator that predicts
$\err{i}$ from the depth $i$ alone, the critical control, since any
signal that merely grows with depth matches it and a global
constant.

Table~\ref{tab:calibration} summarizes held-out performance. The
$\cyc{i}$-based calibrator predicts the unobservable error to within a
factor of $1.14$ (68\%) and $1.29$ (95\%) of held-out points with
near-nominal coverage and a miscalibration area of $0.013$
\citep{kuleshov2018accurate}, at roughly half the interval width of
the depth-only baseline (mean $\log$-space $\sigma$ of $0.13$ vs.\
$0.26$). The headline is the proper-score margin: it improves held-out
log-likelihood by $1.0$ nat over the depth-only calibrator, and the
advantage persists \emph{at every probed depth} ($0.6$--$1.3$ nats at
$i \in \{5,10,20,40,80\}$), establishing that the self-supervised
signal carries per-trajectory information beyond the shared growth of
error with depth. Consistently, $\cyc{i}$ and $\err{i}$ are
rank-correlated within trajectories (Spearman $0.69 \pm 0.16$ across
held-out trajectories; $0.64 \pm 0.14$ on training rollouts) and across
trajectories at fixed depth ($0.91$--$0.98$ at
$i \in \{5,10,20,40,80\}$; $0.97$ at $i{=}20$, Table~\ref{tab:uq},
at or slightly above the reliability ceiling derived in
Supp.~Sec.~B; scatter shown in Supp.~Fig.~1).

\textbf{Transfer to decoded field-space errors.}
The quantity an operator ultimately cares about is the decoded
physical-field error, which interposes decoder reconstruction noise
between $\cyc{i}$ and the target. We therefore fit one calibrator per
field with the same recipe, taking $\log$ decoded per-field error as
the target: the gain over the identically-fit depth-only baseline is
positive for all six fields ($+0.32$ to $+0.73$ nats, mean $+0.59$),
with the true field error predicted to within a factor of
$1.17$--$1.30$ (68\%; per-field breakdown and discussion in
Supp.~Sec.~E, Table~1 there). The latent, self-supervised signal thus
remains a quantitative predictor all the way down to the decoded
physical quantities.

One limitation is visible and expected: while pooled calibration is
near-nominal, depth-conditional calibration drifts at the deepest
rollouts (per-depth miscalibration area up to $0.27$ at $i \ge 40$,
vs.\ $0.013$ pooled), this is the regime where
Corollary~\ref{cor:certificate}'s $\mu^{-i}$ factor predicts the
signal loosens. A trajectory-mean augmentation of the calibrator
recovers $+0.26$ nats and largely repairs the mid-depth drift (full
analysis in Supp.~Sec.~E).

\begin{table}[t]
\centering
\caption{Calibrated error prediction on held-out MHD trajectories
(50 trajectories $\times$ 98 depths); all calibrators fit on training
rollouts only, with the identical heteroscedastic recipe. NLL:
held-out Gaussian negative log-likelihood (nats); RMSE$_{\log}$: RMSE
of $\mu$ in log space; $\times_{68}$/$\times_{95}$: multiplicative factor containing the true error for 68\%/95\% of
points; Cov: empirical $\pm1\sigma$/$\pm2\sigma$ coverage (nominal
$68.3/95.4$); MCA: miscalibration area.}
\label{tab:calibration}
\footnotesize
\setlength{\tabcolsep}{4pt}
\resizebox{\columnwidth}{!}{%
\begin{tabular}{lcccccc}
\toprule
Calibrator & NLL $\downarrow$ & RMSE$_{\log}$ $\downarrow$ &
$\times_{68}$ & $\times_{95}$ & Cov$_{68/95}$ (\%) & MCA $\downarrow$ \\
\midrule
Global constant & $1.44$ & $1.02$ & $1.94$ & $9.83$ & $87.5\,/\,93.7$ & $0.095$ \\
Depth-only      & $0.38$ & $0.34$ & $1.41$ & $1.92$ & $51.8\,/\,84.3$ & $0.059$ \\
$\mu(\cyc{i})$ (ours) & $\mathbf{-0.62}$ & $\mathbf{0.13}$ &
$\mathbf{1.14}$ & $\mathbf{1.29}$ & $\mathbf{66.3\,/\,95.1}$ &
$\mathbf{0.013}$ \\
\bottomrule
\end{tabular}}
\end{table}

\begin{figure*}[h]
\centering
   \includegraphics[width=1.0\textwidth]{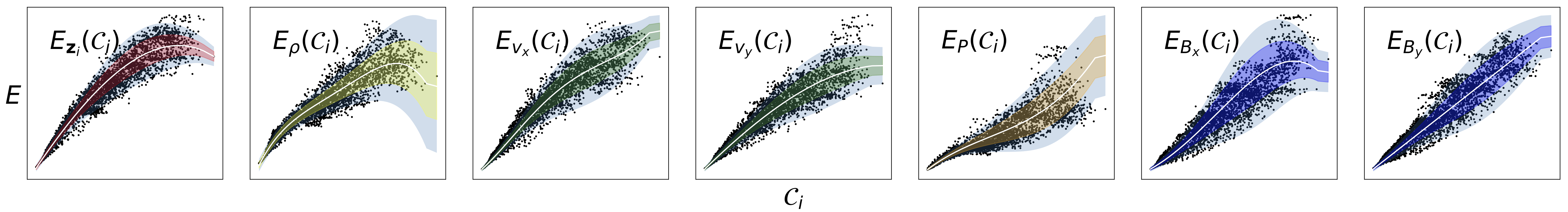}
  \caption{Predicted vs.\ true RMS errors on test data for the rolled-out
latents $\mathbf{z}_i$ and all six MHD fields: the true error (black
dots) is predicted as a function of $\mathcal{C}_i$ and falls within
$\pm 2\sigma$ bands fit from training data alone.}
  \label{fig:fit_MHD}
\end{figure*}

\subsection{Selective Prediction and OOD Detection}
\label{sec:exp-selective}
% [inference-only -- days]
Because $\mu(\cyc{i})$ is a calibrated prediction of the error
(Sec.~\ref{sec:exp-calibration}), it serves directly as a deferral score \citep{geifman2017selective}. Risk--coverage on the held-out trajectories (Supp.~Fig.~6): deferring the $20\%$ of predictions with the highest predicted error reduces the mean incurred latent error by $15\%$ (at $90\%$/$70\%$ coverage: $7\%$/$23\%$), where the critical control, deferring by depth alone (early
stopping), achieves only $5\%$ ($2\%$/$9\%$). The per-trajectory information certified in
nats by the calibrator is thus directly actionable.
For out-of-distribution detection we roll the model on the
Orszag--Tang vortex, dynamics it has never seen, and compare $\cyc{}$ against the 50 held-out in-distribution
trajectories; with a single canonical OOD trajectory, AUROC equals its
percentile rank among them. The alarm rings immediately: AUROC is
$1.0$ at depths $5$ and $10$ (the vortex's cycle error exceeds
\emph{every} in-distribution value), $0.98$ for the trajectory-mean
score, and weakens at mid depths ($0.64$--$0.88$) where
in-distribution errors approach the same saturation scale
(Supp.~Sec.~B), the ideal profile: an OOD flag is most valuable before compute is spent on a long
rollout (Fig. \ref{fig:ood-hist}).

\begin{figure}[t]
\centering
\includegraphics[width=1.0\columnwidth]{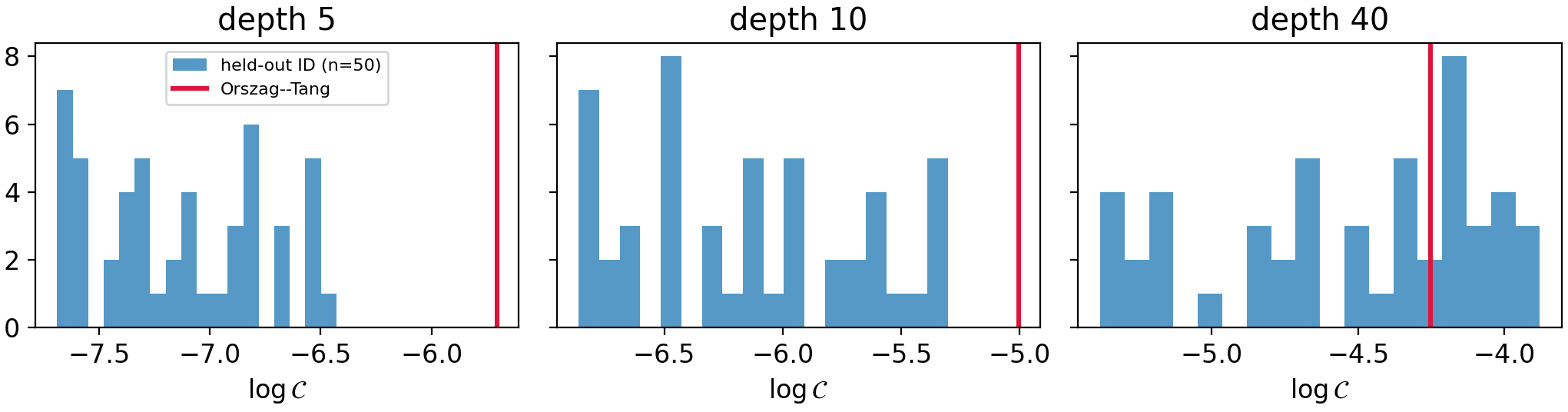}
\caption{Out-of-distribution detection: distribution of
$\log \cyc{}$ over the 50 held-out trajectories at three depths, with the Orszag--Tang vortex marked: total shallow depth separation.}
\label{fig:ood-hist}
\end{figure}

\subsection{Inverse Rollouts}
\label{sec:exp-inverse}
The bidirectionality is also a capability in its own right: setting $\cd{=}-1$ turns the trained surrogate into a fast inverse solver at no additional training cost. Seeded with only a true terminal pair $(\zb_{T-1},\zb_T)$ of a held-out trajectory, backward rollouts reconstruct the preceding plasma history (decoded-field reconstructions in Supp.~Fig.~2) - precisely the rollouts whose residual $\delta_i$ is the noise floor of Supp. Sec. C with backward error growing in depth comparably to the forward direction. Cycling a backward rollout \emph{forward} defines its own consistency error $\cyc{i}^{\,-}$, which tracks the backward rollout error just as $\cyc{i}$ tracks the forward one (Supp.~Fig.~3); the same round trips viewed in the latent space where all metrics are computed are shown in the forward roll out in Supp.~Fig.~4 and backward roll out in Supp.~Fig.~5.

\subsection{Beyond Physics: Natural Face Videos}
\label{sec:exp-faces}
As a probe, not a benchmark battle, we train the \emph{identical} architecture on natural video: face clips from CelebV-HQ \citep{zhu2022celebvhq} (35{,}666 clips, 15{,}653 identities). Frames are resized to $256\times256$ and encoded per-frame by a \emph{frozen}, pretrained Stable Diffusion image VAE \citep{rombach2022high} into
$32\times32\times4$ latents (no video VAE is trained; cf.\ the
per-field encoding of Sec.~\ref{sec:background}). The same
bidirectional denoiser \eqref{eq:bidir}-\eqref{eq:loss} then learns
facial dynamics. Identity-disjoint held-out clips supply true frames,
so $\err{i}$ is measurable in the frozen latent space exactly as in
the physics settings, and $\cyc{i}$ is computed by the same cycle.
Talking-head dynamics are only mildly time-asymmetric, which is what
makes the backward direction learnable here; strongly irreversible
content (e.g., pouring, smoke) is out of scope.
\begin{figure}[t]
\centering
\includegraphics[width=1.0\columnwidth]{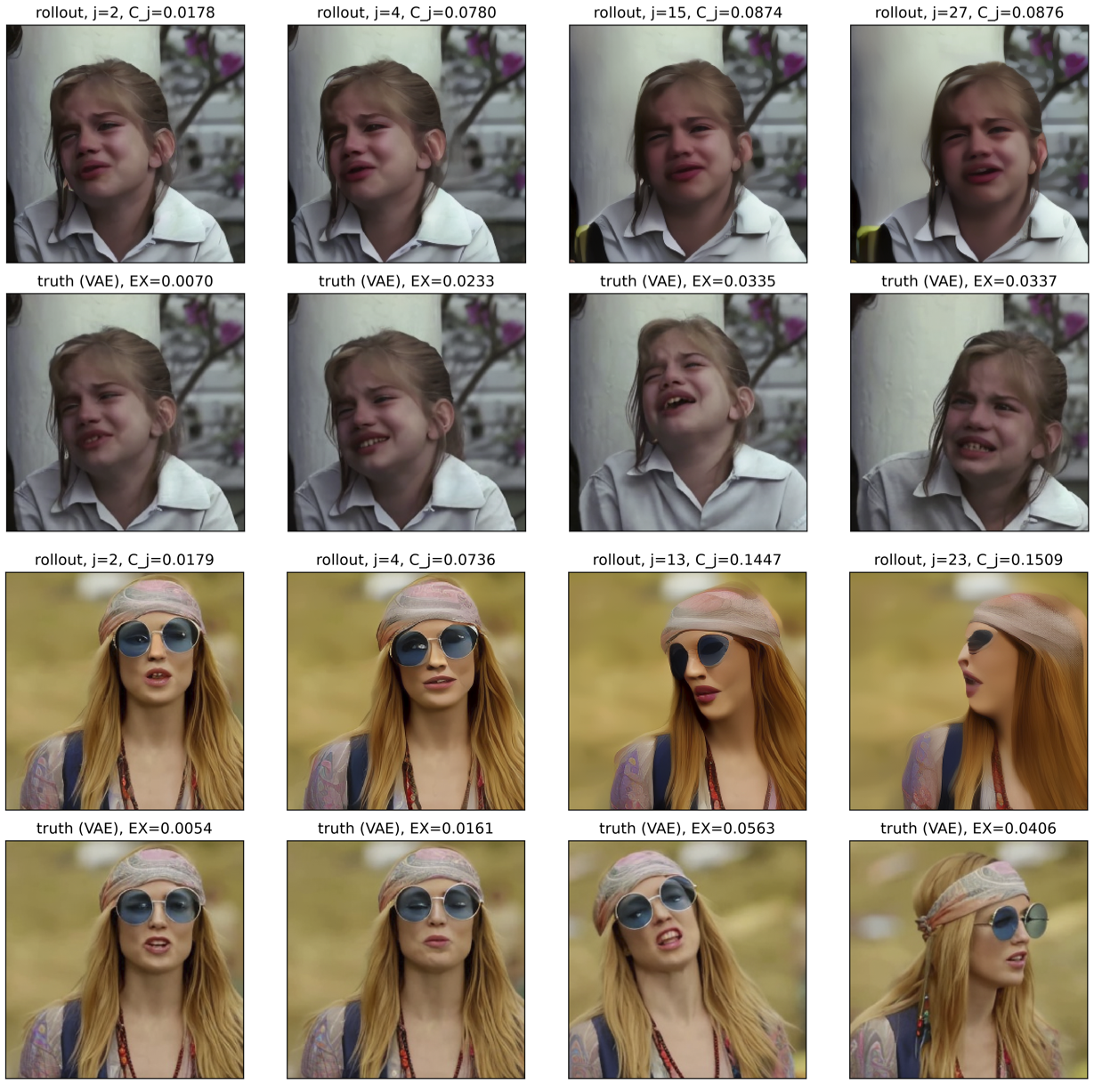}
\caption{$\cyc{}$ tends to detect when faces morph OOD.}
\label{fig:face-drift}
\end{figure}
Absolute rollout fidelity is deliberately modest (a competitive video
predictor requires scale beyond this probe), and talking-head futures
are intrinsically multimodal, so any single rollout's pointwise latent
error grows quickly against the one realized future, this is precisely the
regime in which a trust signal matters and, as
Sec.~\ref{sec:exp-calibration} anticipated, in which the
per-trajectory signal is strongest: clip-to-clip difficulty variation
dominates the metric noise. On 150 identity-disjoint held-out clips,
$\cyc{}$ ranks clips by error at fixed depth with Spearman
$0.79$-$0.81$ (latent) and $0.73$-$0.80$ (decoded pixels) at rollout
depths $3$-$21$ ($\approx 0.4$-$2.5$\,s), \emph{without decay at
depth}. Unlike the radiative layer, the clip factor never
saturates. The same calibration recipe transfers: a
\emph{two-parameter} calibrator (a single global line in $\log \cyc{}$ with constant scatter; selection and configuration in Supp.~Sec.~A.6)
fit on half the held-out clips predicts the disjoint half's latent
error within a factor of $1.28$ (68\%) with $95.3\%$ coverage at
$\pm 2\sigma$, a $+0.43$-nat gain over the depth-only baseline
($+0.34$ nats for decoded pixel error). The probe's
question is answered affirmatively: the model predicts the
growth of its own error, with no test-time ground truth, even when
its predictions have visibly drifted (Fig.~\ref{fig:face-drift}).

\subsection{Learned-Variance UQ on Navier--Stokes}
\label{sec:exp-lepdeuq}

\textbf{Setup.}
We compare directly against LE-PDE-UQ \citep{wu2024lepdeuq} on its own
benchmark: 2D incompressible Navier--Stokes in vorticity form on a unit
torus ($64^2$, $Re{=}10^4$; the standard neural-operator dataset of
\citealp{li2021fourier}), 1200 trajectories of 20 steps, using their exact
data files, splits (train = first 1000, test = last 200), and protocol: the past 10 steps predict the next, rolled out autoregressively
for 10 steps. Our pipeline is unchanged: a per-snapshot cVAE
($64^2 \!\to\! 8{\times}8{\times}8$) and the identical bidirectional
denoiser with $n{=}10$ context frames and deterministic 25-step DDIM
sampling. Two training additions (context/anchor noise; pushforward-style
context augmentation \citep{brandstetter2022message}) and checkpoint
selection by validation \emph{rollout} error are used; both are
motivated and disclosed in full in Supp.~Sec.~E. All uncertainty objects are fit on held-out calibration trajectories
from the validation split, not training trajectories whose
rollout error under-represents deployment error here by $4.8\times$ on
average. Final numbers produced on untouched test
split (full protocol, and the predicted-vs-measured test bands, in
Supp.~Sec.~E and Supp.~Fig.~8). Rollout error on this benchmark grows geometrically, $e_d \approx e_1\, r^{\,d-1}$ with $r \approx 1.21$ per step (fit within ${\sim}5\%$ at every depth; Supp.~Fig.~7), the depth-uniform
amplification the geometric factors of Sec.~\ref{sec:theory}
anticipate. A depth-only calibrator captures the shared rate $r$;
whatever distinguishes reliable from unreliable rollouts lives in the
per-trajectory deviations about this trend (cf.\ the variance
decomposition of Supp.~Sec.~B) which is what $\cyc{i}$ measures.

\textbf{Results.}
Table~\ref{tab:lepdeuq} reports exact metrics and pooling conventions: the \texttt{uncertainty\_toolbox} calibration metrics
\citep{chung2021uncertainty} (miscalibration area (MA), MACE, RMSCE) on
per-pixel predictions, per-trajectory block relative $L_2$, and global
MAE. A single bidirectional model attains a block relative $L_2$ of
$0.244$ on the untouched test split: within $1.16\times$ of
their best single model \emph{with uncertainty} ($0.211$) and
$1.29\times$ of their most accurate configurations (the ten-model
ensemble, $0.190$, and the single model without an uncertainty
head, $0.189$) at a tenth of the training cost, with a larger
MAE ($0.204$ vs.\ their $0.161$-$0.181$),
reflecting that metric's weighting of the late, high-amplitude frames of
this spin-up flow. Their single-model uncertainty is not free: the
$\sigma$ head costs $11.5\%$ accuracy ($0.189 \to 0.211$), whereas the
round-trip signal attaches to the unmodified model at zero accuracy
cost. Averaging $S{=}5$ independently seeded rollouts improves accuracy by
only $0.6\%$, the seed-to-seed dispersion is ${\sim}14\times$
smaller than the realized error: the mistakes are systematic given
context, the same blindness of self-disagreement that motivates
round-trip checking (Sec.~\ref{sec:exp-corr}). 

For calibration, converting the calibrated $\mu(\cyc{})$ into a
constant-per-frame pixel $\sigma$ lands at MA${}=0.100$, on par with
an identically fitted depth-only $\sigma$ ($0.090$): at this
benchmark's moderate heterogeneity (fixed-depth Spearman $0.32$--$0.59$), a frame-level
scalar from $\cyc{}$ adds little beyond depth. The signal's value
here arrives through \emph{shape} and \emph{composition}: the
forward--backward--forward cycle disagreement (one deterministic
model, $3\times$ rollouts) reaches $0.088$, and composing
$\mu(\cyc{})$'s scale with the seed spread's spatial shape reaches
MA${}=0.082$, the best training-free result, against $0.058$
for their learned single-model $\sigma$ and $0.014$ for their ten-model ensemble. Dispersion alone
is not calibration: their ten-model
spread-only reaches $0.182$ and our $S{=}5$ seed spread $0.118$,
both behind the $2\times$ depth-only ablation. The residual gap to
learned per-pixel uncertainty is within-frame heteroscedasticity that
no frame-level scalar represents ($z$-score diagnostics in
Supp.~Sec.~E). In latent space the same run completes the
paper's cross-system heterogeneity spectrum (between the radiative
layer's $0.11$ and the faces' $0.79$); latent-space correlations and
calibration are in Supp.~Sec.~E.

\begin{table}[t]
\centering
\caption{Head-to-head on LE-PDE-UQ's Navier--Stokes benchmark, their
metrics and pooling. LE-PDE-UQ rows: test split, as published. Our rows: one-shot test
evaluation ($n{=}200$ trajectories; all uncertainty objects fit on
the validation split). Cost:
training / inference relative to one model, one rollout.}
\label{tab:lepdeuq}
\footnotesize
\setlength{\tabcolsep}{2.5pt}
\resizebox{\columnwidth}{!}{%
\begin{tabular}{lcccccc}
\toprule
Signal & MA $\downarrow$ & MACE & RMSCE & $L_2$ $\downarrow$ & MAE & Cost \\
\midrule
\multicolumn{7}{l}{\emph{LE-PDE-UQ \citep{wu2024lepdeuq}}} \\
Latent single (no $\sigma$) & -- & -- & -- & $0.189$ & $0.161$ &
$1\times$ tr., no UQ \\
Latent single ($+\sigma$) & $0.058$ & $0.057$ & $0.065$ & $0.211$ &
$0.181$ & $1\times$ tr., $\sigma$ head \\
Latent ensemble (10, spread only) & $0.182$ & $0.181$ & $0.202$ &
$0.190$ & $0.161$ & $10\times$ tr., $10\times$ inf. \\
Latent ensemble (10, $+\sigma$) & $\mathbf{0.014}$ & $0.014$ & $0.016$ &
$0.190$ & $0.161$ & $10\times$ tr., $10\times$ inf. \\
\midrule
\multicolumn{7}{l}{\emph{Ours (single bidirectional model)}} \\
$\mu(\cyc{})$, constant $\sigma$ & $0.100$ & $0.099$ & $0.112$ &
$0.244$ & $0.204$ & $2\times$ rollouts \\
$\;+$ cycle-disagreement shape & $0.088$ & $0.087$ & $0.099$ & $0.244$
& $0.204$ & $3\times$ rollouts \\
$S{=}5$ seed spread alone & $0.118$ & $0.117$ & $0.133$ & $0.242$ &
$0.203$ & $5\times$ rollouts \\
$\mu(\cyc{})\times$ seed-spread shape & $\mathbf{0.082}$ & $0.081$
& $0.092$ & $0.242$ & $0.203$ & $5\times$ rollouts \\
depth-only $\sigma$ (ablation) & $0.090$ & $0.089$ & $0.102$ & $0.244$
& $0.204$ & $2\times$ rollouts \\
\bottomrule
\end{tabular}}
\end{table}

\subsection{Ablations}
\label{sec:exp-ablations}
\textbf{Does bidirectionality cost accuracy?}
The opposite, in the low-data regime. On the turbulent radiative layer (80 training trajectories) we train the identical architecture forward-only, backward-only, and bidirectionally ($p_{\mathrm{bwd}}{=}0.5$) at matched compute, evaluated direction-matched with fixed noise (10
seeds; mean $\pm$ SE; Supp.~Sec.~D). The shared model outperforms each specialist \emph{on the specialist's own direction}: $0.0630 \pm 0.0006$ vs.\ $0.0675 \pm 0.0007$ forward and $0.0621 \pm 0.0006$ vs.\ $0.0689 \pm 0.0007$ backward, a $7$-$10\%$ improvement despite seeing only half as many examples per direction, and it overfits markedly later (Supp.~Sec.~D). 

Each trajectory supplies transitions in both temporal directions, so the direction flag acts as free augmentation. The specialists fail severely off-direction ($5$-$7\times$ higher loss), confirming the round-trip check requires both directions: the structure that enables $\cyc{i}$ thus comes at negative cost: one network, half the parameters of two specialists, better accuracy in both directions. A linear analysis explains the mechanism: for stationary dynamics the time-reversed map is a fixed reparameterization of the forward one ($\Sigma A^{\!\top}\Sigma^{-1}$; exactly $A^{\!\top}$ for whitened latents), so the direction flag implements a correctly-specified weight tying whose estimation variance is up to half a specialist's, a gain scaling as $1/n$, largest in the low-data regime measured here (Supp.~Sec.~D). On the turbulent radiative layer the situation inverts instructively: the per-trajectory factor is smaller ($\pm 11\%$ vs.\ MHD's $\pm 28\%$) and is masked by $\cyc{i}$'s own per-depth noise, the variance decomposition of Supp.~Sec.~B bounds the observable
fixed-depth correlation at $\approx 0.12$, matching the measured
$0.11$ ($n{=}72$), so a depth-only calibrator is already
near-optimal here. The per-trajectory signal is thus valuable in
proportion to how sharply trajectories differ relative to the metric's
noise: large on MHD, marginal here, and maximal by construction for
natural video, where clip-to-clip heterogeneity dominates.

% ====================== Discussion and Conclusion =====================
%%%%%%%%%%%%%%%%%%%%%%%%%%%%%%%%%%%%%%%%%%%%%%%%%%%%%%
%%%%%%%%%%%%%%%%%%%%%%%%%%%%%%%%%%%%%%%%%%%%%%%%%%%%%%
%%%%%%%%%%%%%%%%%%%%%%%%%%%%%%%%%%%%%%%%%%%%%%%%%%%%%%
\section{Discussion and Conclusion}
%%%%%%%%%%%%%%%%%%%%%%%%%%%%%%%%%%%%%%%%%%%%%%%%%%%%%%
%%%%%%%%%%%%%%%%%%%%%%%%%%%%%%%%%%%%%%%%%%%%%%%%%%%%%%
%%%%%%%%%%%%%%%%%%%%%%%%%%%%%%%%%%%%%%%%%%%%%%%%%%%%%%
\label{sec:discussion}

Round-trip consistency turns a structural property of one network
representing both temporal directions of a system into a practical
trust signal: an error meter that travels with the model, needs no
ground truth, ensembles, or governing equations, and costs one extra
rollout. The evidence is quantitative: the meter predicts held-out MHD
rollout error to within a factor of $1.14$ ($68\%$) with near-nominal
coverage, and a single bidirectional model approaches a ten-model
ensemble's accuracy at a tenth of its training cost
(Secs.~\ref{sec:exp-calibration} and~\ref{sec:exp-lepdeuq}).
Bidirectional diffusion models can, in a precise and calibrated sense,
predict their own rollout errors. The approach has clear limitations.
The check is necessary rather than sufficient: cancellation is
empirically rare ($95.1\%$ of held-out points within the calibrated
$\pm2\sigma$ band), but a small $\cyc{i}$ is evidence, not
certification, and Proposition~\ref{prop:sandwich} makes the
dependence on backward-map conditioning explicit. Backward dynamics of
strongly dissipative systems are ill-posed in the continuum; the
learned inverse remains informative at the horizons studied, but very
long horizons and stiff systems deserve dedicated study. And $\cyc{i}$
lives in latent space; relating it to physically calibrated
field-space metrics is a natural refinement. Because $\cyc{i}$ is a
measurable scalar cost, it can in principle be \emph{minimized} online
by gradient-free adaptive feedback \citep{scheinker2016bounded},
turning error detection into error correction at test time;
ongoing work beyond this paper's scope. The principle may also extend
to discrete sequence models, where time-reversed language models
already provide unsupervised feedback \citep{yerram2024trlm}; the
severe non-injectivity of language ($\mu\to0$ in
Sec.~\ref{sec:theory}) makes that a distinct and interesting regime. This method requires the two temporal directions to be learned independently; an architecture whose backward map is the exact algebraic inverse of its forward map closes every round trip by construction and measures nothing.

\section*{Acknowledgments}
This work was funded by Los Alamos National Laboratory and TPUs were provided by Google's TPU developer program.
% ---------------------------------------------------------------------
%%%%%%%%%%%%%%%%%%%%%%%%%%%%%%%%%%%%%%%%%%%%%%%%%%%%%%
%%%%%%%%%%%%%%%%%%%%%%%%%%%%%%%%%%%%%%%%%%%%%%%%%%%%%%
%%%%%%%%%%%%%%%%%%%%%%%%%%%%%%%%%%%%%%%%%%%%%%%%%%%%%%

%%%%%%%%%%%%%%%%%%%%%%%%%%%%%%%%%%%%%%%%%%%%%%%%%%%%%%
%%%%%%%%%%%%%%%%%%%%%%%%%%%%%%%%%%%%%%%%%%%%%%%%%%%%%%
%%%%%%%%%%%%%%%%%%%%%%%%%%%%%%%%%%%%%%%%%%%%%%%%%%%%%%
\bibliography{aaai27_references_v7}
%%%%%%%%%%%%%%%%%%%%%%%%%%%%%%%%%%%%%%%%%%%%%%%%%%%%%%
%%%%%%%%%%%%%%%%%%%%%%%%%%%%%%%%%%%%%%%%%%%%%%%%%%%%%%
%%%%%%%%%%%%%%%%%%%%%%%%%%%%%%%%%%%%%%%%%%%%%%%%%%%%%%

\setcounter{secnumdepth}{2} % kit default is 0, but this paper uses

\noindent This document contains the experimental configurations, additional
analyses, and full derivations supporting the main paper. Works cited only
here appear in this document's own reference list.

\clearpage

\appendix

\begin{center}
    \LARGE \bf Supplementary Materials: \\ Round-Trip Consistency: Bidirectional Diffusion Models Can Predict Their Own Rollout Errors \\ \vspace{0.5em}
    \large Appendix
\end{center}
\vspace{2em}

%%%%%%%%%%%%%%%%%%%%%%%%%%%%%%%%%%%%%%%%%%%%%%%%%%%%%%
%%%%%%%%%%%%%%%%%%%%%%%%%%%%%%%%%%%%%%%%%%%%%%%%%%%%%%
%%%%%%%%%%%%%%%%%%%%%%%%%%%%%%%%%%%%%%%%%%%%%%%%%%%%%%
\section{Experimental Configurations}
%%%%%%%%%%%%%%%%%%%%%%%%%%%%%%%%%%%%%%%%%%%%%%%%%%%%%%
%%%%%%%%%%%%%%%%%%%%%%%%%%%%%%%%%%%%%%%%%%%%%%%%%%%%%%
%%%%%%%%%%%%%%%%%%%%%%%%%%%%%%%%%%%%%%%%%%%%%%%%%%%%%%
\label{app:configs}

%%%%%%%%%%%%%%%%%%%%%%%%%%%%%%%%%%%%%%%%%%%%%%%%%%%%%%
\subsection{Embedding into Latent Space}
%%%%%%%%%%%%%%%%%%%%%%%%%%%%%%%%%%%%%%%%%%%%%%%%%%%%%%
\label{cVAE}

The first step for each of the data sets that we work with is to compress to a lower dimensional latent space by using a variational autoencoder (VAE) or a conditional variational autoencoder (cVAE) \cite{kingma2014vae, sohn2015cvae}. For the CelebV-HQ dataset, we utilize a pre-trained Stable Diffusion VAE \cite{rombach2022high}, whose architecture is inherited from a VQGAN approach \cite{esser2021taming}. The model was aquired using the diffusers library from Huggingface \cite{von-platen-etal-2022-diffusers}, class name AutoencoderKL, diffusers version 0.4.2 with silu activation functions and block out channels $[128, 256, 512, 512]$, 4 latent channels, 2 layers per residual block, and norm groups of size 32.

For all other data, each of the independent multi-channel inputs, such as the six fields of MHD, or the four fields of the turbulent radiative layer, a single cVAE is used for each field, guided by a one-hot encoded conditional vector via Feature-wise Linear Modulation (FiLM) \cite{perez2018film}. Each cVAE uses multiple residual blocks at each stage of compression and our scale/shift 
\begin{equation}
    \mathrm{FiLM}(\mathbf{x}) = \gamma(\mathbf{c}) \odot \mathbf{x} + \beta(\mathbf{c})
\end{equation}
is applied to the GroupNorm output inside residual blocks, as done in adaptive group normalization (AdaGN) used in modern diffusion U-Nets \cite{dhariwal2021diffusion}. Single-head spatial self-attention \cite{vaswani2017attention}, equivalent to a non-local block \cite{wang2018nonlocal}, is applied at the lowest resolution on both sides of the latent bottleneck. 

For The Well, the cVAE trained for 500 epochs, with a batch size of 64, with Adam, with a learning rate of $1e-4$, KL divergence weight of $\beta = 1e-3$, 16 latent channels, 32 input activations which were then multiplied by factors of $[2,4,8,16,32]$ through the encoder and reverse of that through the decoder, resulting in a latent space representation of size $4\times 12 \times 16$ of each input field image of size $128 \times 384 \times 1$, with 2 residual blocks at each resolution and attention at the pinch point with 4 heads, with a total parameter count of 240 M.

For the Navier Stokes / LE-PDE-UQ cVAE, it was trained for 500 epochs with a batch size of 32 and then an additional 500 epochs with a batch size of 8, each with Adam with a learning weight of $1e-4$, KL weight of $1e-4$, two residual blocks per resolution in the encoder and decoder, and 4 attention heads in the attention layers on either side of the pinch. The model has 16 input activations followed by 3 down and up stages whose activation numbers are 32, 64, and 128, with a latent channel size of 8, so that an input image of size $64\times 64 \times 1$ is encoded into a latent of size $8\times 8 \times 8$. The cVAE had a total parameter count of 3.9 M.

The MHD work was all done on a Google V6e TPU, all other work was done on NVIDIA H200 GPUs.

%%%%%%%%%%%%%%%%%%%%%%%%%%%%%%%%%%%%%%%%%%%%%%%%%%%%%%
\subsection{Denoiser Architecture and Conditioning}
%%%%%%%%%%%%%%%%%%%%%%%%%%%%%%%%%%%%%%%%%%%%%%%%%%%%%%
\label{app:denoiser}

All experiments use a denoising diffusion model \citep{ho2020denoising} with the same transformer denoiser setup; only the latent geometry $(H_\ell, W_\ell, F)$ changes between systems. The denoiser is a diffusion transformer (DiT) \cite{peebles2023scalable}: the noisy target frame, the $N$ nearest context frames, and an anchor frame are decomposed into $2\times2$ patches, linearly embedded, and concatenated into a single token sequence with learned positional embeddings, so that self-attention acts jointly over space and a short temporal window. The scalar conditions (diffusion timestep, the simulation-time index of the target, and a temporal direction flag) are fused into one conditioning vector (sinusoidal embeddings for the two timesteps, a learned two-entry embedding for the direction) that modulates every transformer block via adaLN-Zero \cite{peebles2023scalable}. The direction flag makes a single network bidirectional: it is trained on a mixture of forward and reverse-time examples in which the anchor is the initial condition $z_0$ when predicting $z_{t+1}$ from the $N$ frames ending at $z_t$ (direction $+1$), and the final state $z_{T-1}$ when predicting $z_{t-1}$ from the $N$ frames ending at $z_{t+1}$ (direction $-1$); context slots are ordered nearest-last, so each slot carries the same distance-to-target meaning in both directions. Training minimizes the standard $\epsilon$-prediction objective \cite{ho2020denoising} under a cosine noise schedule \cite{nichol2021improved} with min-SNR-$\gamma$ loss weighting ($\gamma = 5$) \cite{hang2023efficient}, and an exponential moving average of the parameters is retained for sampling. At inference, each frame is drawn with a deterministic DDIM sampler \cite{song2021denoising} and the model is rolled out autoregressively in either time direction, appending each generated frame to the sliding context window, in the spirit of autoregressive conditional diffusion emulators for turbulent flow \cite{kohl2026benchmarking}.

\paragraph{Tokenization.}
Each latent frame is a channels-last array of shape $H_\ell \times
W_\ell \times F$. The input to the denoiser is the ordered frame list
$[\,\text{noised target},\ \text{context}_1, \dots, \text{context}_n,\
\text{anchor}\,]$, with contexts ordered farthest-to-nearest (the
frame adjacent to the target is last) and the anchor fixed to $z_0$
for forward and $z_{T-1}$ for backward prediction; near a trajectory
boundary, missing context slots are padded with the anchor frame.
Every frame is split into non-overlapping $p \times p$ patches
($p{=}2$), giving $(H_\ell/p)(W_\ell/p)$ tokens of dimension $p^2 F$
per frame, each mapped by a shared linear layer to the transformer
width. The $(n{+}2)$ frames' tokens are concatenated along the
sequence axis, and a single learned positional embedding (initialized
$0.02\,\mathcal{N}(0,\mathbf{I})$) over the full sequence encodes both
spatial location and frame role, since each frame occupies a fixed
block of positions. Conditioning on past states therefore requires no
extra machinery: the target's tokens attend to the context and anchor
tokens through ordinary self-attention.

\paragraph{Scalar conditioning (adaLN-Zero).}
The three scalars: denoising step $k$, the target frame's simulation-time index $t$, and the direction flag $\cd \in \{+1,-1\}$ form a single conditioning vector
\begin{eqnarray}
c &=& \mathrm{MLP}\!\left( \mathrm{SiLU} \left [ W_k\,
\mathrm{sin}(k)\right ] + \mathrm{SiLU} \left [ W_t\, \mathrm{sin}(t)\right ] \right . \nonumber \\
&& \left .+ E_d[\mathds{1}_{\cd<0}] \right ) \in \mathbb{R}^{256},
\end{eqnarray}
where $\mathrm{sin}(\cdot)$ is a sinusoidal embedding and $E_d$ a learned two-entry table. Following adaLN-Zero \citep{peebles2023dit}, each transformer block applies pre-norm LayerNorm without affine parameters and derives, from $c$ through a zero-initialized linear map, six per-block vectors: shift, scale, and residual gate for the attention sublayer and again for the MLP sublayer. Zero initialization makes every block the identity at the start of training. The final layer applies a two-parameter adaLN modulation followed by a zero-initialized linear projection back to patch dimension; only the \emph{target} frame's tokens pass through it and are un-patchified into the noise prediction $\hat\epsilon \in \mathbb{R}^{H_\ell \times W_\ell \times F}$.
 
\paragraph{Block and size details.}
For the diffusion transformer, as describe in \ref{app:denoiser}, we used width 384, depth 12, 6 attention heads, MLP ratio 4 with GELU, conditioning width 256, patch size 2, fixed across all systems. For the face-video runs the attention uses qk-normalization (per-head LayerNorm applied to queries and keys before the dot product), which we found necessary to prevent attention-logit growth and the associated late-training divergence; the physics runs did not require it. Per system, with $n{=}2$ context frames the token sequence lengths are: MHD, $(H_\ell, W_\ell, F) = (16, 16, 24)$ (six fields, four latent channels each), $4 \times 64 = 256$ tokens; turbulent radiative layer, $(4, 12, F)$ with $F = 64$ ($4$ fields $\times$ $16$ latent channels), $4 \times 12 = 48$ tokens; face videos, $(32, 32, 4)$ (frozen Stable Diffusion VAE latents), $4 \times 256 = 1024$ tokens. Training uses the $\epsilon$-prediction objective of the main paper with min-SNR-$\gamma$ weighting ($\gamma{=}5$) and an exponential moving average of the weights (decay $0.999$); sampling uses deterministic DDIM. Optimization schedules and per-system data processing are given in the configuration appendices.

\paragraph{Parameter counts.}
For the CelebV-HQ diffusion model, it is working on latent images of size $32 \times 32 \times 4$ and this was the most challenging data set by far because of the multi-modal nature, with each videos actual dynamics independent of each other. For DDIM training, 1000 diffusion steps were used, and for sampling at inference time we used 30 steps. With a patch size of 2, hidden token dimension of 2048, 16 attention heads were used per transformer and the model was 16 transformers deep, the conditional embedding dimension used was 512, and the total parameter count was 912 million parameters.

For the MHD data, for DDIM training, 1000 diffusion steps were used, and for sampling at inference time we used 50 steps. The diffusion model had a depth of 12 with 6 attention heads each, a conditional dimension of 256, hidden token dimension of 384, conditional embedding dimension of 256, resulting in a total parameter count of 29 million parameters. 

For the Well data, the same setup was used as for MHD with 200 diffusion steps for DDIM training and 25 steps for sampling. The diffusion model had a depth of 12 with 6 attention heads each, a conditional dimension of 256, hidden token dimension of 384, conditional embedding dimension of 256, resulting in a total parameter count of 29 million parameters. 

For the Navier-Stokes data, the same setup was used as for MHD and Well with 200 diffusion steps for DDIM training and 25 steps for sampling. The diffusion model had a depth of 12 with 6 attention heads each, a conditional dimension of 256, hidden token dimension of 384, conditional embedding dimension of 256, resulting in a total parameter count of 29 million parameters.

%%%%%%%%%%%%%%%%%%%%%%%%%%%%%%%%%%%%%%%%%%%%%%%%%%%%%%
\subsection{2D Compressible MHD}
%%%%%%%%%%%%%%%%%%%%%%%%%%%%%%%%%%%%%%%%%%%%%%%%%%%%%%
\label{app:mhd}

For modeling 2D MHD, we utilize the tool and approach of P. Mocz \cite{mocz2014constrained}, which considers the conservation law dynamics
\begin{equation}
    \frac{\partial \mathbf{U}}{\partial t} + \nabla \cdot \mathbf{F}(\mathbf{U}) = 0,
\end{equation}
\begin{equation}
\mathbf{U} = \begin{pmatrix}
\rho \\
\rho\mathbf{v} \\
\rho e \\
\mathbf{B}
\end{pmatrix}, 
\quad 
\mathbf{F}(\mathbf{U}) = \begin{pmatrix}
\rho \mathbf{v} \\
\rho\mathbf{v}\mathbf{v}^T + p - \mathbf{B}\mathbf{B}^T\\
\rho e \mathbf{v} + p\mathbf{v} - \mathbf{B}(\mathbf{v}\cdot \mathbf{B}) \\
\mathbf{B} \mathbf{v}^T - \mathbf{v}\mathbf{B}^T,
\end{pmatrix}
\end{equation}
where $p=p_\mathrm{gas} + \frac{1}{2}\mathbf{B}^2$ is the total gas pressure, $e=u+\frac{1}{2}\mathbf{v}^2+\frac{1}{2\rho}\mathbf{B}^2$ is the total energy per unit mass, and $u$ is the thermal energy per unit mass.

For MHD, at each time step one MHD state is represented by six fields ($\rho, P, v_x, v_y, B_x, B_y$), having size $512\times 512 \times 6$. We split that into 6 single-channel inputs of size $512 \times 512 \times 1$ and then encode the first field with a conditional vector $[1,0,0,0,0,0]$, the second with $[0,1,0,0,0,0]$, $\dots$, and the sixth with $[0,0,0,0,0,1]$, with FiLM conditioning on each residual block of both the encoder and decoder, as described in \ref{cVAE}. Each channel is encoded into a $16{\times}16{\times}4$ latent ($\beta=10^{-3}$). The number of weights in the encoder's down-sampling layers of the cVAE are $[32, 64, 128, 256, 512]$, which is mirrored for the decoder, with 1 residual block per resolution and an 8 headed attention block on either side of the latent bottleneck. With this configuration our cVAE model has 192 million parameters. We use 500 train and 50 test trajectories where each trajectory consists of 100 steps over 1 second, as generated by using a spatial resolution of $512\times512$, with 100 steps over 1\,s via the constrained transport code \citep{mocz2014constrained}. The random vector potential generating the initial magnetic fields
is a mixture of 100 Gaussians. The model is trained for 1M steps with a batch size of 16. Latent roll out error vs round-trip consistency for the MHD data is shown, colored by turnaround depth, in Fig. \ref{fig:corr-scatter}. Examples of forward and backward autoregressive MHD field rollouts are shown in Fig. \ref{fig:inv-fields}.

\begin{figure}[t]
\centering
\includegraphics[width=0.9\columnwidth]{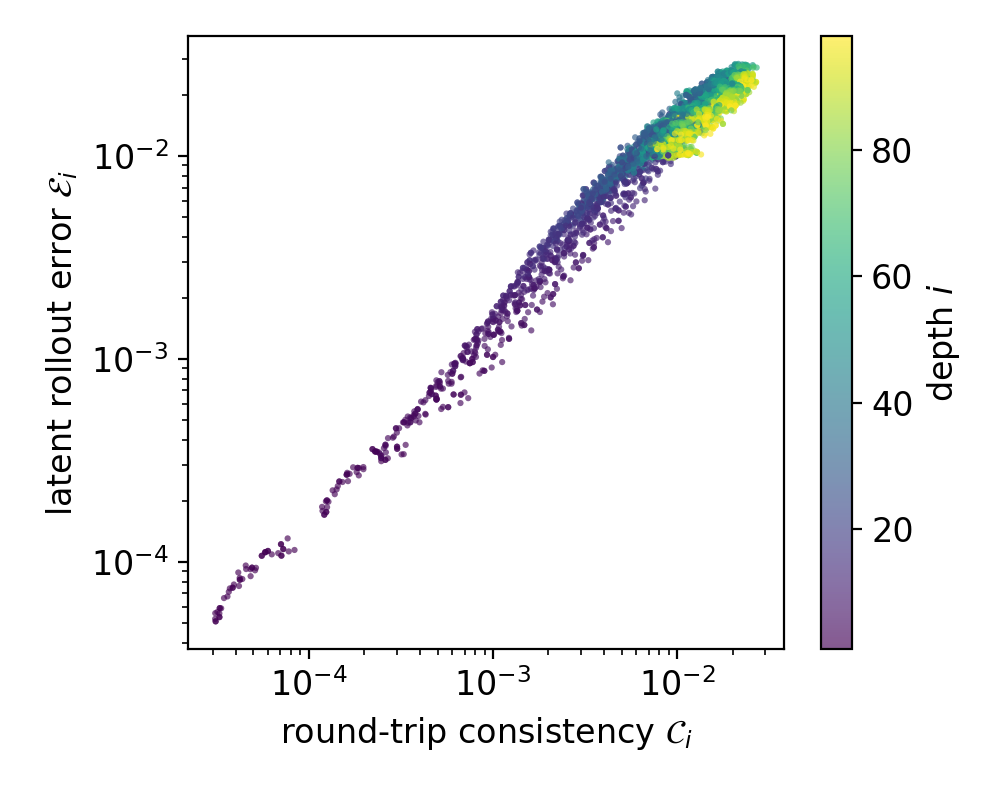}
\caption{Round-trip consistency $\cyc{i}$ vs.\ true latent rollout
error $\err{i}$ on the 50 held-out MHD trajectories, colored by depth $i$ (turnaround-anchored cycles; deterministic sampler). Fixed-depth Spearman is $0.91$--$0.98$.}
\label{fig:corr-scatter}
\end{figure}

%%%%%%%%%%%%%%%%%%%%%%%%%%%%%%%%%%%%%%%%%%%%%%%%%%%%%%
\subsection{External 2D Benchmark: Turbulent Radiative Layer (the Well)}
%%%%%%%%%%%%%%%%%%%%%%%%%%%%%%%%%%%%%%%%%%%%%%%%%%%%%%
\label{app:well}
\texttt{turbulent\_radiative\_layer\_2D} from the Well
\citep{ohana2024well}: an astrophysical radiative turbulent mixing
layer (Kelvin--Helmholtz instability between hot dilute and cold dense
gas with radiative cooling); 90 trajectories (10 seeds $\times$ 9
cooling times) of 101 steps at $128\times384$, four fields (density,
pressure, $v_x$, $v_y$), periodic in $x$ and zero-gradient in $y$. We
use the official train/valid/test split and report the benchmark's
standard VRMSE alongside our metrics for anchoring. Each field is
encoded by its own VAE with the recipe of Appendix~\ref{app:mhd},
yielding $4{\times}12{\times}4$ latents per field (rectangular latent
grids; no square-shape assumptions); the bidirectional denoiser
configuration is unchanged unless noted. 

\paragraph{Comparison to The Well baselines.}
Following the evaluation protocol of \citet{ohana2024well}, we report the
variance-scaled root mean squared error,
\begin{equation}
\mathrm{VRMSE}\left(\hat{u}, u\right) \;=\;
\sqrt{\frac{\big\langle (\hat{u}-u)^{2} \big\rangle_{\Omega}}
     {\big\langle (u-\langle u\rangle_{\Omega})^{2} \big\rangle_{\Omega} + \varepsilon}},
\qquad \varepsilon = 10^{-7},
\end{equation}
where $\langle\cdot\rangle_{\Omega}$ is the spatial mean over the $128\times384$
grid and the denominator uses the unbiased spatial variance, matching the
reference implementation. VRMSE is computed per snapshot and per field in
physical units---latent predictions are first decoded with the cVAE---and then
averaged, so a score of $1$ corresponds to predicting the spatial mean of the
target field. On single-step prediction for the
\texttt{turbulent\_radiative\_layer\_2D} test split, with each frame predicted
from the two preceding ground-truth frames, our model attains
$\mathrm{VRMSE}=0.59$ (density $0.21$, pressure $1.32$, $v_x$ $0.30$,
$v_y$ $0.54$), versus $0.50$ (FNO), $0.50$ (TFNO), $0.24$ (U-net), and $0.20$
(CNextU-net) for the baselines of \citet{ohana2024well}. Three protocol
differences favor the baselines: they condition on four past states rather
than our two; they operate at full spatial resolution, whereas our predictions
pass through a $64\times$-compressed latent bottleneck, cVAE reconstruction
of the ground truth alone yields $\mathrm{VRMSE}=0.30$ on this split ($0.67$
on pressure), an irreducible floor in this latent space that already exceeds
the total one-step error of the two strongest baselines---and they are
deterministic regressors, whereas we report a single stochastic sample.
Consistent with the per-field analysis of \citet{ohana2024well}, error is
concentrated in the pressure field; notably, this holds for the
reconstruction floor as well, and the one-step error is roughly twice the
floor in every field, indicating that the pressure field on this dataset is
intrinsically difficult both to compress and to predict. We emphasize that
the model was not tuned for single-step accuracy: our contribution concerns
forward--backward rollout consistency, for which competitive (rather than
state-of-the-art) single-step fidelity is enough. Over autoregressive
rollouts, seeded and conditioned throughout on a two-frame context, our
time-windowed VRMSE is $1.12$ over steps 6--12 and $2.86$ over steps 13--30,
versus $0.54$/$1.01$ for CNextU-net, $0.66$/$1.04$ for U-net, $1.79$/$3.54$
for FNO, and $6.01$/${>}10$ for TFNO: despite the latent bottleneck and
shorter context, our model outperforms both spectral baselines at every
horizon while trailing the convolutional ones. We note that pointwise VRMSE
structurally favors deterministic regressors at long horizons on chaotic
dynamics: as predictability is lost, a regressor reverting to the conditional
mean approaches a score of $1$, whereas an emulator drawing a sharp sample
that is statistically consistent with, but decorrelated from, the reference
incurs $\mathrm{VRMSE}\approx\sqrt{2}$ even in the ideal case. Part of the
long-horizon gap therefore reflects the metric's preference for
mean-reverting predictions rather than sample fidelity per se.

%%%%%%%%%%%%%%%%%%%%%%%%%%%%%%%%%%%%%%%%%%%%%%%%%%%%%%
\subsection{Navier-Stokes}
%%%%%%%%%%%%%%%%%%%%%%%%%%%%%%%%%%%%%%%%%%%%%%%%%%%%%%
\label{app:NavStok}
For the 2D turbulent radiative layer data, the governing equations are
\begin{eqnarray}
    \frac{\partial \rho}{\partial t} &=& - \nabla\cdot(\rho v), \\
    \frac{\partial \rho v}{\partial t} &=& - \nabla \cdot (\rho v v + P ), \\
    \frac{\partial E}{\partial t} &=& - \nabla \cdot ((E+P)v) - \frac{E}{t_{\mathrm{cool}}}, \\
    E &=& \frac{P}{\gamma-1}, \qquad \gamma = \frac{5}{3},
\end{eqnarray}
where $\rho$ is density, $v$ is 2D velocity, $P$ is pressure, $E$ is total energy, and $t_{\mathrm{cool}}$ the cooling time.

%%%%%%%%%%%%%%%%%%%%%%%%%%%%%%%%%%%%%%%%%%%%%%%%%%%%%%
\subsection{CelebV-HQ Face Videos}
%%%%%%%%%%%%%%%%%%%%%%%%%%%%%%%%%%%%%%%%%%%%%%%%%%%%%%
\label{app:faces}

\paragraph{Evaluation.} $\err{i}$ as latent-space MSE on
identity-disjoint test clips (decoded LPIPS as a secondary check);
$\cyc{i}$ as defined in the main paper, with the deterministic sampler
(S-averaged variant ablated). Some examples of autoregressive CELEBV-HQ rollouts and their $\cyc{i}$ values are shown in Fig. \ref{fig:CELEBVHQ_Rolls} and Fig. \ref{fig:CELEBVHQ_Rolls_More}.

%%%%%%%%%%%%%%%%%%%%%%%%%%%%%%%%%%%%%%%%%%%%%%%%%%%%%%
\section{Why the Per-Trajectory Signal Is Strong on MHD and Marginal on the Radiative Layer: A Variance Decomposition}
%%%%%%%%%%%%%%%%%%%%%%%%%%%%%%%%%%%%%%%%%%%%%%%%%%%%%%
\label{app:heterogeneity}

For each system we remove the depth trend from the log quantities and split the residual variance into a between-trajectory component (a per-trajectory difficulty factor) and within-trajectory fluctuation ($\sigma_b$, $\sigma_w$ below; ICC $= \sigma_b^2/(\sigma_b^2{+}\sigma_w^2)$), on the training ensembles ($n{=}100$ MHD, $n{=}72$ radiative layer):

\begin{center}
\footnotesize
\setlength{\tabcolsep}{4pt}
\begin{tabular}{llcccc}
\toprule
System & & $\sigma_b$ & $\sigma_w$ & ICC & rel. \\
\midrule
MHD         & $\err{}$ & $0.250$ & $0.074$ & $0.92$ & $0.96$ \\
            & $\cyc{}$ & $0.349$ & $0.118$ & $0.90$ & $0.95$ \\
Rad.\ layer & $\err{}$ & $0.104$ & $0.191$ & $0.23$ & $0.48$ \\
            & $\cyc{}$ & $0.057$ & $0.220$ & $0.06$ & $0.25$ \\
\bottomrule
\end{tabular}
\end{center}

Here reliability (rel.) $= \sigma_b / \sqrt{\sigma_b^2 + \sigma_w^2}$ is the correlation between a single-depth reading and the trajectory factor it estimates. Both systems possess a per-trajectory factor; its magnitude differs ($\pm 28\%$ vs.\ $\pm 11\%$ multiplicative), and so does the metric noise it must be read through.

\paragraph{Attenuation ceiling.} If the underlying factors of
$\cyc{}$ and $\err{}$ are correlated at $\rho_f$, the observable
fixed-depth correlation across trajectories is bounded by classical
measurement-error attenuation,
\begin{eqnarray}
\rho_{\mathrm{obs}} &=& \rho_f \cdot
\frac{\sigma_b^{\cyc{}}}{\sqrt{(\sigma_b^{\cyc{}})^2 +
(\sigma_w^{\cyc{}})^2}} \cdot
\frac{\sigma_b^{\err{}}}{\sqrt{(\sigma_b^{\err{}})^2 +
(\sigma_w^{\err{}})^2}} \nonumber \\
&\le& 0.91 \text{ (MHD)}, \quad 0.12 \text{ (rad.\ layer)}. \nonumber
\end{eqnarray}
The radiative-layer measurement matches its ceiling ($0.11$ at
$n{=}80$, averaged over probed depths), implying $\rho_f \approx 1$;
independently, the MHD calibrator's residual reduction ($0.34 \to
0.13$ log-units, $\approx 85\%$ of residual variance) back-solves to
$\rho_f \approx 0.96$--$1.0$. A single model therefore explains both
systems: \emph{the trajectory factors of $\cyc{}$ and $\err{}$ are
essentially perfectly coupled; what varies is whether the factor is
large enough, relative to the per-depth noise of a single cycle, to be
resolved}, trivially on MHD, marginally on the radiative
layer. The measured MHD fixed-depth Spearman ($0.91$-$0.98$ across
probed depths, $n{=}50$) sits at and slightly \emph{above} the $0.91$
ceiling: the static-factor model treats the within-trajectory
fluctuations of $\cyc{}$ and $\err{}$ as independent, but they share
the forward rollout leg and co-move, so the ceiling is conservative
where the model is accurate; at $i{=}80$, where error saturation
decouples the legs, the measurement ($0.85$ train / $0.91$ val)
matches the bound. Consistently, the pointwise $\mu(\cyc{})$
calibrator trails the identically fit depth-only calibrator by $0.78$
nats on the radiative layer's held-out test trajectories (with the
trajectory factor unresolvable, $\cyc{}$ acts only as a noisy depth
proxy) and $\cyc{i}$ retains only a weak within-rollout association
there (Spearman $0.36 \pm 0.13$, train). Averaging $\log \cyc{}$ over depths raises its reliability (to $\approx 0.93$ even on the radiative layer), capping the recoverable trajectory-mean gain there at $\approx +0.11$ nats; the measured gain on the nine-trajectory test split is $-0.008 \pm 0.005$, the nine calibration trajectories are too few to estimate a $\pm 11\%$ factor, while the same construction earns $+0.26$ nats on MHD. Incidentally, $\cyc{}$'s between-std exceeds $\err{}$'s on MHD: the round trip compounds forward and backward factors, making it a slightly amplified reading of trajectory difficulty.

%%%%%%%%%%%%%%%%%%%%%%%%%%%%%%%%%%%%%%%%%%%%%%%%%%%%%%
\section{Detailed Derivations for the Sandwich Bound}
%%%%%%%%%%%%%%%%%%%%%%%%%%%%%%%%%%%%%%%%%%%%%%%%%%%%%%
\label{app:theory-details}

\paragraph{Notation, written out.}
Each latent frame is $\zb\in\mathbb{R}^{n}$ (the flattened latent).
The pair state stacks two consecutive frames,
$\sv_k=(\zb_{t+k-1},\zb_{t+k})\in\mathbb{R}^{2n}$, and
$\pnorm{\sv}=\big(\tfrac{1}{2n}\|\sv\|_2^2\big)^{1/2}$ is the RMS over
all $2n$ entries, so that for any pair difference
\begin{eqnarray}
\pnorm{(\,a_1,a_2)-(b_1,b_2)}^2
&=& \tfrac{1}{2n}\big(\|a_1-b_1\|_2^2+\|a_2-b_2\|_2^2\big) \nonumber \\
&=& \tfrac12\big(\MSE(a_1,b_1) \nonumber \\
&& +\MSE(a_2,b_2)\big).
\end{eqnarray}
Applied to $(\sv_0,\tilde\sv_0)$ this is exactly the definition of
$\cyc{i}$ in the main paper:
the $\tfrac{1}{2n}$ normalization is chosen so pair-level and
frame-level (MSE) quantities interconvert without stray constants.
With a deterministic sampler, one model step is the shift-and-append
map on pairs,
\begin{equation}
\fwd(a,b)=\big(b,\,f^{+}_\theta(a,b)\big),\qquad
\bwd(a,b)=\big(f^{-}_\theta(a,b),\,a\big),
\label{eq:shiftappend}
\end{equation}
where $f^{\pm}_\theta$ is one full DDIM generation conditioned on the
pair (and the anchor and scalars, suppressed). The forward rollout is
$\hat\sv_i=\fwd^{\,i}(\sv_0)$, whose components we write
$\hat\sv_i=(\hat\zb_{t+i-1},\hat\zb_{t+i})$; by
\eqref{eq:shiftappend}, the older component of $\hat\sv_i$ equals the
newer component of $\hat\sv_{i-1}$, both are the same chain's
prediction of frame $t{+}i{-}1$.

\begin{lemma}[Pair error is the average of two frame errors]
\label{lem:pairframe}
With $\err{k}=\MSE(\zb_{t+k},\hat\zb_{t+k})$ and the convention
$\err{0}:=0$ (the seed is exact),
\begin{equation}
\perr{i}\;=\;\tfrac12\big(\err{i-1}+\err{i}\big),
\ \text{hence} \ 
\tfrac12\,\err{i}\;\le\;\perr{i}
\ \text{and} \
\err{i}\;\le\;2\,\perr{i},
\label{eq:pairframe}
\end{equation}
with $\err{i}=2\,\perr{i}$ iff $\err{i-1}=0$ (e.g., at $i{=}1$, where
the older slot of $\hat\sv_1$ is still the true seed $\zb_t$).
\end{lemma}

\begin{proof}
$\perr{i}=\pnorm{\sv_i-\hat\sv_i}^2
=\tfrac{1}{2n}\big(\|\zb_{t+i-1}-\hat\zb_{t+i-1}\|_2^2
+\|\zb_{t+i}-\hat\zb_{t+i}\|_2^2\big)
=\tfrac12(\err{i-1}+\err{i})$, using that the older component of
$\hat\sv_i$ is the chain's frame-$(t{+}i{-}1)$ prediction. Dropping
the nonnegative $\err{i-1}$ gives the bounds; the equality condition
is immediate.
\end{proof}

\paragraph{Statements recalled from the main paper.}
For self-containedness we restate the objects whose proofs are completed
here (numbering refers to the main paper). Assumption~1 posits constants
$0<\mu\le L<\infty$ such that, on the set $\mathcal{D}$ containing the two
backward trajectories, all pairs $a,b$ satisfy
\begin{equation}
  \mu\,\pnorm{a-b}\;\le\;\pnorm{\bwd(a)-\bwd(b)}\;\le\;L\,\pnorm{a-b},
  \label{eq:bilip-s}
\end{equation}
with the backward residual $\delta_i:=\pnorm{\bwd^{\,i}(\sv_i)-\sv_0}$.
Proposition~1 (the sandwich bound) states
\begin{equation}
  \Big(\max\big\{\mu^{\,i}\sqrt{\perr{i}}-\delta_i,\,0\big\}\Big)^{2}
  \;\le\;\cyc{i}\;\le\;
  \Big(L^{\,i}\sqrt{\perr{i}}+\delta_i\Big)^{2},
  \label{eq:sandwich-s}
\end{equation}
and Corollary~1 (the certificate) states
\begin{equation}
  \sqrt{\err{i}}\;\le\;\sqrt{2}\;\mu^{-i}\Big(\sqrt{\cyc{i}}+\delta_i\Big).
  \label{eq:certificate-s}
\end{equation}

\paragraph{The assumption, unpacked.}
Assumption~1 of the main paper, recalled in \eqref{eq:bilip-s},
constrains the backward step $\bwd$ only on
the set $\mathcal{D}$ of pairs actually visited by the two backward
trajectories in play from the true terminal pair and from the
predicted one not globally. The upper inequality is ordinary
Lipschitz continuity: one backward step cannot amplify a pair
discrepancy by more than $L$. The lower inequality (co-Lipschitz, or
expansiveness) says one backward step cannot \emph{shrink} a
discrepancy below a factor $\mu>0$. For differentiable $\bwd$ this is
equivalent to the singular values of its Jacobian lying in $[\mu,L]$
along the visited segment, and it is precisely the anti-cancellation
condition: cancellation means two different terminal pairs (true and
predicted) get mapped backward onto nearly the same returned seed,
which is exactly a collapse of distance that $\mu>0$ forbids. If the
learned backward map does collapse ($\mu\to0$), the certificate
weakens honestly: the bound \eqref{eq:certificate-s} blows up rather
than silently failing. The residual
$\delta_i=\pnorm{\bwd^{\,i}(\sv_i)-\sv_0}$ isolates the backward
model's \emph{own} inaccuracy on clean inputs: it is what the
round-trip would report even for a perfect forward rollout, i.e., the
detector's noise floor.

\begin{lemma}[$i$-step separation propagation]
\label{lem:iterate}
Let $u_k:=\bwd^{\,k}(\sv_i)$ and $v_k:=\bwd^{\,k}(\hat\sv_i)$ for
$k=0,\dots,i$, so $u_0=\sv_i$, $v_0=\hat\sv_i$, and
$v_i=\tilde\sv_0$. Under the bi-Lipschitz condition \eqref{eq:bilip-s},
for $k=0,\dots,i$,
\begin{equation}
\mu^{\,k}\,\pnorm{\sv_i-\hat\sv_i}\;\le\;\pnorm{u_k-v_k}\;\le\;
L^{\,k}\,\pnorm{\sv_i-\hat\sv_i}.
\label{eq:iterk}
\end{equation}
\end{lemma}

\begin{proof}
Induction on $k$. Base $k=0$: both sides equal
$\pnorm{\sv_i-\hat\sv_i}$. Step: $u_k,v_k\in\mathcal{D}$ by the
definition of $\mathcal{D}$ (it contains both backward trajectories),
so \eqref{eq:bilip-s} applies with $a=u_k$, $b=v_k$:
$\pnorm{u_{k+1}-v_{k+1}}=\pnorm{\bwd(u_k)-\bwd(v_k)}
\in\big[\mu\,\pnorm{u_k-v_k},\,L\,\pnorm{u_k-v_k}\big]$; multiplying
the inductive bounds by $\mu$ and $L$ respectively gives
\eqref{eq:iterk} at $k{+}1$.
\end{proof}

\paragraph{Proof of Proposition~1 of the main paper, in full.}
Set $k=i$ in Lemma~\ref{lem:iterate} and abbreviate
$P:=\sqrt{\perr{i}}=\pnorm{\sv_i-\hat\sv_i}$:
\begin{equation}
\mu^{\,i}P\;\le\;\pnorm{u_i-v_i}\;\le\;L^{\,i}P.
\label{eq:sep-at-i}
\end{equation}
The quantity of interest is
$\sqrt{\cyc{i}}=\pnorm{\sv_0-v_i}$ (the returned seed's distance to
the true seed). Insert $u_i$ and use the triangle inequality both
ways, together with $\pnorm{\sv_0-u_i}=\delta_i$ by definition:
\begin{align}
\pnorm{\sv_0-v_i}
&\le \pnorm{\sv_0-u_i}+\pnorm{u_i-v_i}
\;\le\; \delta_i + L^{\,i}P,
\label{eq:upper-branch}\\
\pnorm{\sv_0-v_i}
&\ge \pnorm{u_i-v_i}-\pnorm{u_i-\sv_0}
\;\ge\; \mu^{\,i}P-\delta_i.
\label{eq:lower-branch}
\end{align}
Since $\pnorm{\sv_0-v_i}\ge0$ always, \eqref{eq:lower-branch}
strengthens to
$\pnorm{\sv_0-v_i}\ge\max\{\mu^{\,i}P-\delta_i,\,0\}$. Both outer
bounds are now nonnegative, so squaring preserves the ordering and
yields \eqref{eq:sandwich-s}. \hfill$\square$

\paragraph{Proof of Corollary~1 of the main paper, in full.}
Inequality \eqref{eq:lower-branch} \emph{before} clipping rearranges,
for every case (whether or not the max clipped to zero), to
\begin{equation}
\mu^{\,i}\sqrt{\perr{i}}\;\le\;\sqrt{\cyc{i}}+\delta_i
\quad\Longrightarrow\quad
\sqrt{\perr{i}}\;\le\;\mu^{-i}\big(\sqrt{\cyc{i}}+\delta_i\big).
\end{equation}
Combining with Lemma~\ref{lem:pairframe},
$\sqrt{\err{i}}\le\sqrt{2\,\perr{i}}
=\sqrt{2}\,\sqrt{\perr{i}}
\le\sqrt{2}\,\mu^{-i}\big(\sqrt{\cyc{i}}+\delta_i\big)$,
which is \eqref{eq:certificate-s}. \hfill$\square$

\paragraph{Remarks on the constants.}
(i)~\emph{Everything is estimable offline.} $\delta_i$ is measured
directly on held-out data (backward rollouts from true terminal
pairs; reported in the experiments). $\mu$ and $L$ along visited
pairs are the extreme singular values of the backward step's
Jacobian, estimable by Jacobian--vector products or finite-difference
probes $\pnorm{\bwd(a+\epsilon r)-\bwd(a)}/(\epsilon\pnorm{r})$ over
random directions $r$ at pairs $a$ sampled from backward rollouts, these are model-only quantities requiring no test-time ground truth. See
\citet{virmaux2018lipschitz} for network-wide Lipschitz estimation.
(ii)~\emph{Tightness.} When $\delta_i\to0$ the sandwich pins the
ratio $\cyc{i}/\perr{i}$ into $[\mu^{2i},L^{2i}]$; if additionally
$\mu\approx L\approx1$ (the backward step nearly preserves pair
distances along the attractor), $\cyc{i}$ tracks $\perr{i}$ almost
proportionally, this is the regime the strong empirical correlations
suggest the trained models occupy.
(iii)~\emph{Honest looseness.} The window $[\mu^{2i},L^{2i}]$ widens
exponentially in $i$ whenever $\mu<1<L$, so the worst-case guarantee
degrades with depth even though the empirical relationship stays
tight; the calibrator of the main paper is the
practical instrument, with \eqref{eq:certificate-s} explaining
\emph{when} it can be trusted and \eqref{eq:sandwich-s}'s lower branch
explaining when it cannot (cancellation, $\mu\to0$).
(iv)~\emph{Certified constants.} Architectures with certified
bi-Lipschitz bounds, such as invertible residual networks
\citep{behrmann2019invertible}, offer a route to guaranteed $(\mu,L)$
rather than estimated ones.
(v)~\emph{Stochastic sampling.} With a stochastic sampler,
Proposition~1 of the main paper applies verbatim to the deterministic
probability-flow sampler, or to $S$-averaged cycles
$\bar{\mathcal{C}}_i=\tfrac1S\sum_s \mathcal{C}_i^{(s)}$ with the
additional sampling-variance term absorbed into $\delta_i$.

%%%%%%%%%%%%%%%%%%%%%%%%%%%%%%%%%%%%%%%%%%%%%%%%%%%%%%
\section{Why Bidirectional Training Can Beat Specialist Training:
A Linear Analysis}
%%%%%%%%%%%%%%%%%%%%%%%%%%%%%%%%%%%%%%%%%%%%%%%%%%%%%%
\label{app:bidir-theory}
 
The ablation in the main paper finds that one
bidirectional model outperforms direction specialists \emph{on each
specialist's own task} at matched compute. (Protocol of the measured
ablation, for reference: the identical architecture is trained
forward-only, backward-only, and bidirectionally with
$p_{\mathrm{bwd}}{=}0.5$ on the turbulent radiative layer, 80 training
trajectories, at matched compute; all three checkpoints are evaluated
direction-matched with fixed noise, 10 seeds, mean $\pm$ SE.) Here we isolate the
mechanism in the simplest setting that exhibits it: stationary linear
dynamics with a shared linear hypothesis.
 
\paragraph{Setting.}
Let $\zb_t\in\mathbb{R}^d$ be a zero-mean stationary Gaussian process
with one-step dynamics $\zb_{t+1}=A\zb_t+\varepsilon_t$,
$\varepsilon_t\sim\mathcal{N}(0,Q)$ independent of $\zb_t$, and
stationary covariance $\Sigma=\mathbb{E}[\zb_t\zb_t^{\!\top}]$
satisfying $\Sigma=A\Sigma A^{\!\top}+Q$.
 
\begin{lemma}[Time reversal is a reparameterization, not a new task]
\label{lem:reversal}
The reverse-time conditional is linear with
\begin{equation}
\mathbb{E}[\zb_t\mid \zb_{t+1}]
=A_{-}\,\zb_{t+1},\qquad
A_{-}=\Sigma A^{\!\top}\Sigma^{-1},
\label{eq:reversal}
\end{equation}
and residual covariance
$Q_{-}=\Sigma-\Sigma A^{\!\top}\Sigma^{-1}A\Sigma\succeq0$. In
particular, for whitened latents ($\Sigma=I$, which per-channel
standardization approximately enforces), $A_{-}=A^{\!\top}$ exactly:
the forward and backward prediction tasks share one parameter object,
and a single network queried with a direction flag is a
\emph{correctly specified} weight tying.
\end{lemma}
 
\begin{proof}
$(\zb_t,\zb_{t+1})$ is jointly Gaussian with
$\mathrm{Cov}(\zb_{t+1},\zb_t)=A\Sigma$, hence
$\mathrm{Cov}(\zb_t,\zb_{t+1})=\Sigma A^{\!\top}$. The Gaussian
conditional-mean formula gives
$\mathbb{E}[\zb_t\mid\zb_{t+1}]
=\mathrm{Cov}(\zb_t,\zb_{t+1})\,\mathrm{Cov}(\zb_{t+1})^{-1}\zb_{t+1}
=\Sigma A^{\!\top}\Sigma^{-1}\zb_{t+1}$, and the conditional
covariance is
$\Sigma-\Sigma A^{\!\top}\Sigma^{-1}A\Sigma$, positive semidefinite
as a Gaussian conditional covariance. With $\Sigma=I$ the stationarity
identity reads $AA^{\!\top}+Q=I$, so all singular values of $A$ are at
most one and $Q_{-}=I-A^{\!\top}A\succeq0$ directly.
\end{proof}
 
\paragraph{Estimation with and without tying.}
Lemma~\ref{lem:reversal} converts the training question into: the same matrix $A$ is observable through two regression views: forward pairs, in which it acts as $A$, and backward pairs, in which it acts as $A^{\!\top}$ (taking $\Sigma=I$). A direction specialist uses one view; the bidirectional model, both. We compare them in the standard fixed-design idealization.
 
\begin{proposition}[Tying halves the estimation variance]
\label{prop:tying}
Suppose $A$ is estimated by least squares from $m_+$ forward
observations $y_k=Ax_k+\varepsilon_k$ and $m_-$ backward observations
$v_k=A^{\!\top}u_k+\varepsilon'_k$, with ($\mathbb{E}[xx^{\!\top}]=\mathbb{E}[uu^{\!\top}]=I$) and independent noise of equal scale $\sigma^2$ per coordinate. Then the expected squared estimation error obeys
\begin{equation}
\mathbb{E}\,\|\hat A-A\|_F^{2}
\;=\;\frac{\sigma^{2}d^{2}}{m_{+}+m_{-}}\big(1+o(1)\big).
\label{eq:risk}
\end{equation}
At matched total sample budget $m$, the specialist ($m_+=m$,
$m_-=0$) attains $\sigma^2d^2/m$, while the tied bidirectional
estimator ($m_+=m_-=m/2$) attains the \emph{same}
$\sigma^2d^2/m$ on the union of views, but per parameter of the
deployed system it has solved both directions with one $d\times d$
matrix, where the specialists require two. Equivalently, at matched
\emph{per-direction} data $n$ (each trajectory supplies both views),
the tied estimator's error is $\sigma^2d^2/(2n)$: half the
specialist's $\sigma^2d^2/n$.
\end{proposition}
 
\begin{proof}
Vectorize: each forward observation contributes $d$ scalar equations
in the entries of $A$ with isotropic regressors; each backward
observation contributes $d$ scalar equations in the entries of
$A^{\!\top}$, i.e., in the \emph{same} unknowns under a fixed
permutation. Stacking gives a least-squares problem for
$\mathrm{vec}(A)\in\mathbb{R}^{d^2}$ with
$(m_++m_-)d$ equations and asymptotically isotropic design;
the standard OLS risk $\sigma^2\times(\text{\#params})/(\text{\#obs
rows})\times d = \sigma^2 d^2/(m_++m_-)$ follows. The two
readings of the budget give the two displayed comparisons.
\end{proof}
 
\paragraph{Remarks.}
(i)~\emph{Where the measured effect lives.} The gain is a
\emph{variance} reduction, so it scales as $1/n$ and vanishes with
abundant data. The mechanism predicts the benefit is largest in
the low-data regime, which is exactly where the ablation measures it
(80 trajectories), and variance reduction is regularization, matching
the observed delayed overfitting (epoch ${\sim}250$ vs.\
${\sim}180$). (ii)~\emph{Why the idealization overstates the
constant.} The two views of a trajectory share its underlying
randomness, so the independent-noise factor of two is an upper bound
on the mechanism; the measured $7$-$10\%$ per-direction improvement
is consistent with partially redundant views plus optimization
effects. (iii)~\emph{Off-direction failure of specialists.} The
converse is immediate: a forward specialist's objective never
constrains the transpose view, so nothing anchors its behavior under
the backward query, which is consistent with the $5$-$7\times$
off-direction degradation measured. (iv)~\emph{Beyond linear.} For
the nonlinear denoiser the tying acts through a shared representation
rather than a shared matrix; multitask representation bounds
\citep{baxter2000model,maurer2016benefit} give the analogous
complexity-per-task reduction when tasks share structure, which
Lemma~\ref{lem:reversal} identifies here as near-exact rather than
merely assumed. What is specific to our setting is that the same tying that improves the forward model is the structure that makes the round-trip meter of the main paper available at all.

%%%%%%%%%%%%%%%%%%%%%%%%%%%%%%%%%%%%%%%%%%%%%%%%%%%%%%
\section{Additional Results and Protocol Details}
%%%%%%%%%%%%%%%%%%%%%%%%%%%%%%%%%%%%%%%%%%%%%%%%%%%%%%
\label{app:extra}

\paragraph{Per-field calibration of decoded MHD errors.}
Table~\ref{tab:calibration-fields} accompanies the ``Transfer to
decoded field-space errors'' paragraph of the main paper. Per-field
polynomial degrees are selected by the same trajectory-level
cross-validation protocol as the latent calibrator, additionally
restricted to monotone candidates. The gains are largest for the
hydrodynamic quantities ($\rho, v_x, v_y, P$: $+0.67$--$+0.73$ nats)
and smaller for the magnetic components ($B_x, B_y$:
$+0.32$--$+0.36$), whose advantage over the depth-only baseline
concentrates at deeper rollouts ($i \ge 40$); pressure is the hardest
field in absolute terms.

\begin{table}[t]
\centering
\caption{Per-field calibration of \emph{decoded} reconstruction errors
on held-out trajectories. $\Delta$NLL: log-likelihood gain over the
identically-fit depth-only calibrator (nats); other columns as in
Table~3 of the main paper.}
\label{tab:calibration-fields}
\footnotesize
\setlength{\tabcolsep}{4pt}
\resizebox{\columnwidth}{!}{%
\begin{tabular}{lccccc}
\toprule
Field & $\Delta$NLL $\uparrow$ & $\times_{68}$ & $\times_{95}$ &
Cov$_{68/95}$ (\%) & MCA $\downarrow$ \\
\midrule
$\rho$ & $+0.73$ & $1.18$ & $1.39$ & $63.6\,/\,92.5$ & $0.022$ \\
$v_x$  & $+0.73$ & $1.17$ & $1.37$ & $65.1\,/\,93.1$ & $0.058$ \\
$v_y$  & $+0.72$ & $1.19$ & $1.38$ & $63.0\,/\,93.7$ & $0.040$ \\
$P$    & $+0.67$ & $1.30$ & $1.64$ & $67.2\,/\,95.1$ & $0.042$ \\
$B_x$  & $+0.32$ & $1.27$ & $1.66$ & $74.3\,/\,96.1$ & $0.024$ \\
$B_y$  & $+0.36$ & $1.26$ & $1.57$ & $63.9\,/\,92.0$ & $0.059$ \\
\bottomrule
\end{tabular}}
\end{table}

\paragraph{MHD calibration: safeguards, depth-conditional behavior,
and a trajectory-mean augmentation.}
Two safeguards apply to every calibrator fit reported in the main
paper: the mean function $\mu$ is verified monotone on the data range,
and inputs are clamped to the training support. While pooled
calibration of the $\mu(\cyc{i})$ calibrator is near-nominal,
depth-conditional calibration drifts at the deepest rollouts: the
per-depth miscalibration area reaches $0.27$ at $i \ge 40$, against
$0.013$ pooled, with the same drift recurring per decoded field, this is the regime where the certificate's $\mu^{-i}$ factor predicts the
signal loosens. Augmenting the calibrator with the trajectory mean of
$\log \cyc{}$ and the depth (five additional coefficients, same
fitting recipe) yields a further $+0.26$ nats (paired per-trajectory,
$\pm 0.02$) and largely repairs the mid-depth calibration drift
(miscalibration area $0.23 \to 0.09$ at $i{=}40$) while leaving
shallow-depth calibration intact; adding temporal-shape features of
the $\cyc{}$-curve (early-window level, slope) brings no further gain
($+0.01 \pm 0.02$), indicating that per-depth $\cyc{i}$ values act as
exchangeable measurements of a trajectory-level quality factor rather
than carrying timing information. Because the trajectory mean reads
the full curve, this variant suits post-hoc assessment of a completed
rollout (a prefix mean provides a causal analogue); residual drift at
the deepest depths ($0.27 \to 0.22$ at $i{=}80$) is where per-depth
recalibration or a split-conformal wrapper
\citep{angelopoulos2023conformal} would restore finite-sample
coverage, which we leave to future work.

\paragraph{LE-PDE-UQ: training additions and model selection.}
Because their baseline trains with a weighted four-step unrolled
objective while ours is single-step, we disclose two training additions:
Gaussian noise on context and anchor frames, and pushforward-style context
augmentation \citep{brandstetter2022message} (1--3 unrolled no-gradient
steps supply model-generated contexts; targets remain the true next
frames), applied as a short low-learning-rate fine-tune. Checkpoints are
selected by validation \emph{rollout} error rather than the validation
denoising loss: the $\epsilon$-prediction objective is a teacher-forced,
single-step surrogate that reflects neither the iterated sampler nor
robustness to self-generated contexts, and in our runs the two
anti-correlate: the $\epsilon$-selected checkpoint is the worst rollout
model in the sweep.

\paragraph{LE-PDE-UQ: calibration protocol.}
All fitted uncertainty objects including the heteroscedastic calibrators, the
depth-only baseline tables, pixel-scale factors, and shape-mixing weights
$\lambda$, are fit on rollouts from \emph{held-out calibration
trajectories} (drawn from the validation split), never on the surrogate's
training trajectories. This departs from the MHD calibration of the main paper, where fitting on training rollouts sufficed, for a measured reason: here the surrogate's held-out rollout error exceeds its training-trajectory
rollout error by a factor of $4.8$ on average (up to $6.3$ at depth $10$), so training-set rollouts under-represent deployment error and calibrators fit on them transfer systematically undersized uncertainty; calibration data must resemble the data being calibrated \citep{kuleshov2018accurate,
angelopoulos2023conformal}. During development we fit on one interleaved
half of the validation trajectories and report on the other; final numbers
are produced once, fitting on the full validation split and evaluating on
the untouched test split. Model selection uses validation rollout error
and is disjoint from calibration fitting.

\paragraph{LE-PDE-UQ: pixel $z$-score diagnostics and latent-space
results.}
The residual gap between the best training-free pixel calibration of
the main paper and learned per-pixel uncertainty is within-frame
heteroscedasticity that no frame-level scalar represents: pixel
$z$-scores under the composed calibrator are over-concentrated at
$\pm1\sigma$ ($P(|z|{<}1)\approx0.80$ vs.\ nominal $0.68$) with heavy
tails (kurtosis ${\approx}10$). The per-frame scale is right on
average while extreme pixels exceed Gaussian tails. In the latent
space where $\cyc{i}$ is computed, the same run completes the
cross-system spectrum of the main paper: $\cyc{i}$ and $\err{i}$ are
rank-correlated within trajectories (Spearman $0.88\pm0.17$, test) and
across trajectories at fixed depth ($0.32$-$0.59$), and the
identically fit latent calibrator improves pooled test NLL by $+0.04$
nats over depth-only, the gain concentrated at shallow depths (NLL
$0.06$ vs.\ $0.16$ at $i{=}1$; $0.34$ vs.\ $0.48$ at $i{=}4$) and gone
by $i{=}10$ as errors approach the attractor scale.

%%%%%%%%%%%%%%%%%%%%%%%%%%%%%%%%%%%%%%%%%%%%%%%%%%%%%%
\section{Inverse Rollouts}
%%%%%%%%%%%%%%%%%%%%%%%%%%%%%%%%%%%%%%%%%%%%%%%%%%%%%%
\label{app:inverse}

This section collects the evidence behind the ``Inverse Rollouts''
subsection of the main paper. All results use the single trained
bidirectional model with the direction flag set to $\cd{=}-1$ and the
deterministic sampler; no additional training or fine-tuning is
involved, and one backward step costs the same as one forward step.
Backward rollouts are seeded with the \emph{true} terminal pair
$(\zb_{T-1},\zb_T)$ of a held-out test trajectory and iterated toward
$t{=}0$; these are exactly the rollouts whose residual $\delta_i$
defines the noise floor of the sandwich bound
(Sec.~\ref{app:theory-details}), so the reconstructions of
Fig.~\ref{fig:inv-fields} visualize what $\delta_i$ measures. For the
direction-symmetric consistency check of Fig.~\ref{fig:inv-ci}, a
backward rollout of depth $i$ is cycled \emph{forward} for $i$ steps
and compared with its own terminal seed pair, yielding the
mirror-image round-trip error $\cyc{i}^{\,-}$: the defining formula
of the main paper with the roles of $\fwd$ and $\bwd$ exchanged.
Fig.~\ref{fig:inv-latent-forward} shows the same forward rollout in the latent space where all metrics are computed and Fig.~\ref{fig:inv-latent-back} shows the backward rollout.

\begin{figure*}[t]
\centering
\includegraphics[width=\textwidth]{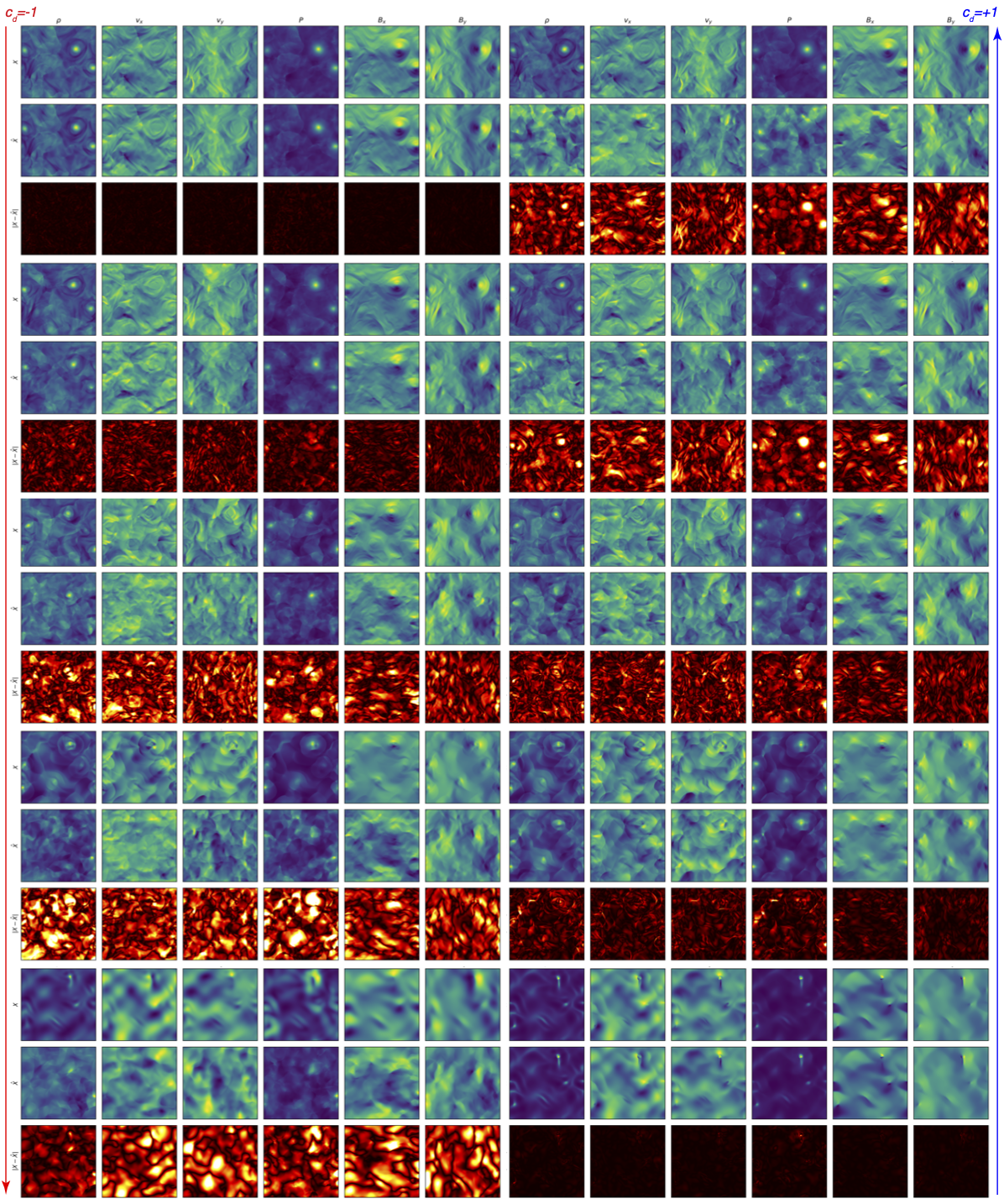}
\caption{\textbf{Physical-field inverse rollouts.} Backward rollout
($\cd{=}-1$) of a held-out MHD test trajectory seeded with only the
true terminal pair $(\zb_{T-1},\zb_T)$: VAE-decoded backward
reconstructions vs.\ ground truth at increasing backward depth $i$
(top to bottom, left 6 columns), with pixelwise absolute error beneath each pair. Reconstructions stay sharp and physically coherent at
shallow-to-moderate backward depths and degrade gracefully as
backward error accumulates, mirroring the forward direction. The forward rollout (bottom to top) is shown in the right 6 columns for all 6 MHD fields.}
\label{fig:inv-fields}
\end{figure*}

\begin{figure*}[t]
\centering
\includegraphics[width=0.85\textwidth]{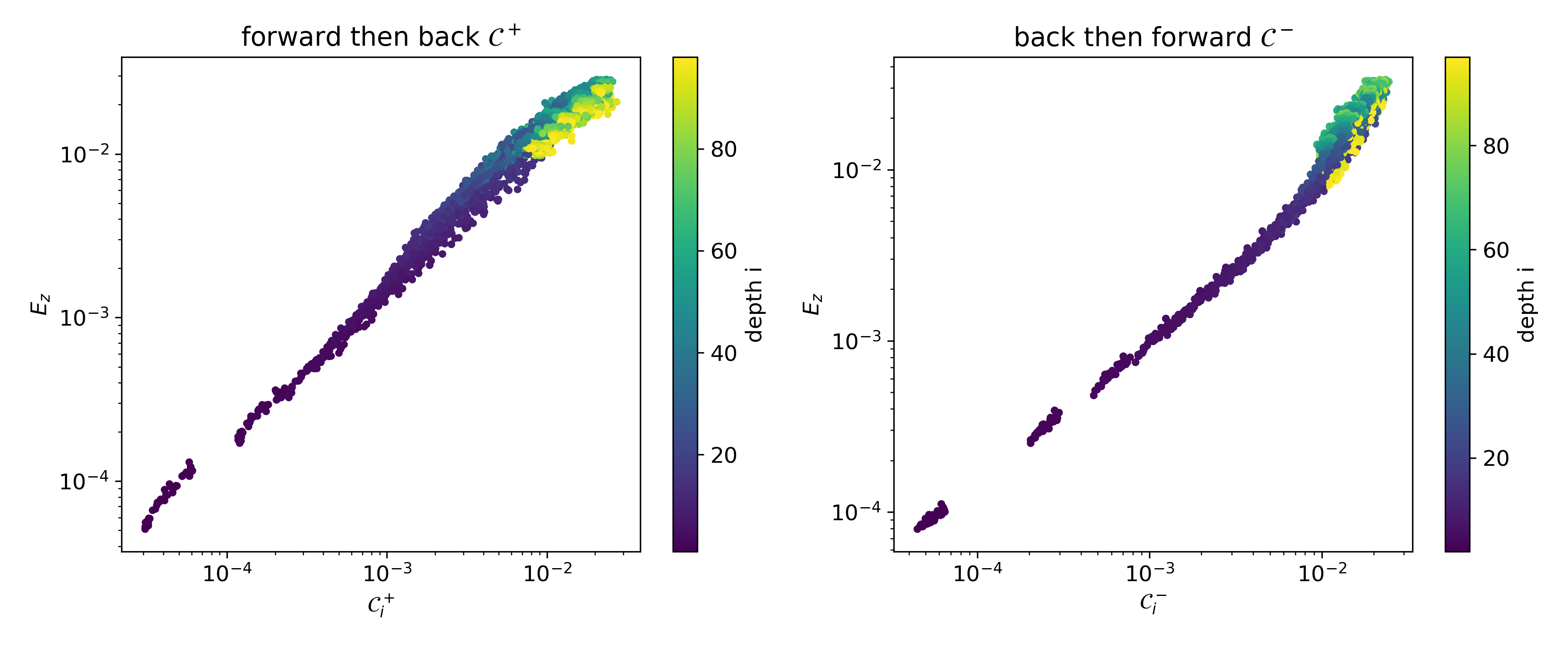}
\caption{\textbf{Predicting unknown latent rollout error: Round-trip consistency forward or backward in time.}
Forward rollouts cycled backward (the default $\cyc{i}$, left) and
backward rollouts cycled forward ($\cyc{i}^{\,-}$,
right), each plotted against the corresponding true rollout error on
the 50 held-out MHD validation trajectories, colored by rollout depth. The backward cycle tracks backward error when the model is deployed as an inverse solver.}
\label{fig:inv-ci}
\end{figure*}

\begin{figure*}[t]
\centering
\includegraphics[width=\textwidth]{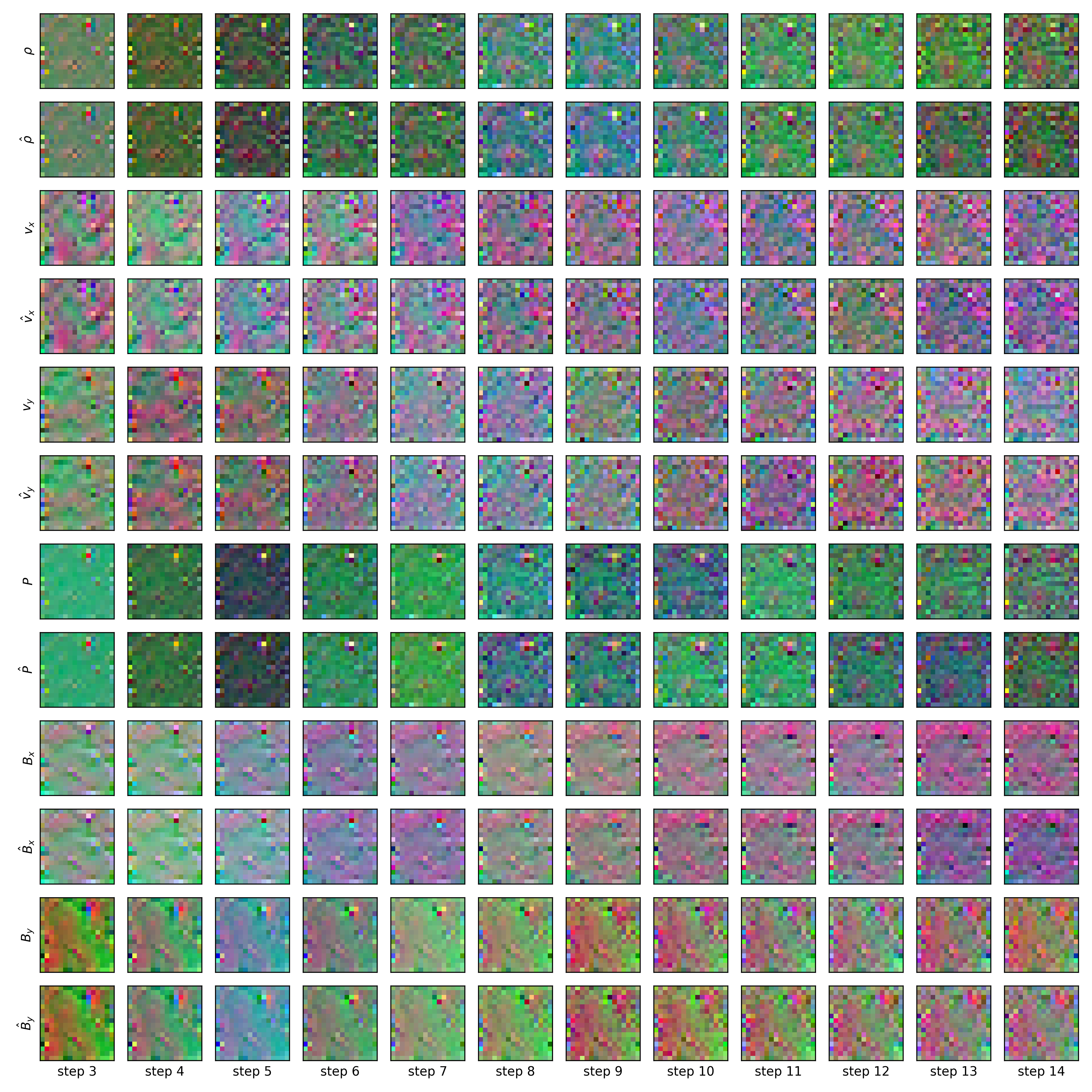}
\caption{\textbf{Latent rollouts, forward.} Latent
trajectories of a held-out MHD test trajectory in the
$16{\times}16{\times}4$ per-field latent space where $\cyc{i}$,
$\err{i}$, and $\delta_i$ are computed: encoded ground truth (top
of each row), forward rollout from the initial pair (from left to right), shown at
matched times. Drift away from the encoded truth grows with depth --- the latent-level picture behind the decoded fields of Fig.~\ref{fig:inv-fields}.}
\label{fig:inv-latent-forward}
\end{figure*}

\begin{figure*}[t]
\centering
\includegraphics[width=\textwidth]{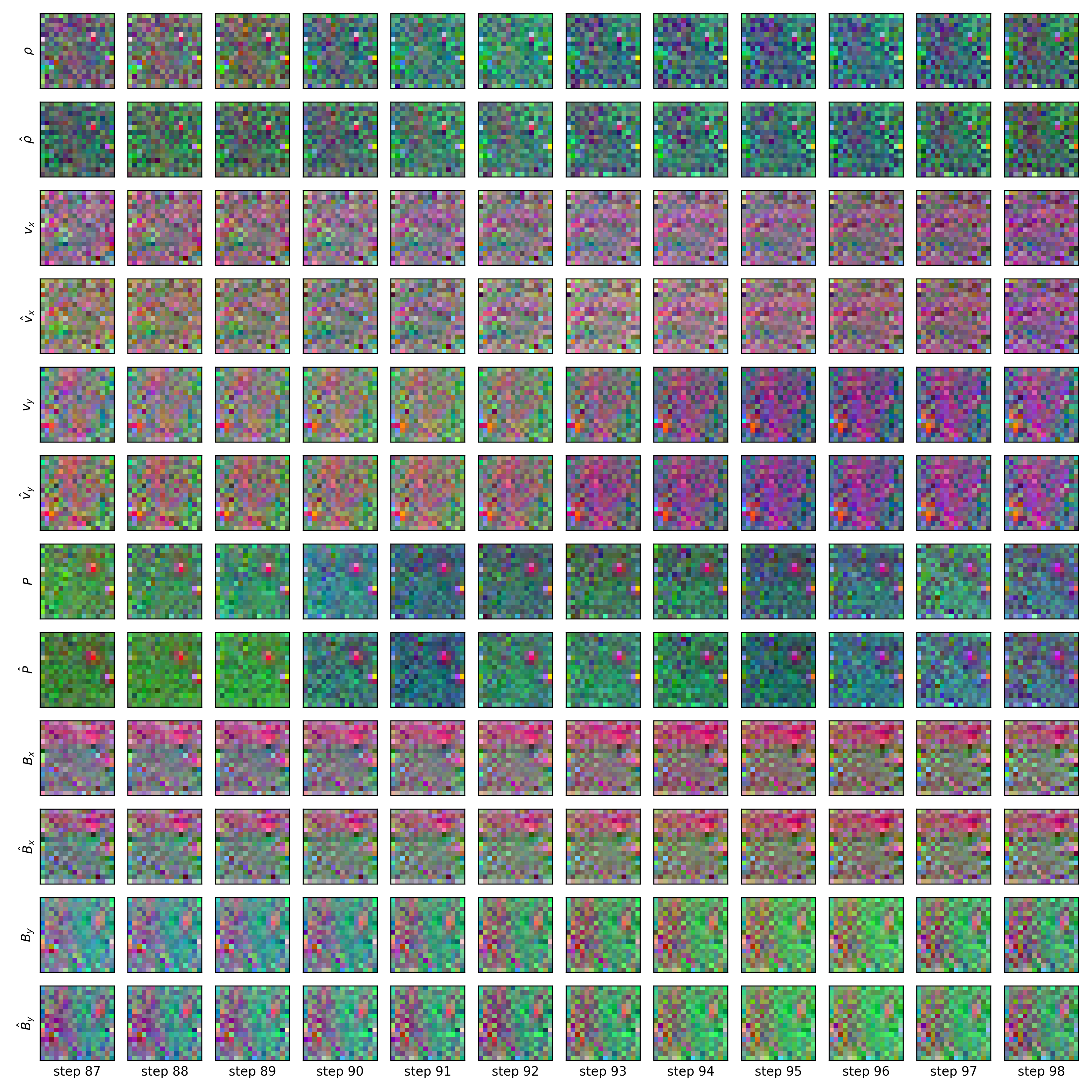}
\caption{\textbf{Latent rollouts, backward.} Latent
trajectories of a held-out MHD test trajectory in the
$16{\times}16{\times}4$ per-field latent space where $\cyc{i}$,
$\err{i}$, and $\delta_i$ are computed: encoded ground truth (top
of each row), backward rollout from the terminal pair (from right to left), shown at
matched times. Drift away from the encoded truth grows with depth --- the latent-level picture behind the decoded fields of Fig.~\ref{fig:inv-fields}.}
\label{fig:inv-latent-back}
\end{figure*}

\begin{figure}[t]
\centering
\includegraphics[width=0.85\columnwidth]{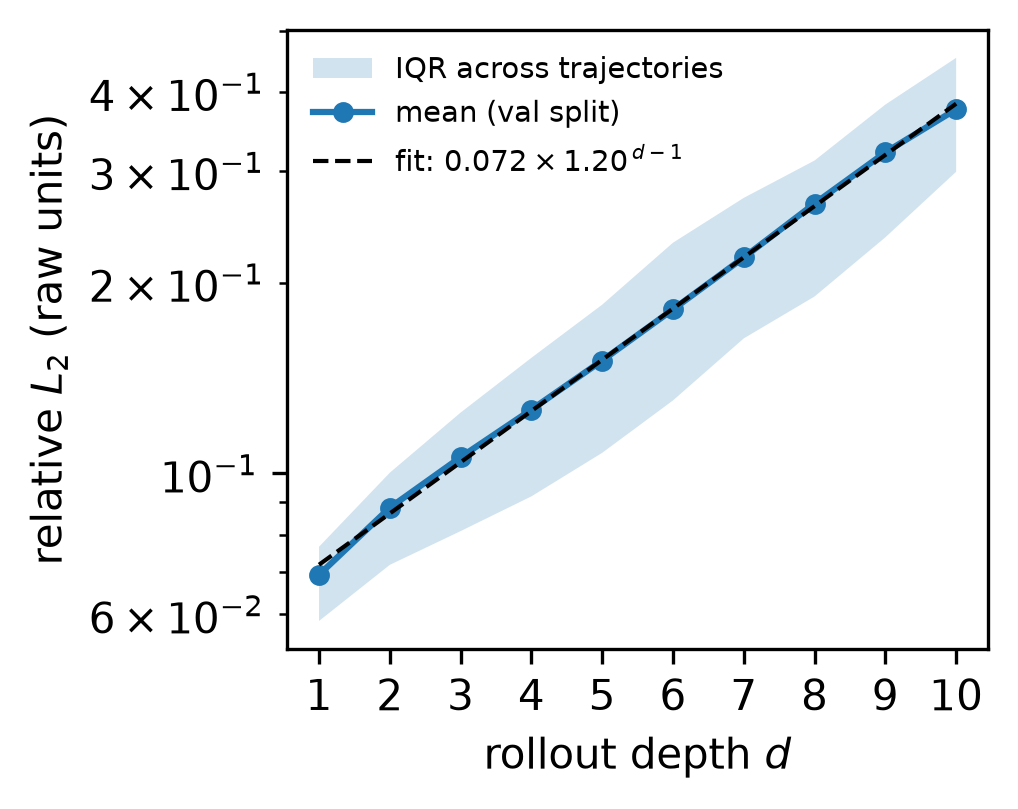}
\caption{Forward-rollout error on the Navier--Stokes benchmark: mean
relative $L_2$ vs.\ rollout depth with the interquartile range across
trajectories. Growth is geometric, $e_d \approx e_1 r^{\,d-1}$ with
$r \approx 1.21$ per step (dashed).}
\label{fig:fno_depth_growth}
\end{figure}

\paragraph{LE-PDE-UQ: additional details.}
The LE-PDE-UQ system evolves according to a set of 2D Navier-Stokes equations governing a viscous, incompressible fluid in vorticity form in a unit torus:
\begin{eqnarray}
    \frac{\partial w(t,x)}{\partial t} &=& -u(t,x)\cdot \nabla w(t,x) + \nu \Delta w(t,x) + f(x), \nonumber \\
    \nabla \cdot u(t,x) &=& 0, \nonumber \\
    w(0,x) &=& w_0(x), \ x\in (0,1)^2, \ t\in[0,T]
\end{eqnarray}
where $w(t,x)=\nabla \times u(t,x)$ is the vorticity, $\nu \in \mathbb{R}_+$ is the viscosity coefficient, the domain is discretized into a $64 \times 64$ grid and $Re=10^4$ (turbulent).

\begin{figure}[t]
\centering
\includegraphics[width=1.0\columnwidth]{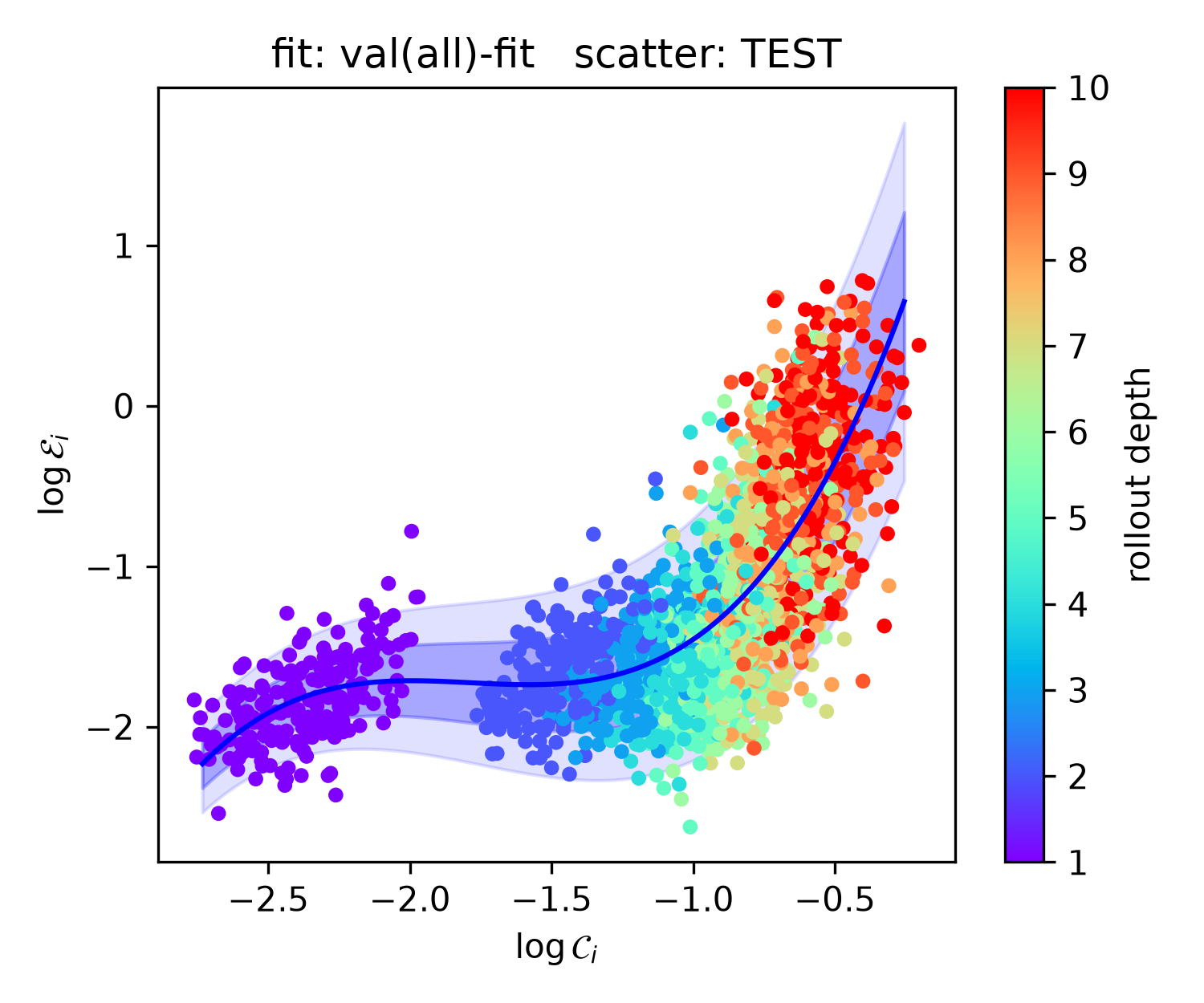}
\caption{The color bands show predicted $\mu\pm2\sigma$ of $\mathcal{E}_i$ as a function of $\cyc{i}$, which was fit on validation data, relative to the actual measurements of $(\cyc{i},\mathcal{E}_i)$ of the test data for the LE-PDE-UQ test dataset colored by rollout distance.}
\label{fig:logClogE_test}
\end{figure}

\paragraph{Additional CELEBV-HQ Rollouts} Additional forward autoregressive rollouts of CELEBV-HQ data, using just the first 2 images in a video, for held out test data, are shown in Figures \ref{fig:CELEBVHQ_Rolls}, \ref{fig:CELEBVHQ_Rolls_More}.

\paragraph{Additional Well Turbulent Flow Rollouts} Figures \ref{fig:back_forward_turb_flow}-\ref{fig:fwd_bck_trb_678} show forward and backward rollouts of all 9 test trajectories of The Well turbulent flow data. An asymmetry is visible here in that after approximately 50 steps, it is easier for the model to go backwards in time rather than forwards because the initial conditions are very similar for all examples, while the final conditions vary widely.

\begin{figure*}[t]
\centering
\includegraphics[width=0.495\textwidth]{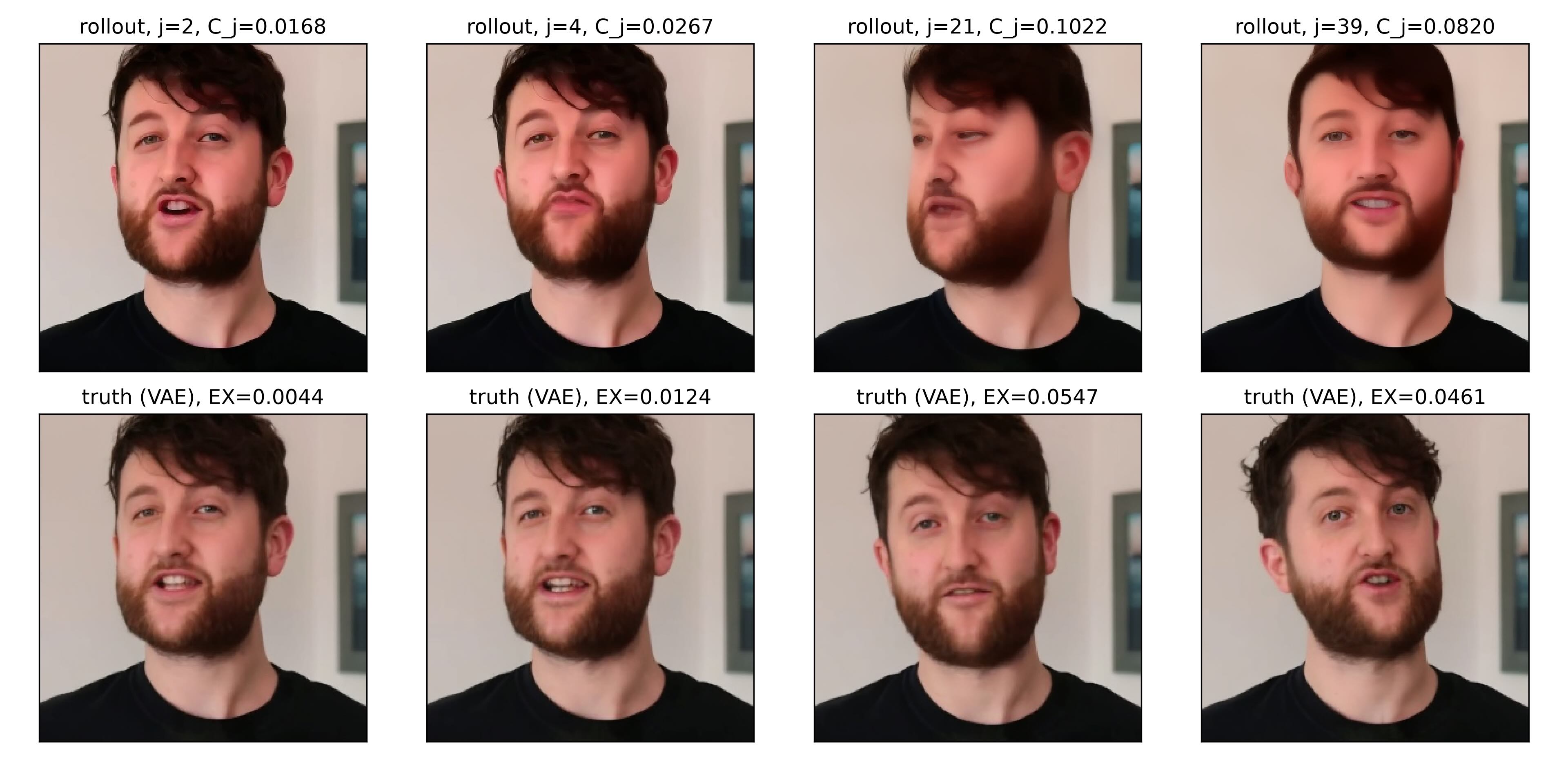}
\includegraphics[width=0.495\textwidth]{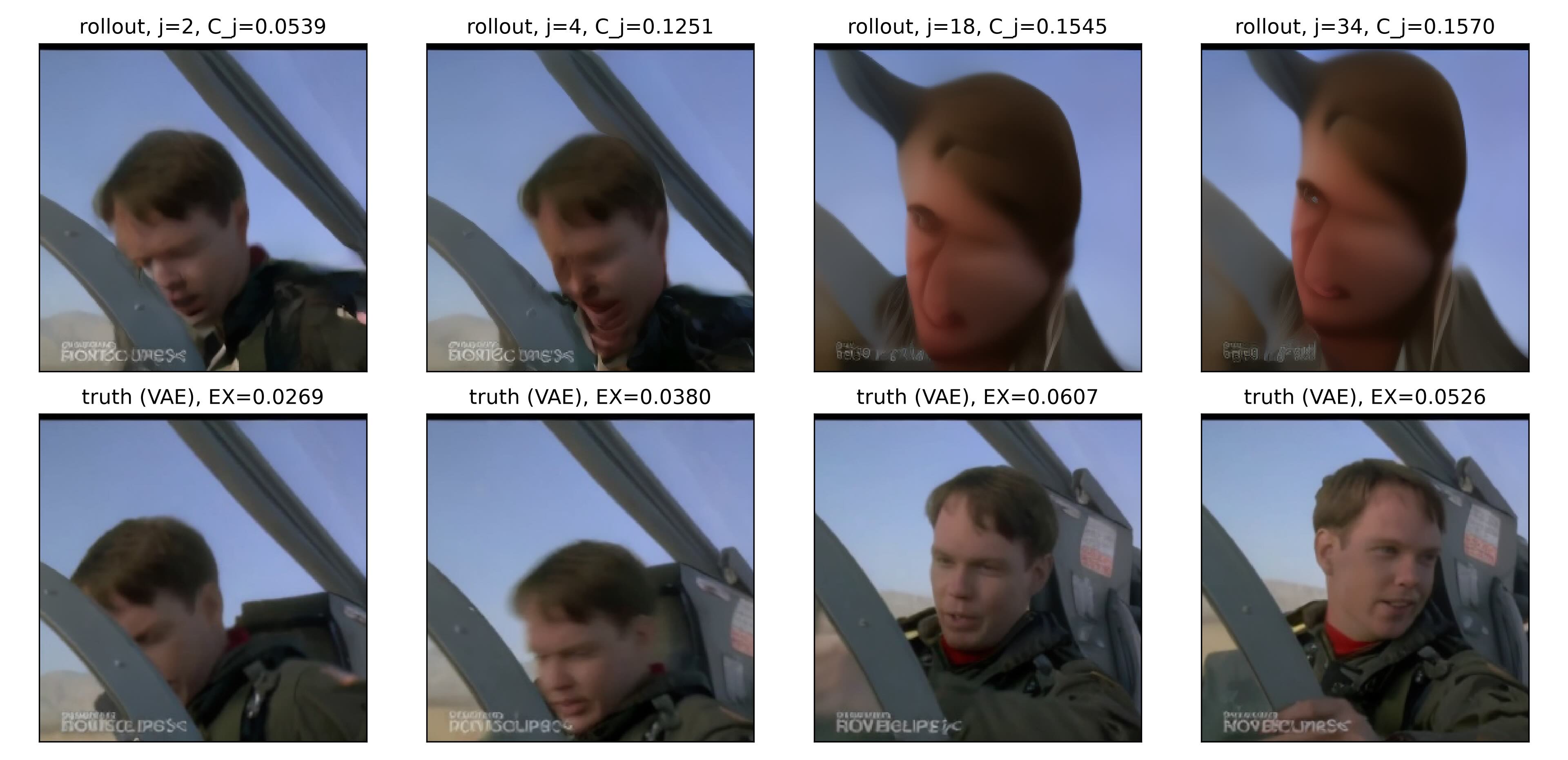}
\includegraphics[width=0.495\textwidth]{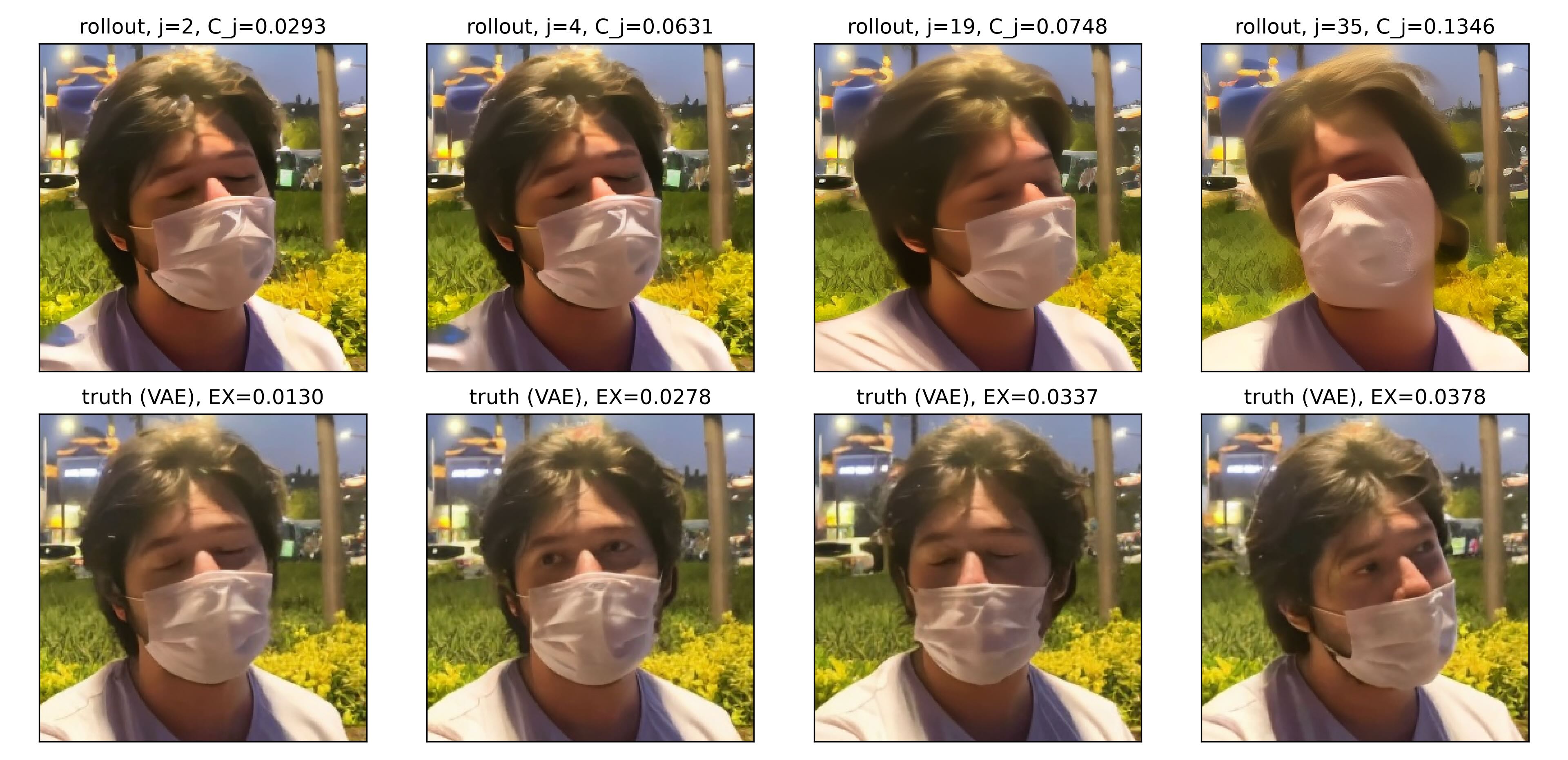}
\includegraphics[width=0.495\textwidth]{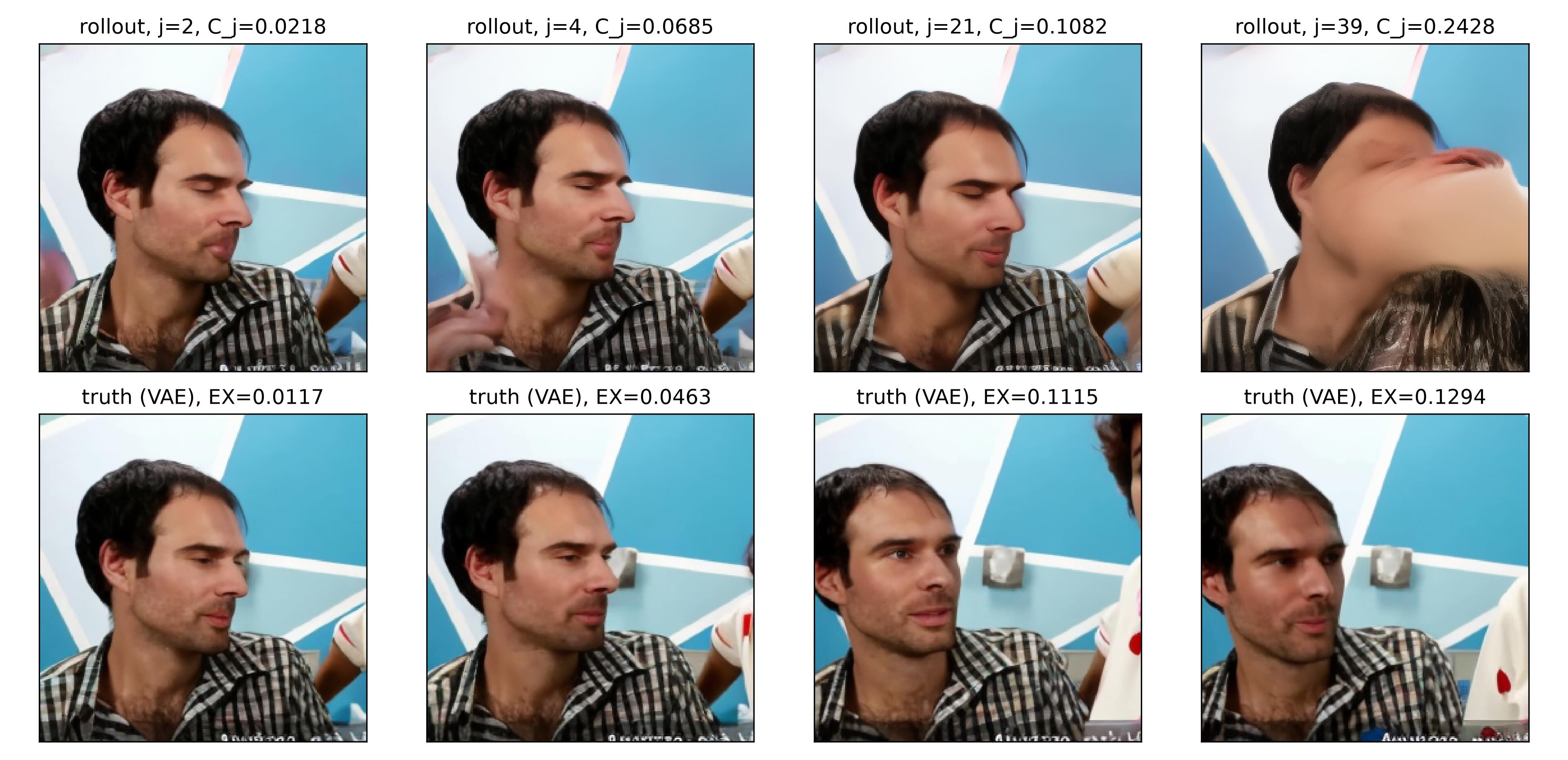}
\includegraphics[width=0.495\textwidth]{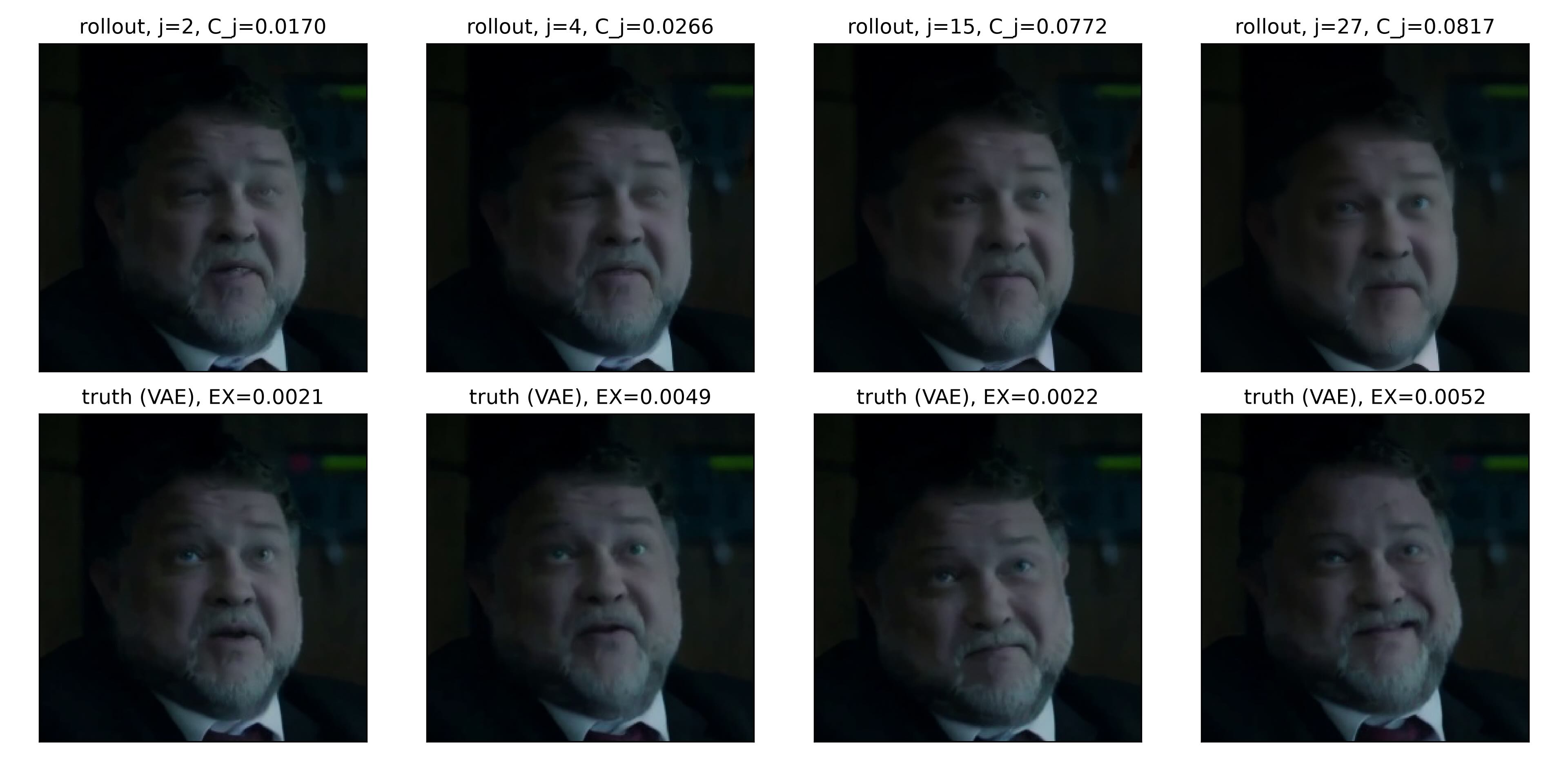}
\includegraphics[width=0.495\textwidth]{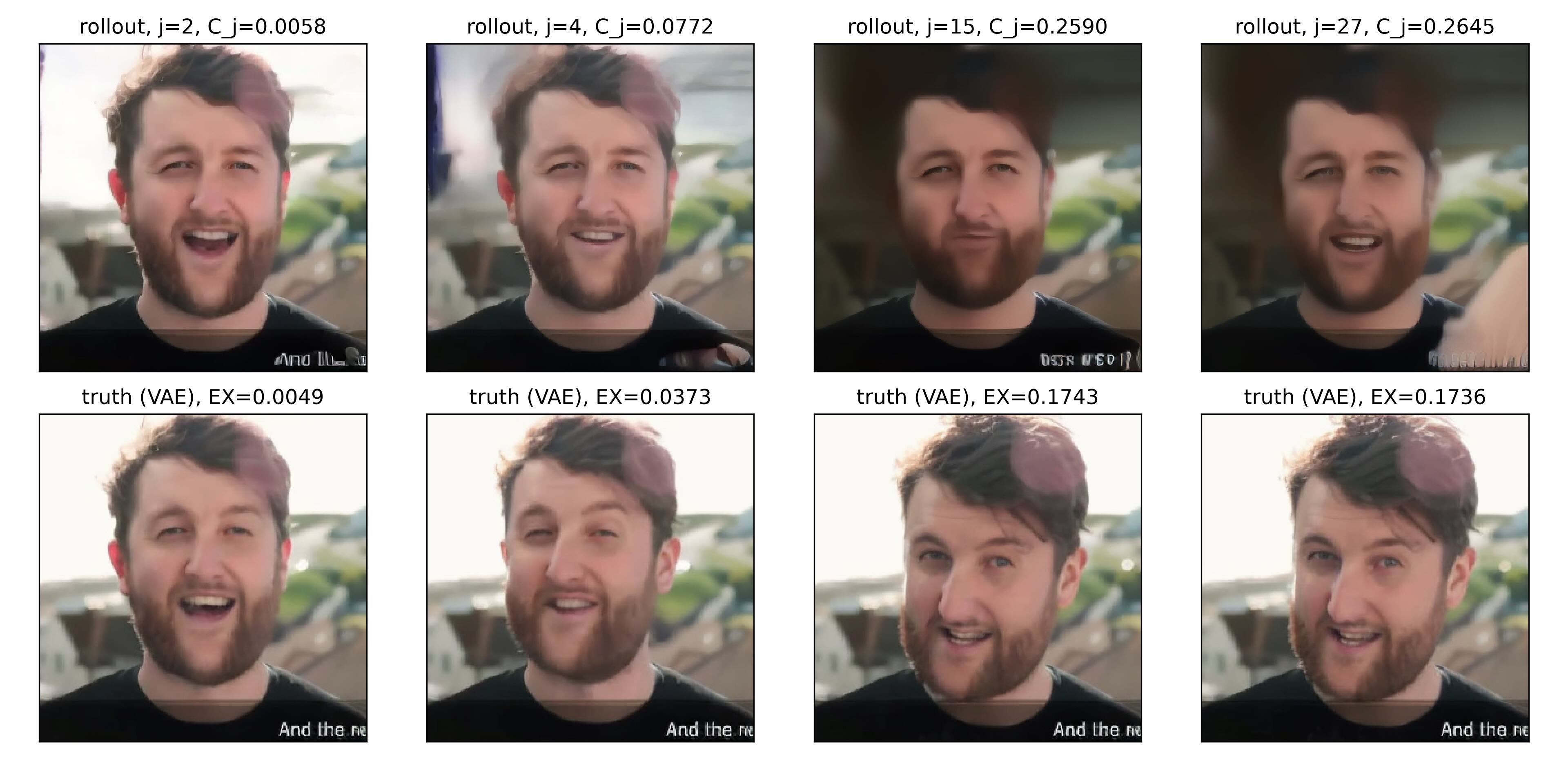}
\includegraphics[width=0.495\textwidth]{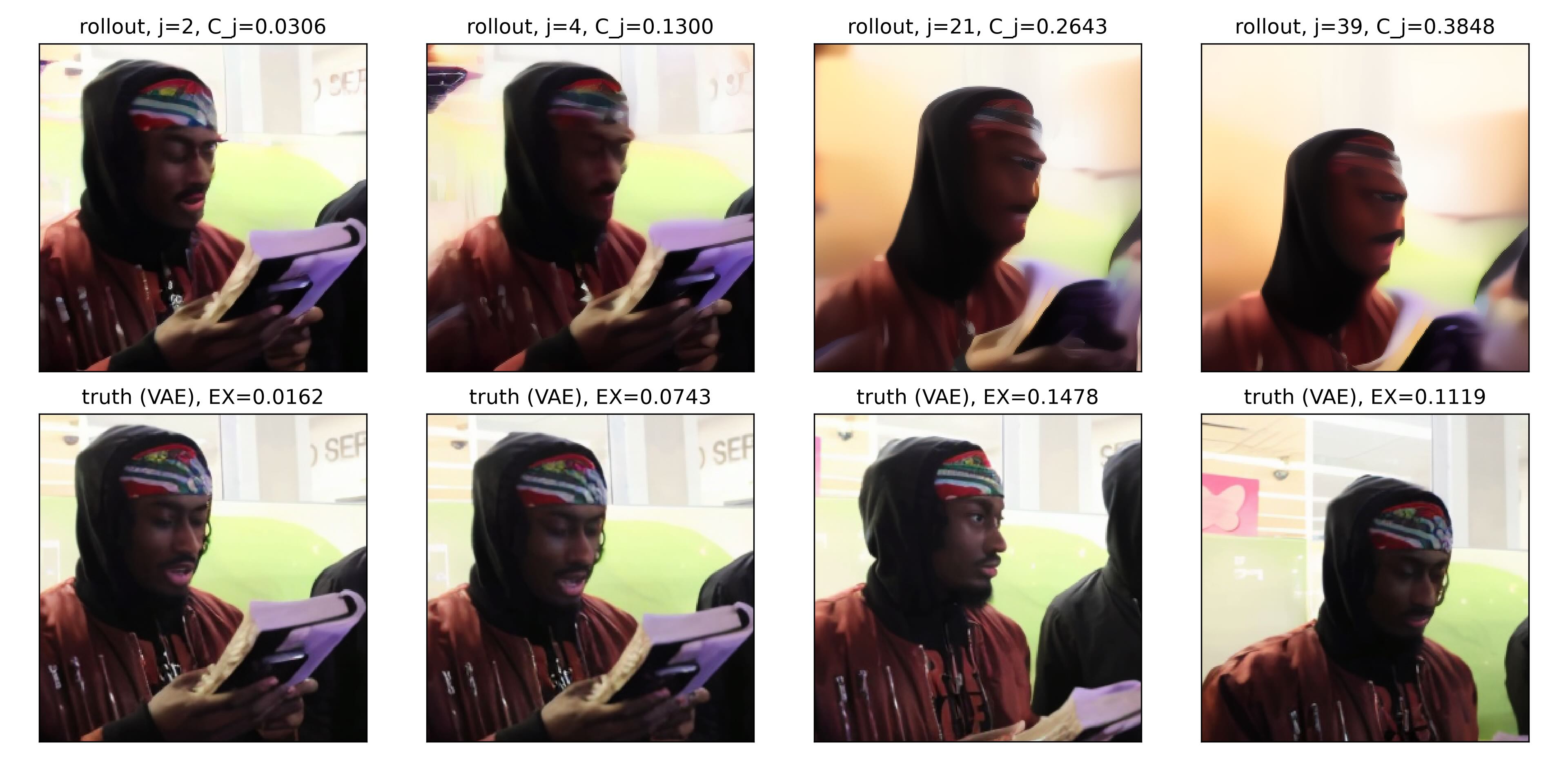}
\includegraphics[width=0.495\textwidth]{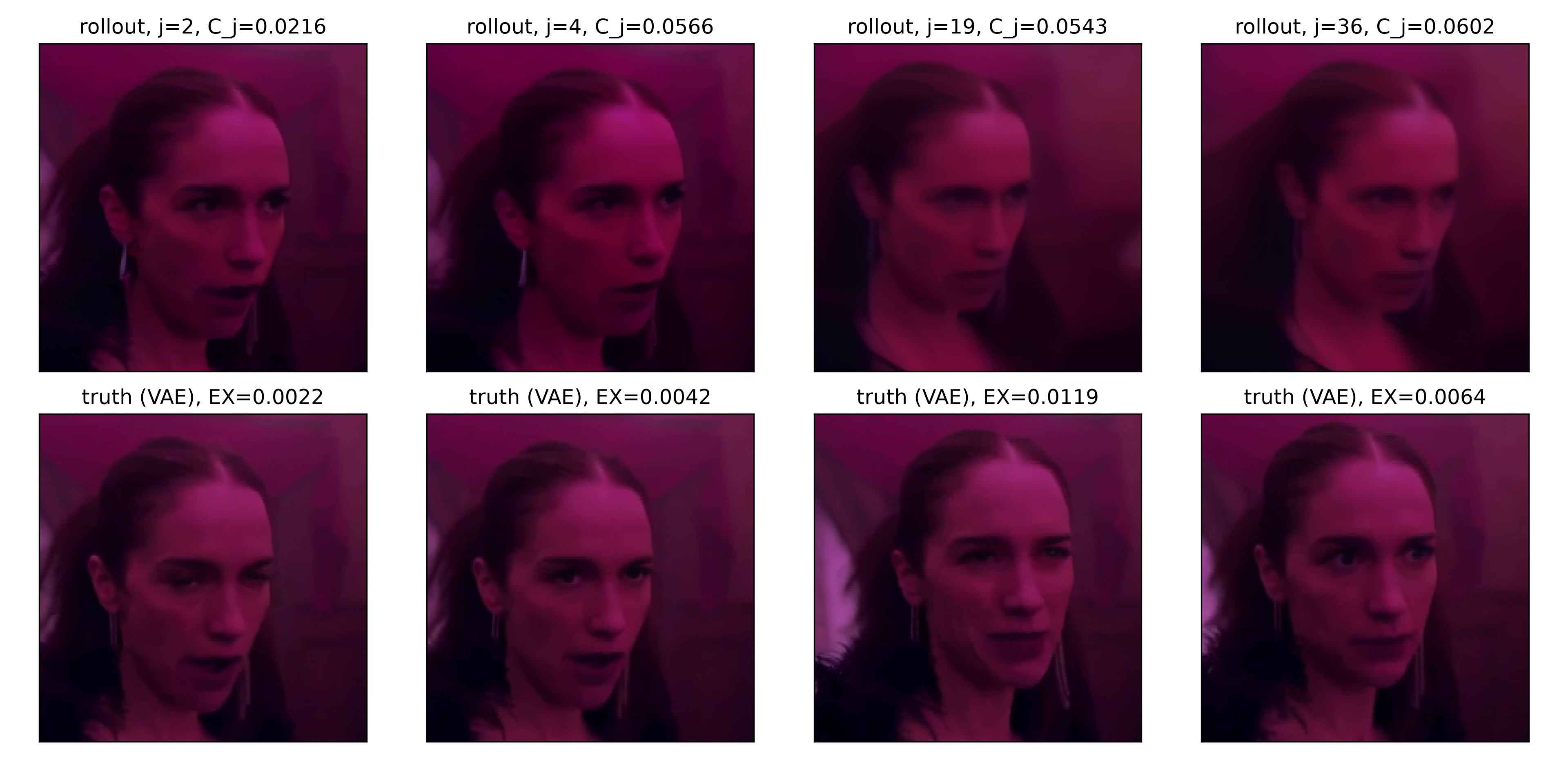}
\includegraphics[width=0.495\textwidth]{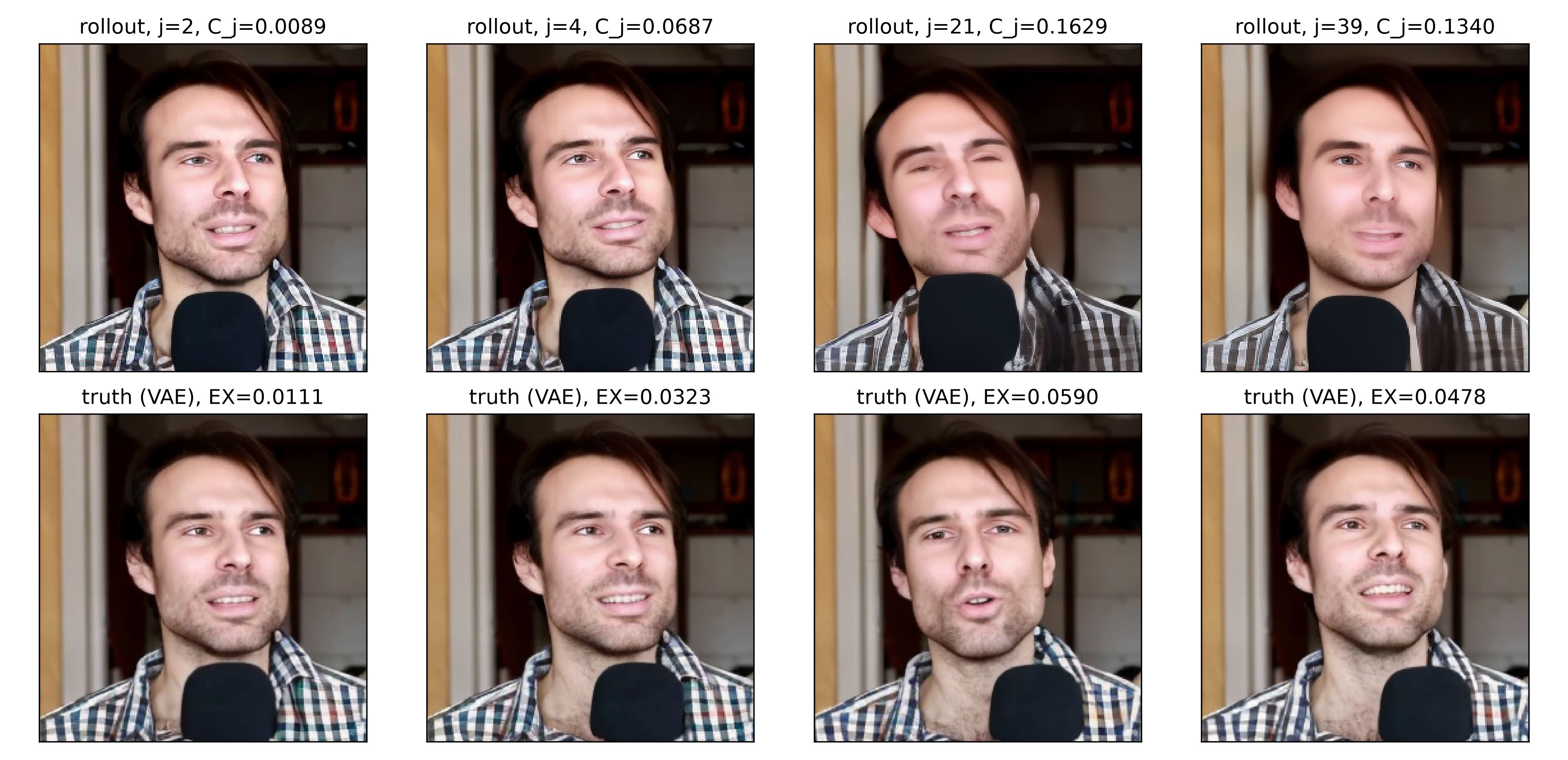}
\includegraphics[width=0.495\textwidth]{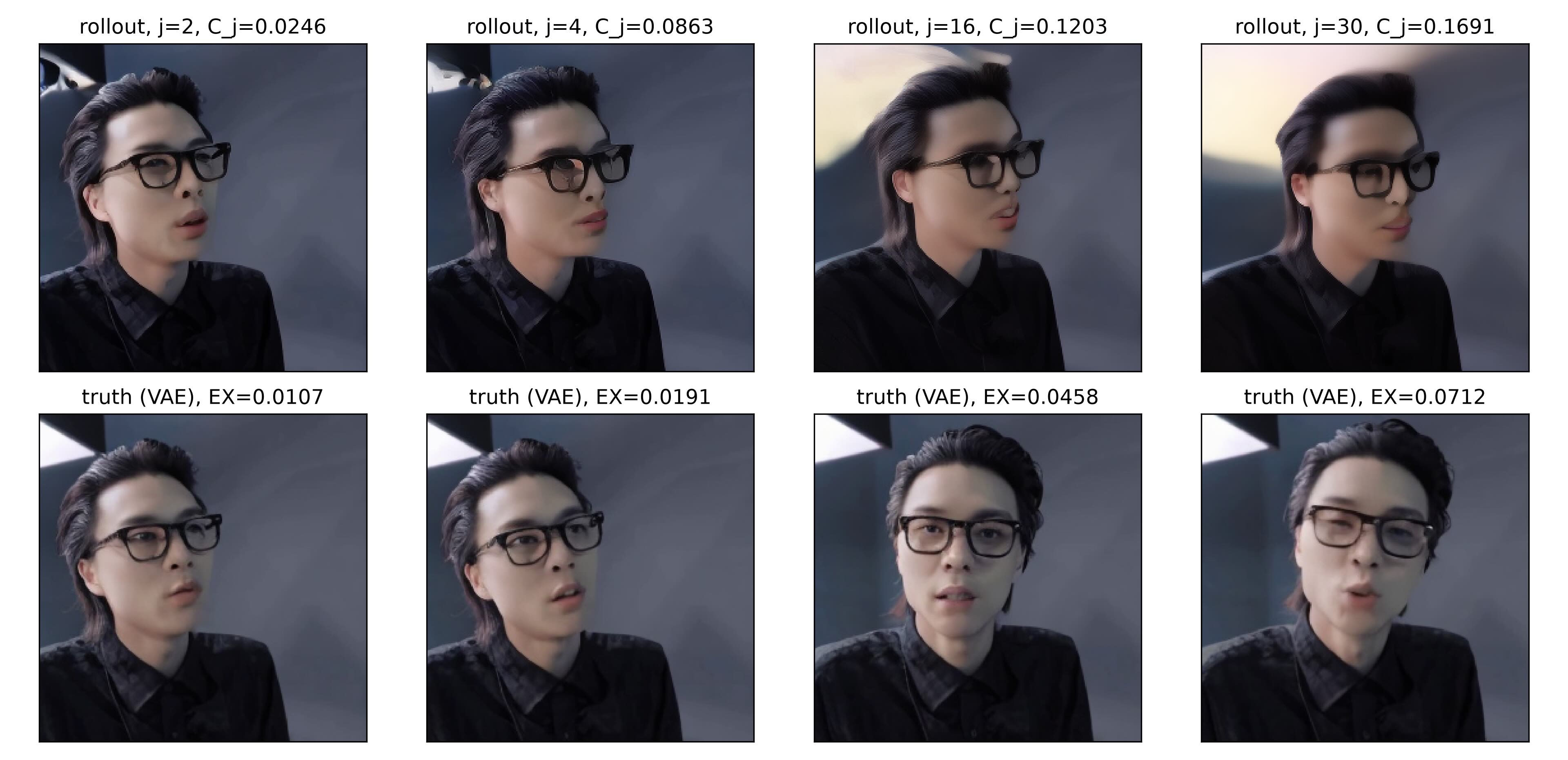}
\caption{CELEBV-HQ rollouts show with error and $\cyc{}$ values. While for such a multi-modal dataset it is impossible for the model to predict rollouts for test data, it is able to flag, in a self-supervised way, OOD data with increased $\cyc{}$ levels.}
\label{fig:CELEBVHQ_Rolls}
\end{figure*}

\begin{figure*}[t]
\centering
\includegraphics[width=0.495\textwidth]{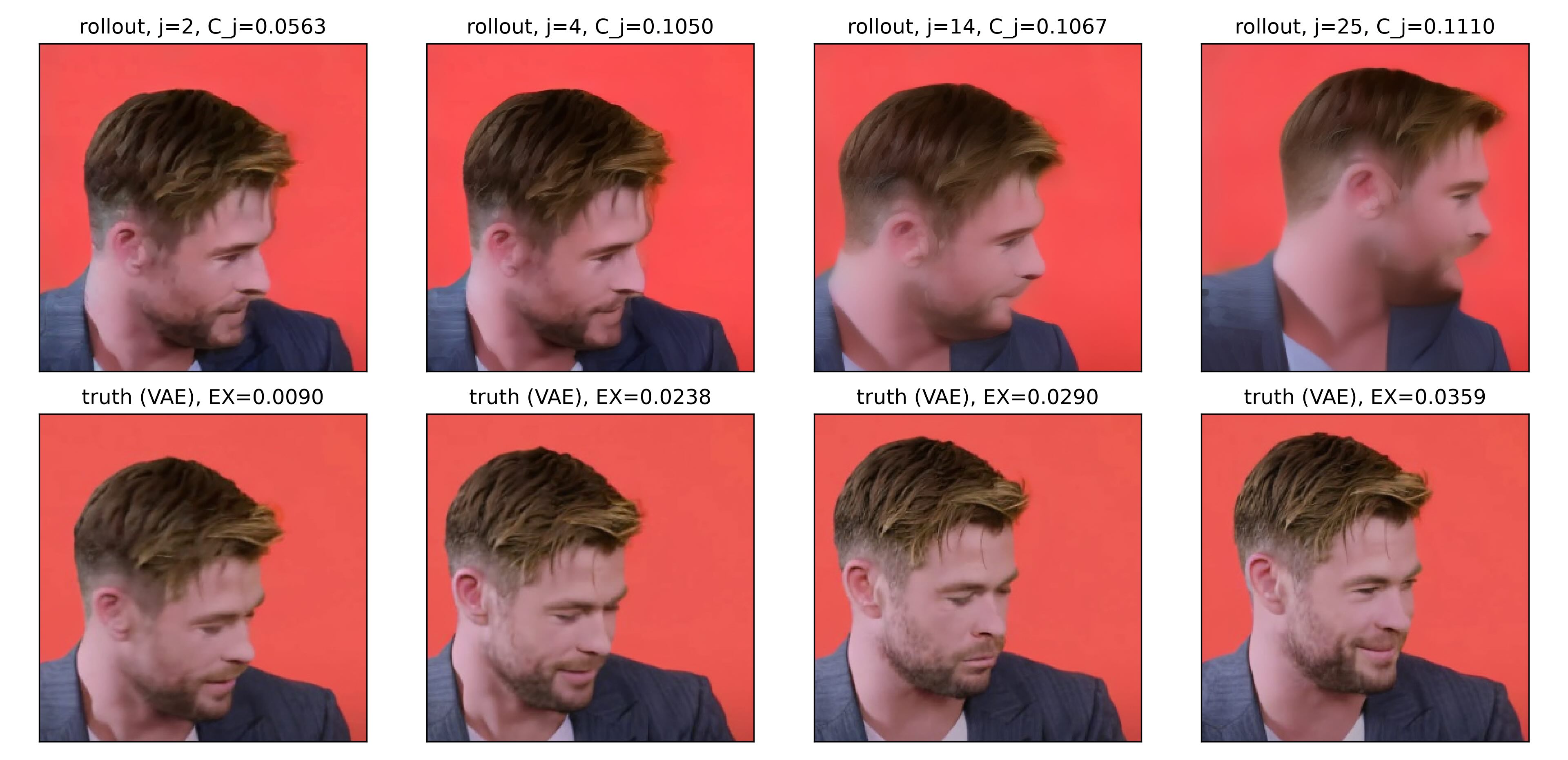}
\includegraphics[width=0.495\textwidth]{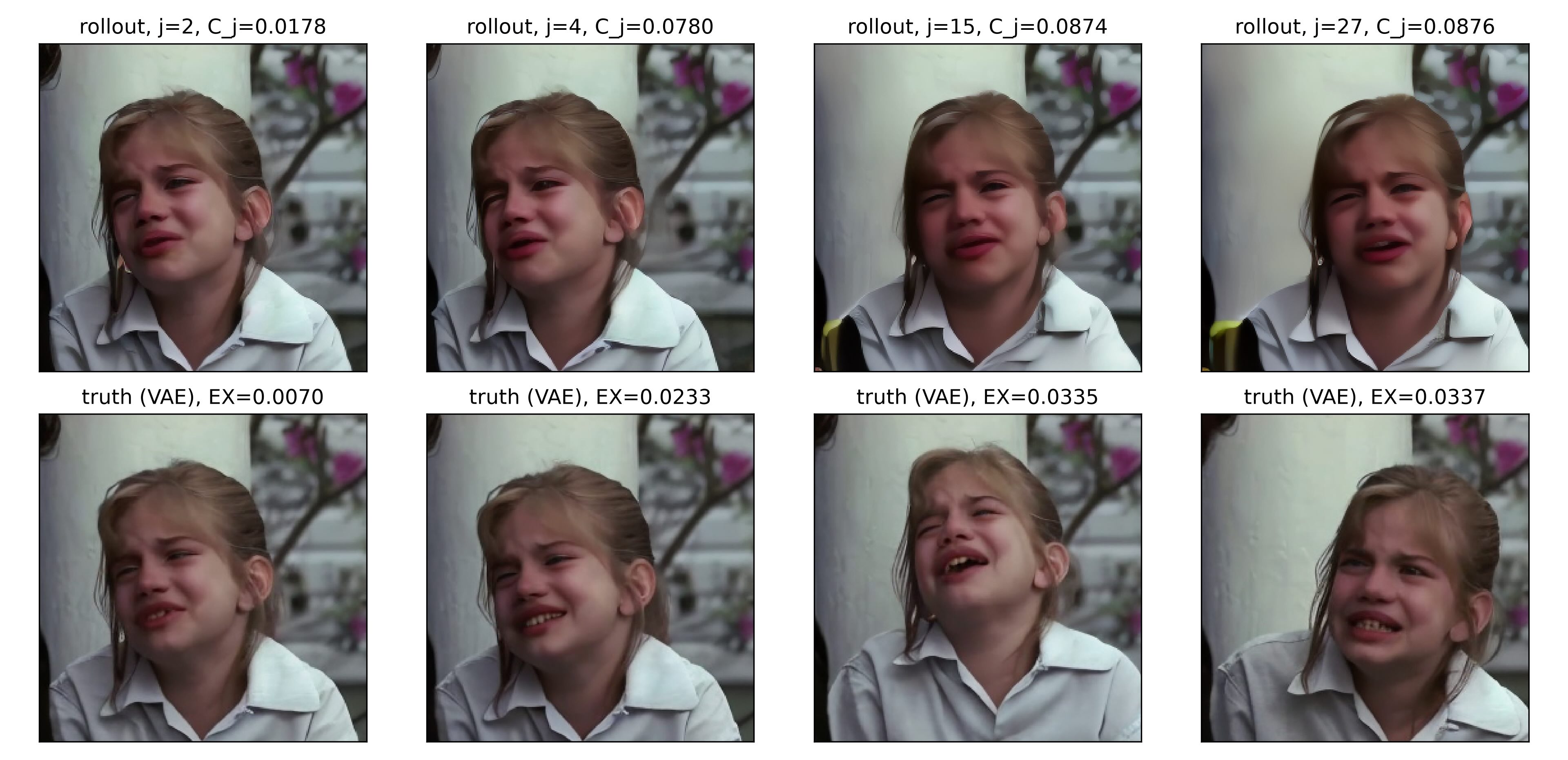}
\includegraphics[width=0.495\textwidth]{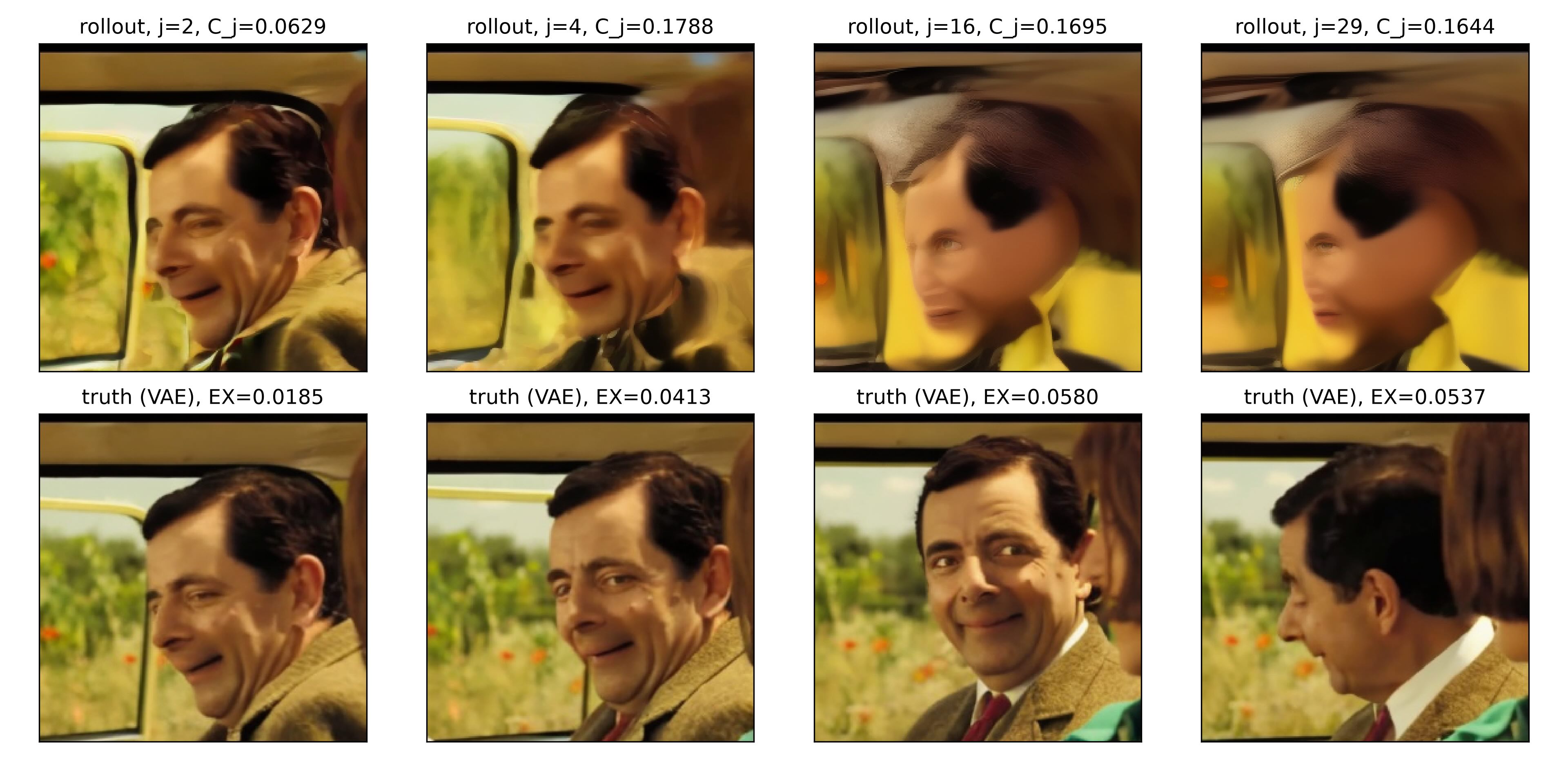}
\includegraphics[width=0.495\textwidth]{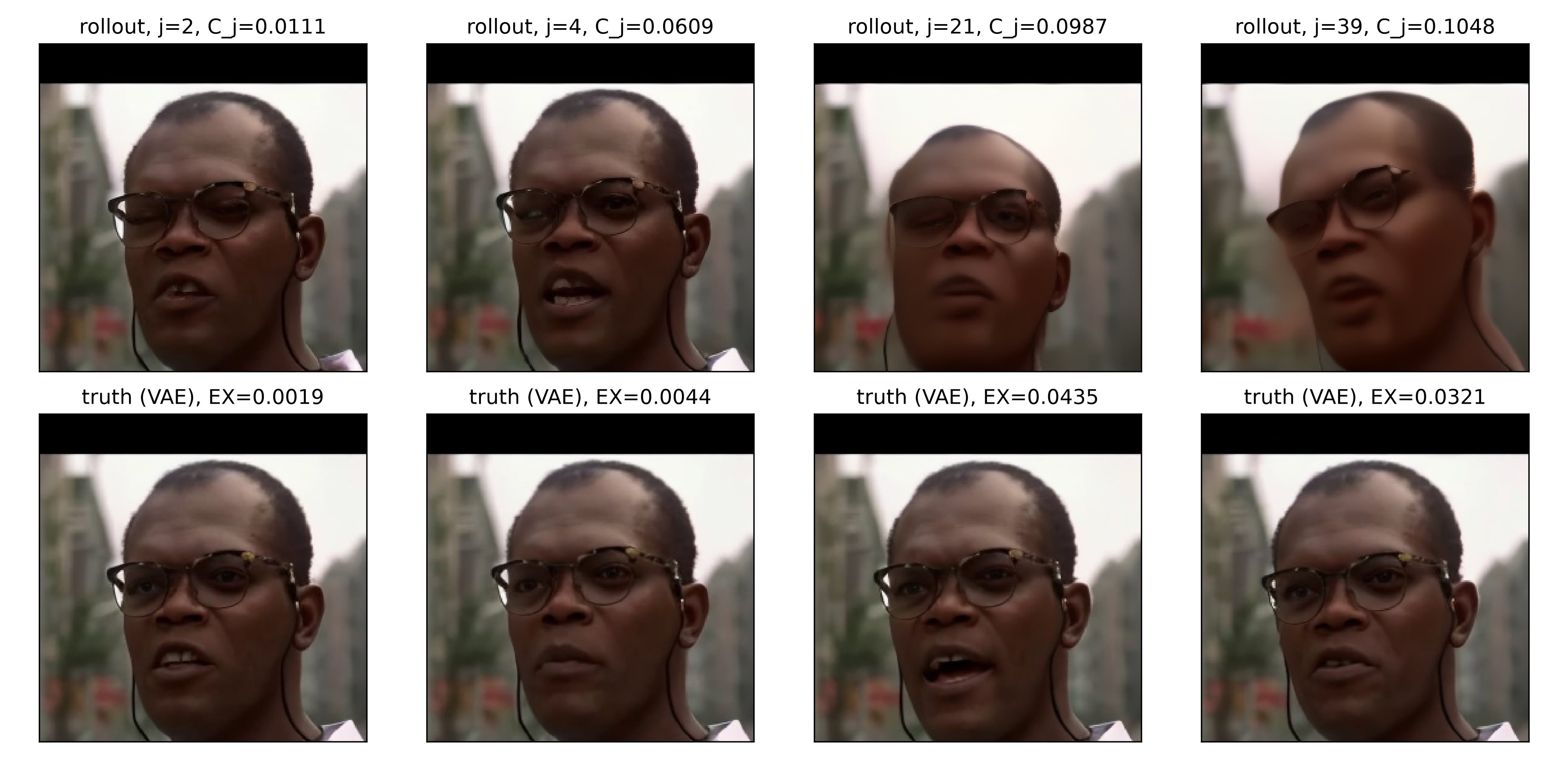}
\includegraphics[width=0.495\textwidth]{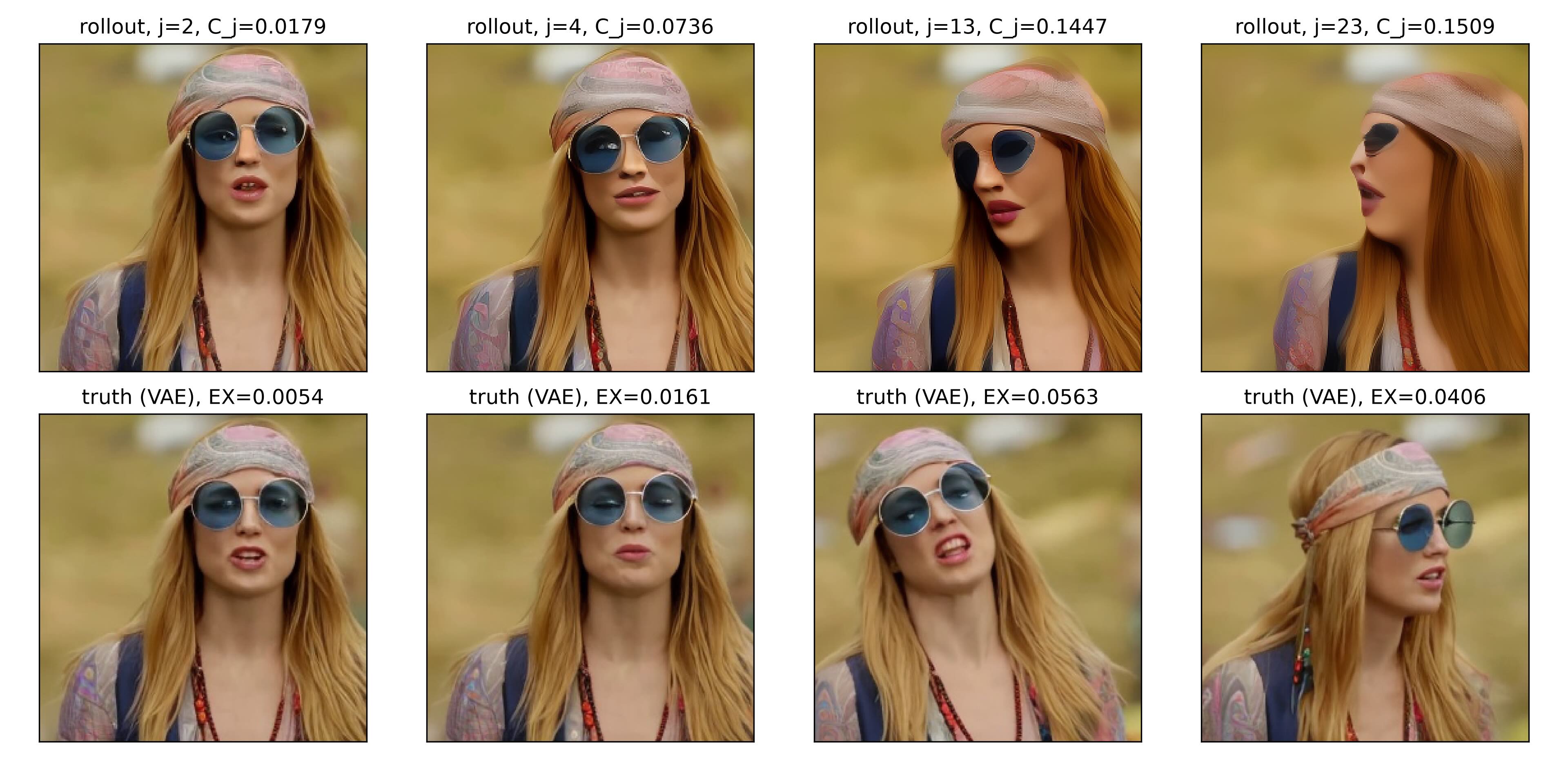}
\includegraphics[width=0.495\textwidth]{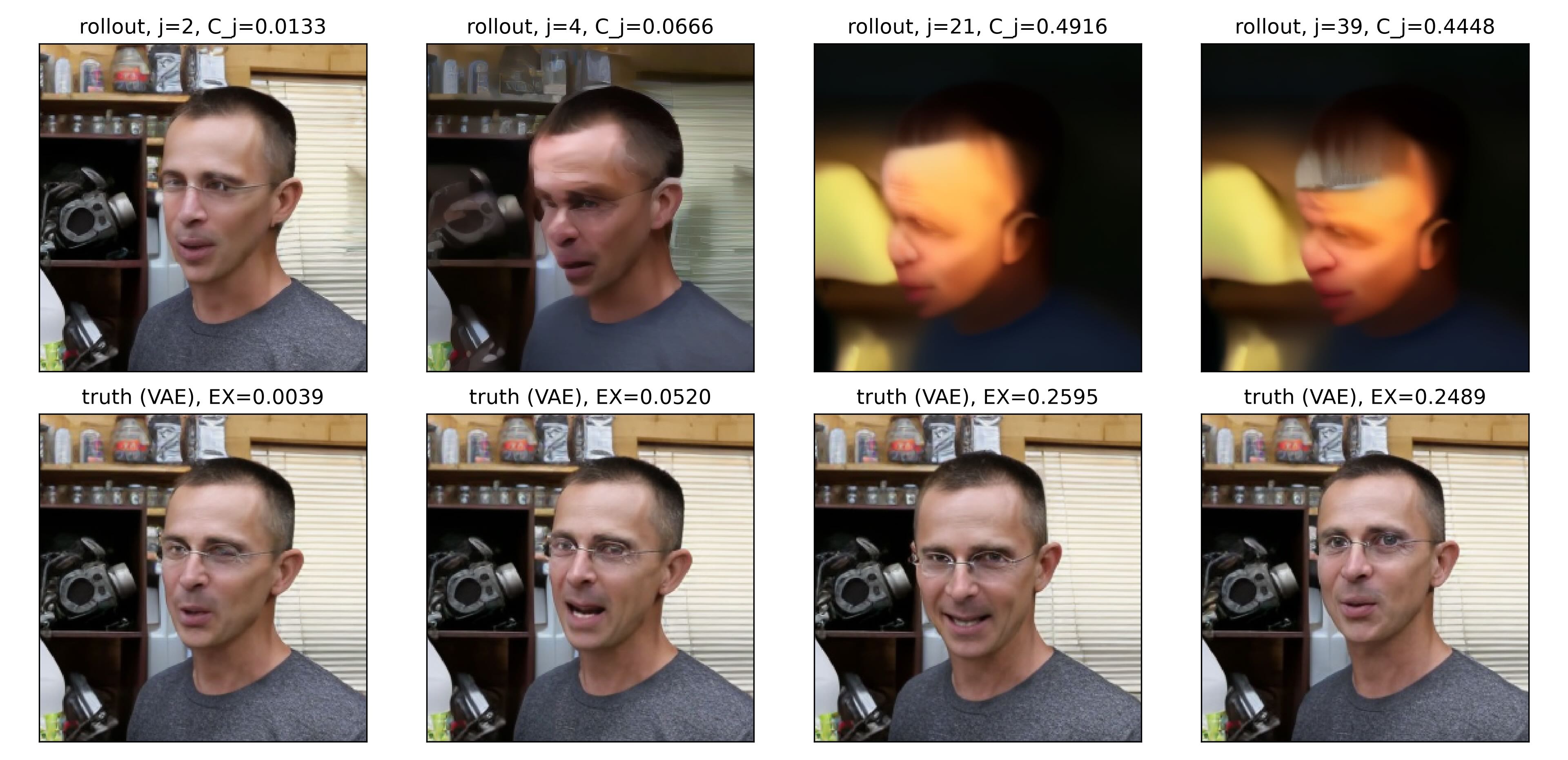}
\includegraphics[width=0.495\textwidth]{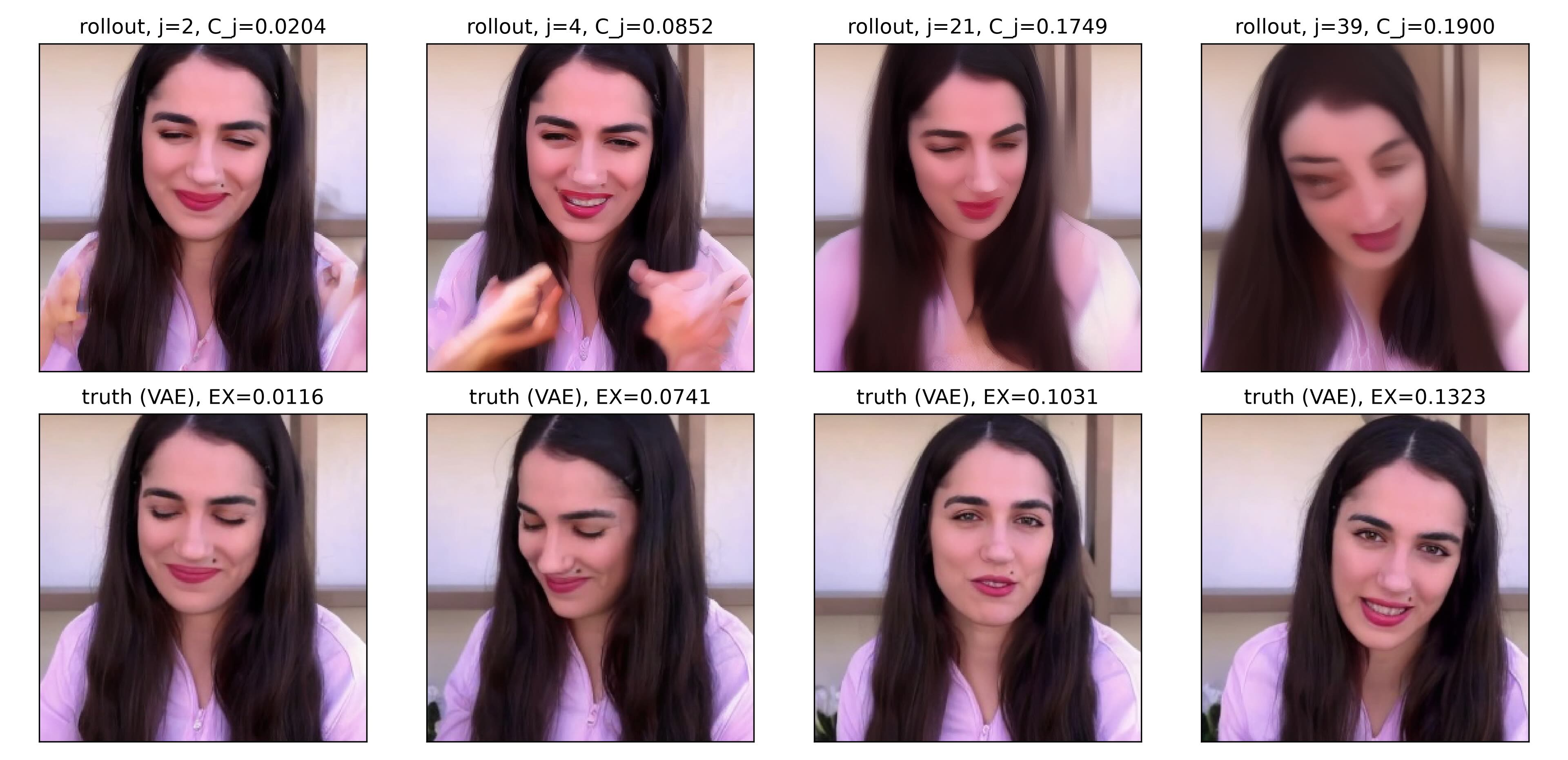}
\includegraphics[width=0.495\textwidth]{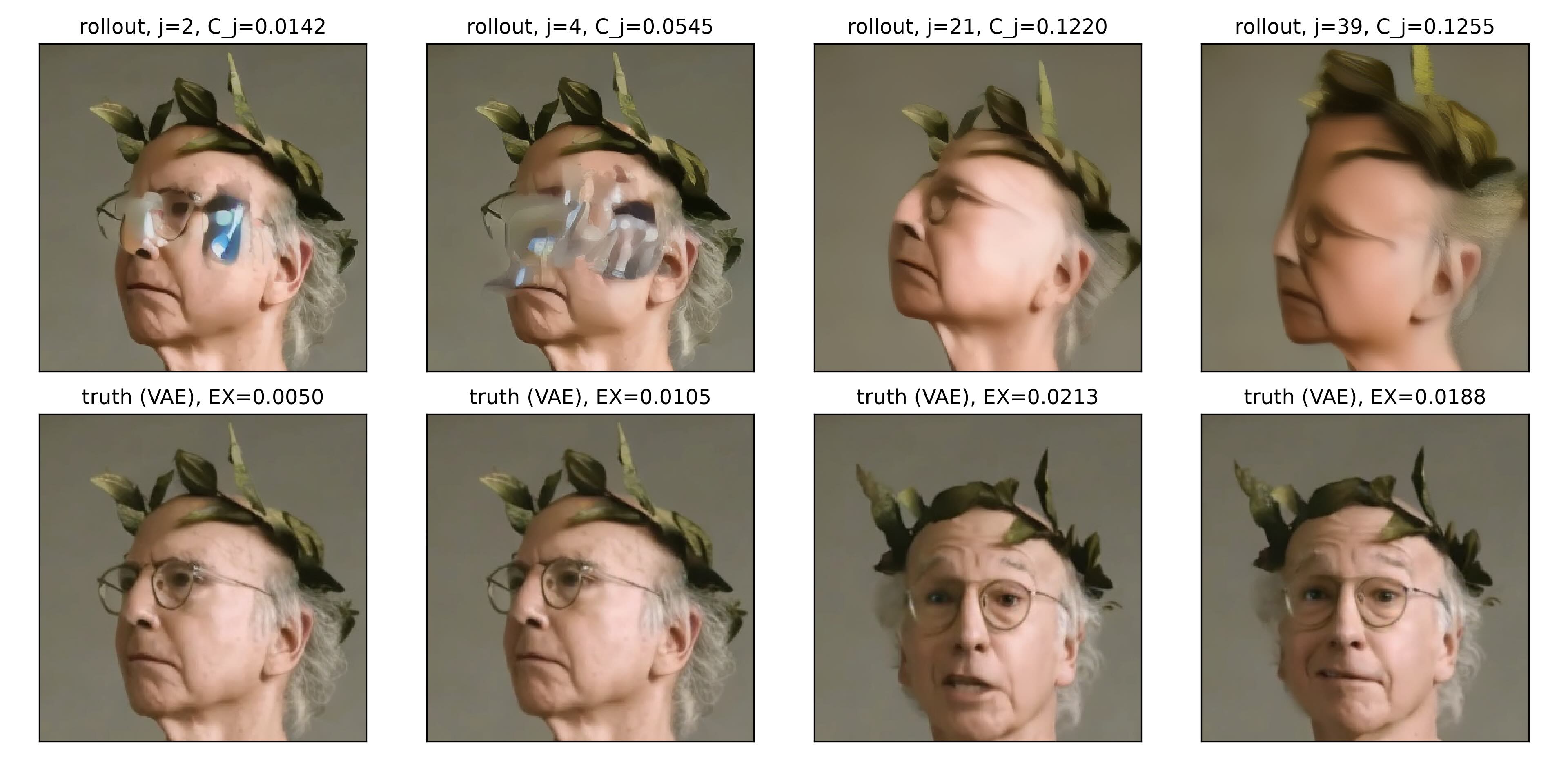}
\includegraphics[width=0.495\textwidth]{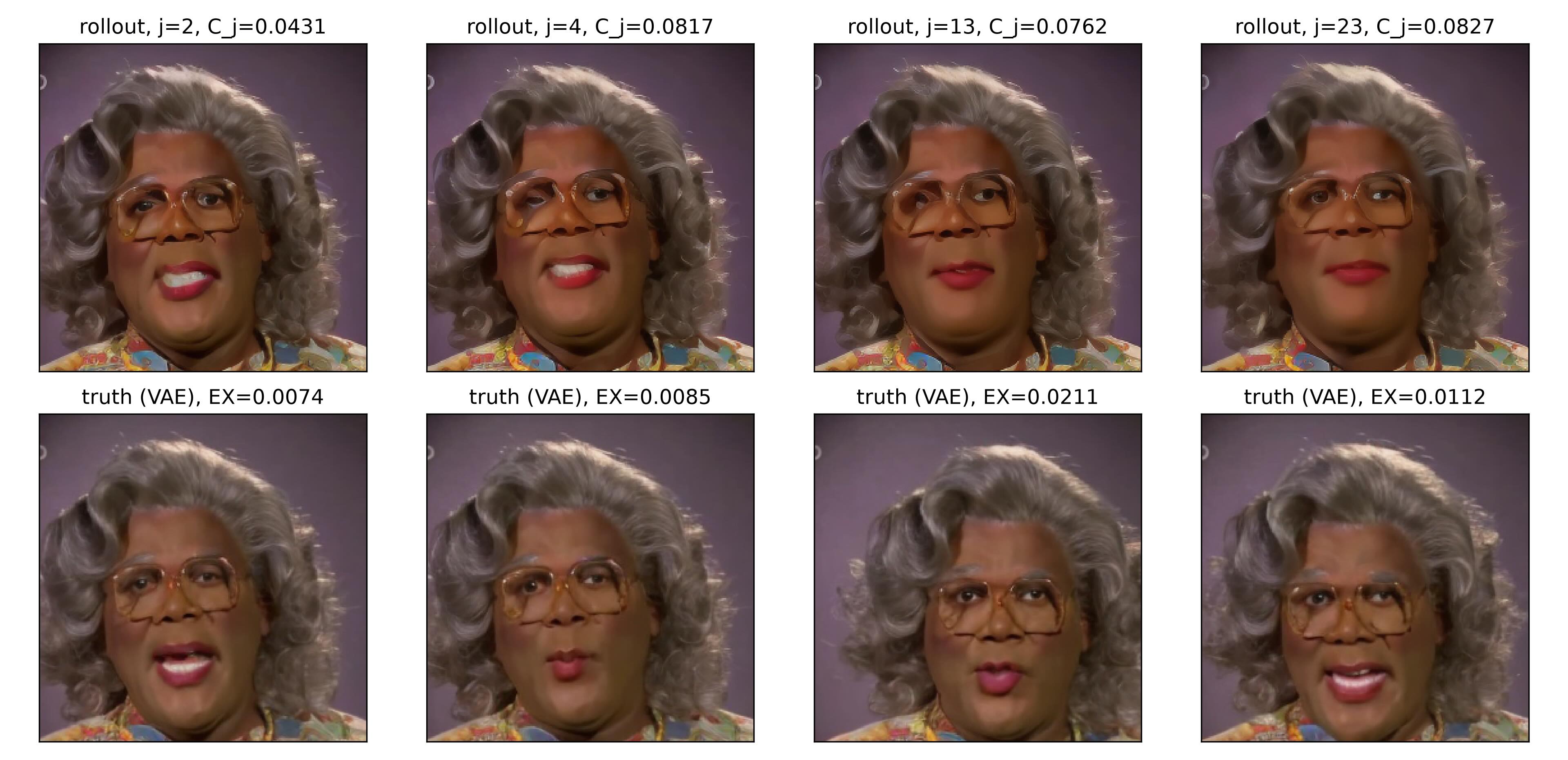}
\includegraphics[width=0.495\textwidth]{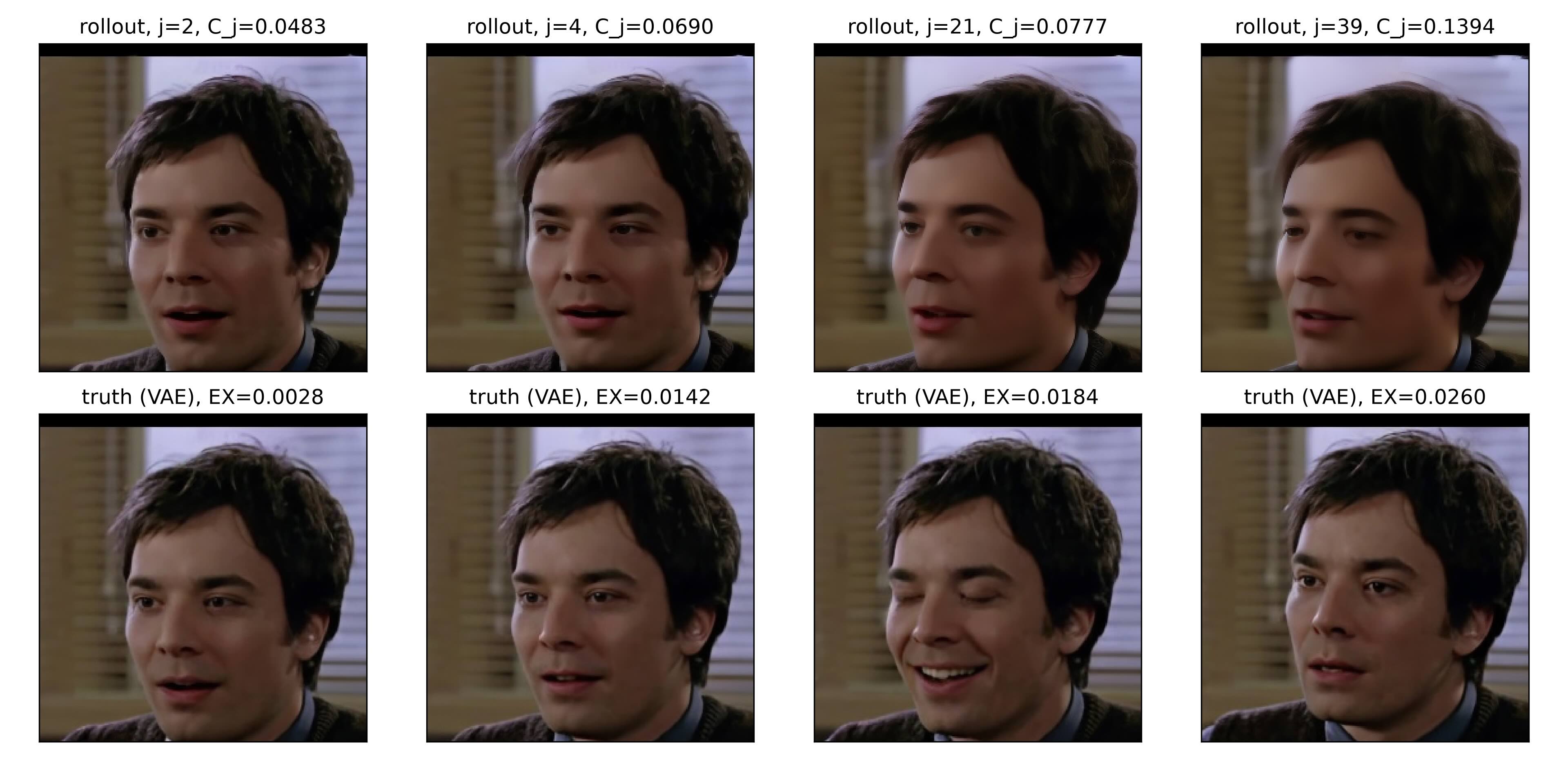}
\caption{CELEBV-HQ rollouts show with error and $\cyc{}$ values. While for such a multi-modal dataset it is impossible for the model to predict rollouts for test data, it is able to flag, in a self-supervised way, OOD data with increased $\cyc{}$ levels.}
\label{fig:CELEBVHQ_Rolls_More}
\end{figure*}

\begin{figure*}[h]
\centering
   \includegraphics[width=1.0\textwidth]{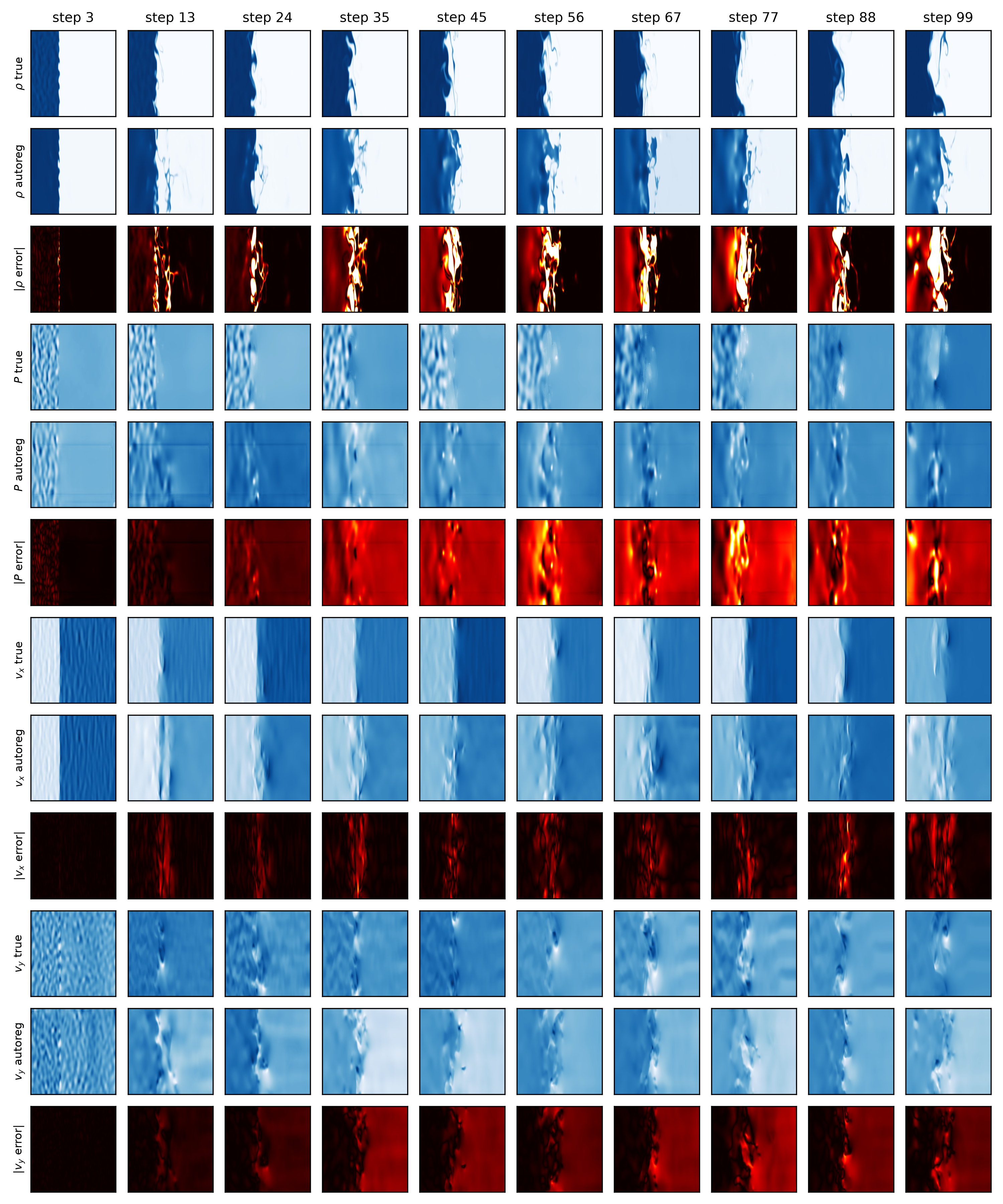}
  \caption{Forward rollout 99 steps for one test trajectory.}
  \label{fig:back_forward_turb_flow}
\end{figure*}

\begin{figure*}[h]
\centering
   \includegraphics[width=1.0\textwidth]{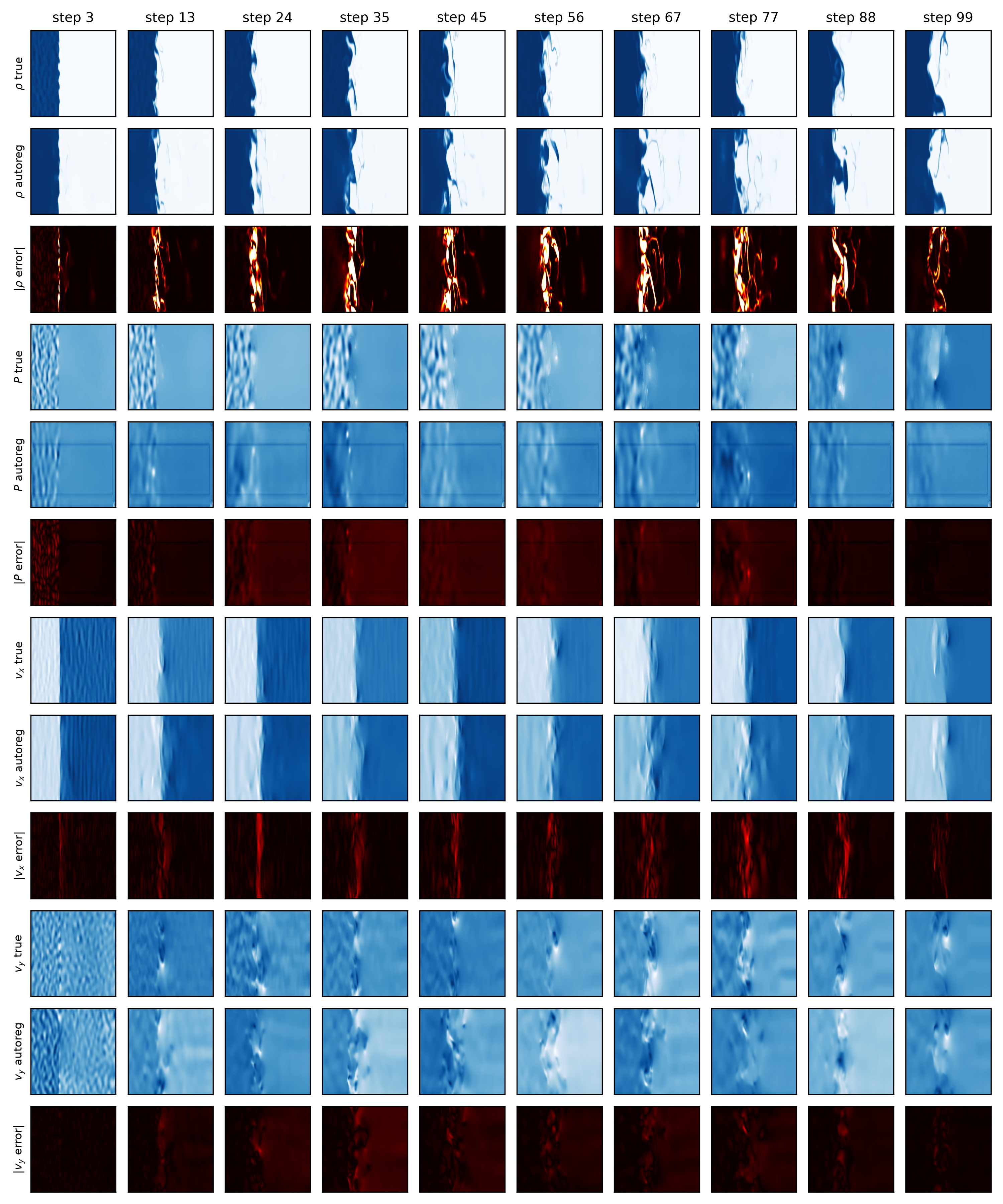}
  \caption{Backward rollout 99 steps for one test trajectory.}
  \label{fig:back_forward_turb_flow}
\end{figure*}

\begin{figure*}[h]
\centering
   \includegraphics[width=0.75\textwidth]{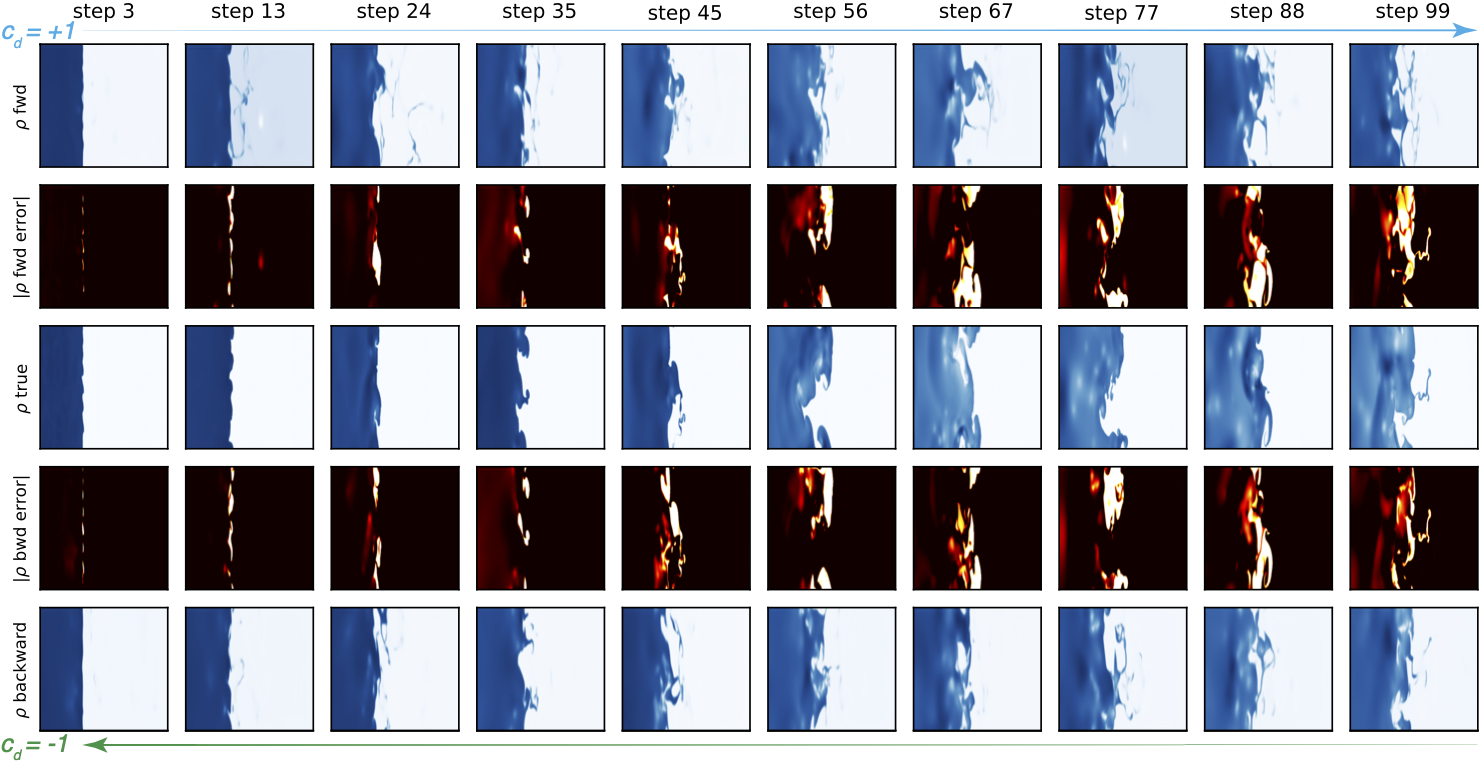}
   \includegraphics[width=0.75\textwidth]{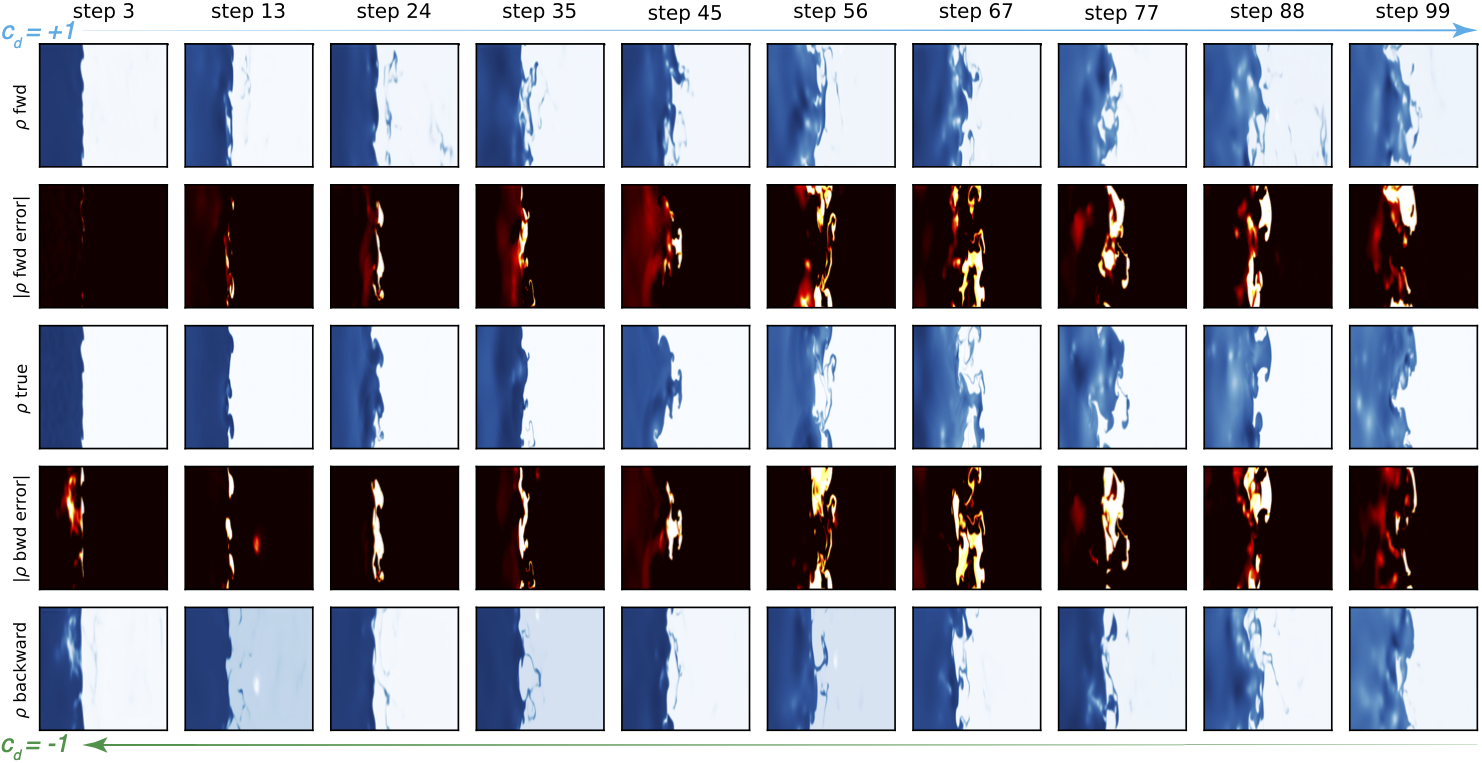}
   \includegraphics[width=0.75\textwidth]{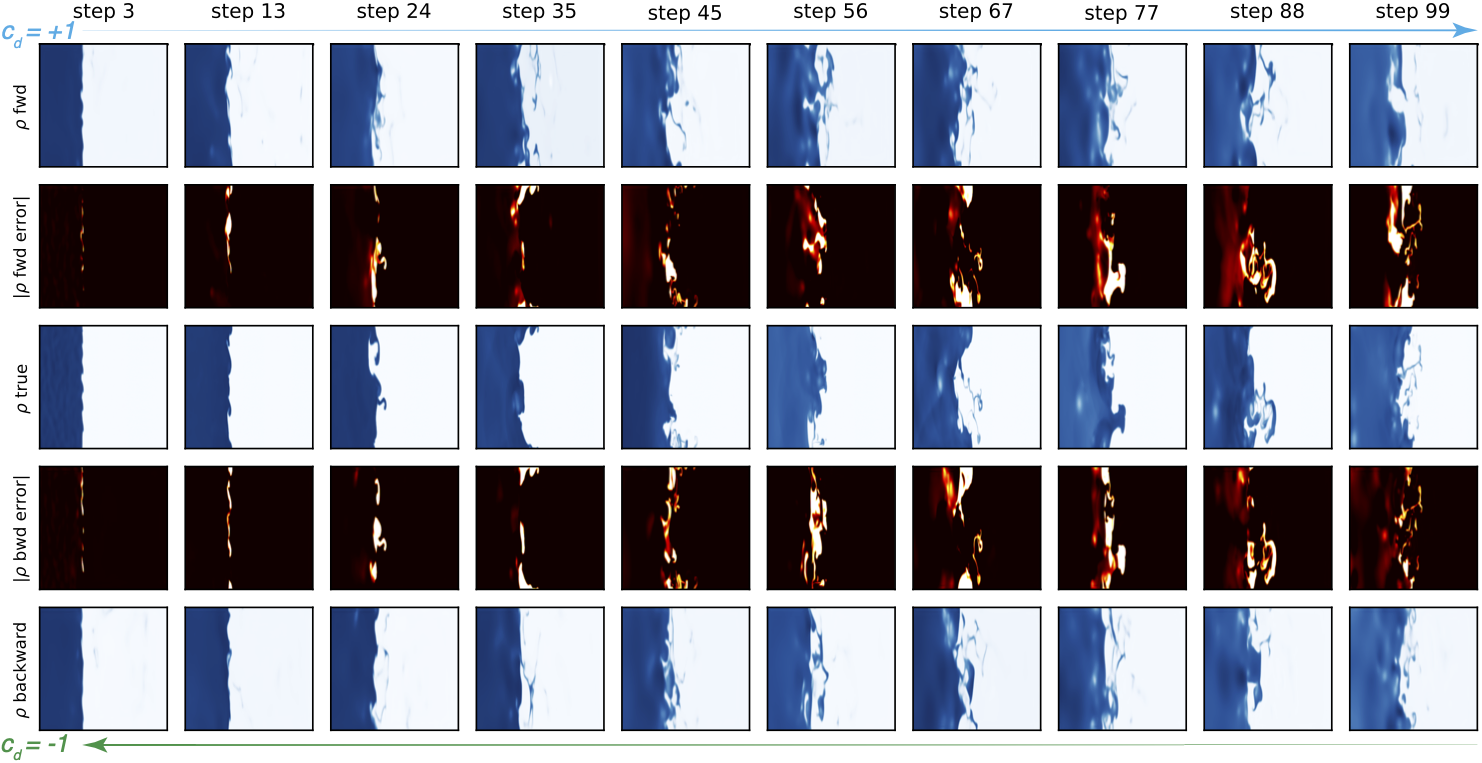}
  \caption{Comparing forward and backward rollout of the denisity fields for 99 steps for test trajectories 1, 2, and 3. For each of the 3 images, the 1st row shows forward autoregressive generation using $c_d=+1$, the 2nd row shows the absolute difference between forward generated density and the true density, the 3rd row shows the true density at each time step, the 4th row shows the absolute difference between backward generated density and the true density, and finally the 5th row shows the backward autoregressive generation using $c_d=-1$.}
  \label{fig:fwd_bck_trb_012}
\end{figure*}

\begin{figure*}[h]
\centering
   \includegraphics[width=0.75\textwidth]{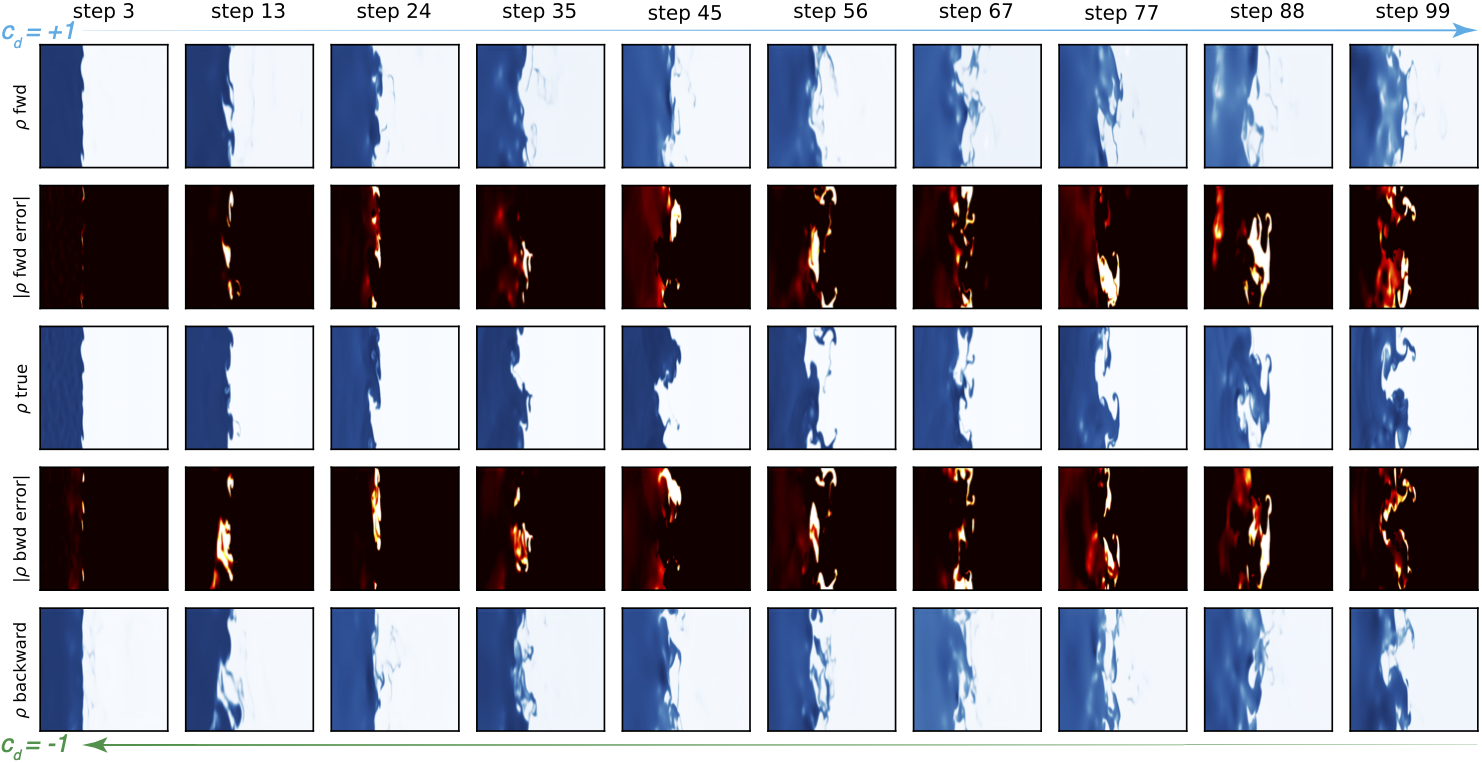}
   \includegraphics[width=0.75\textwidth]{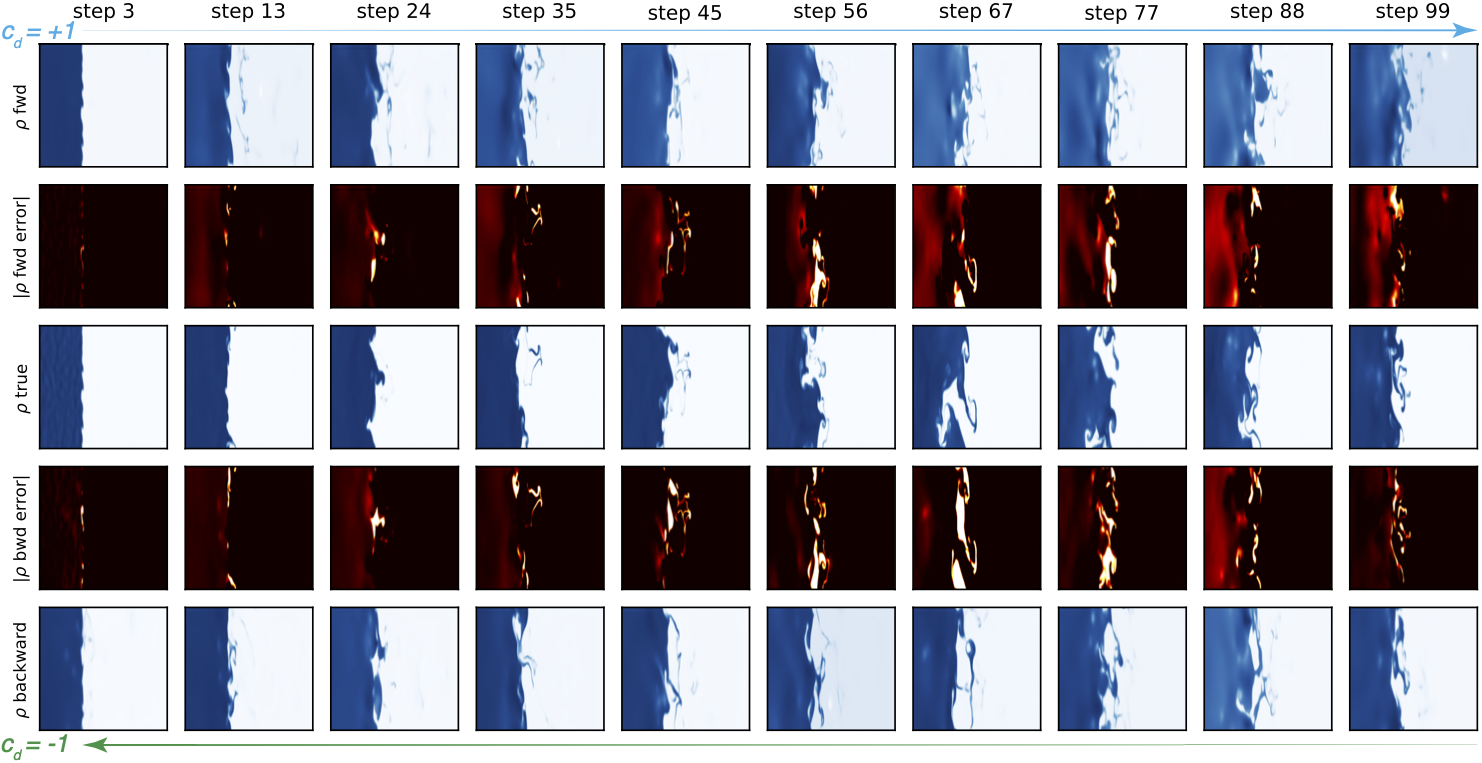}
   \includegraphics[width=0.75\textwidth]{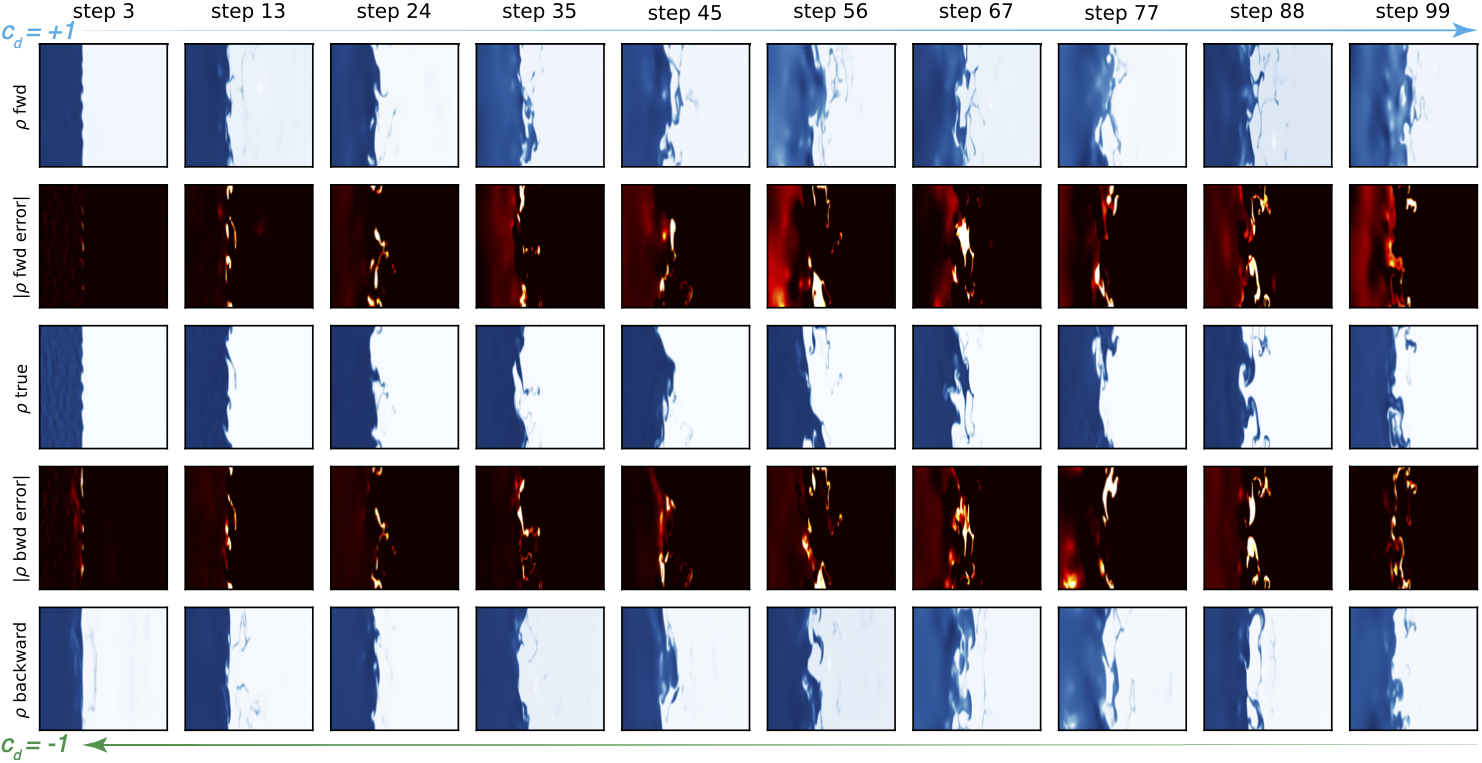}
  \caption{Comparing forward and backward rollout of the denisity fields for 99 steps for test trajectories 4, 5, and 6. For each of the 3 images, the 1st row shows forward autoregressive generation using $c_d=+1$, the 2nd row shows the absolute difference between forward generated density and the true density, the 3rd row shows the true density at each time step, the 4th row shows the absolute difference between backward generated density and the true density, and finally the 5th row shows the backward autoregressive generation using $c_d=-1$.}
  \label{fig:fwd_bck_trb_345}
\end{figure*}

\begin{figure*}[h]
\centering
   \includegraphics[width=0.75\textwidth]{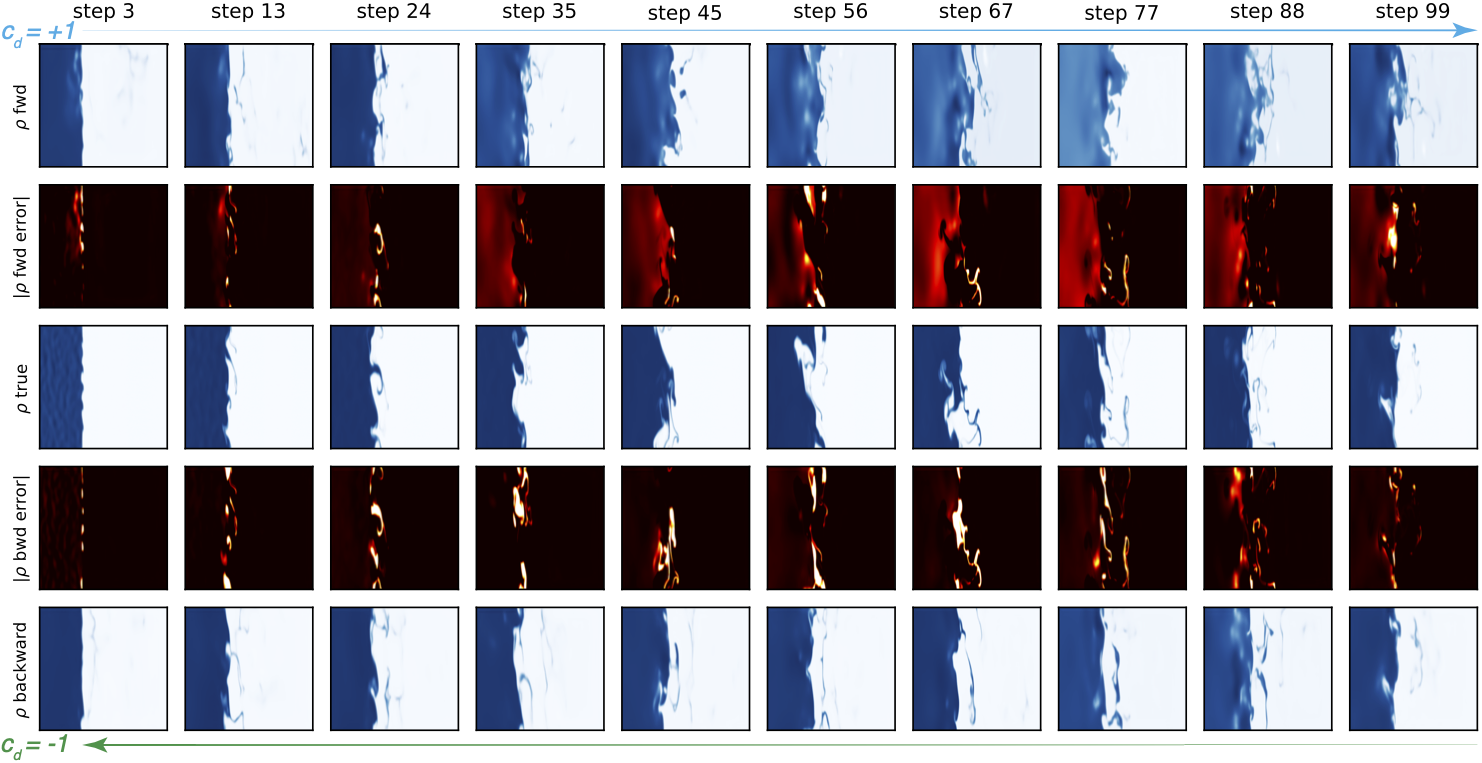}
   \includegraphics[width=0.75\textwidth]{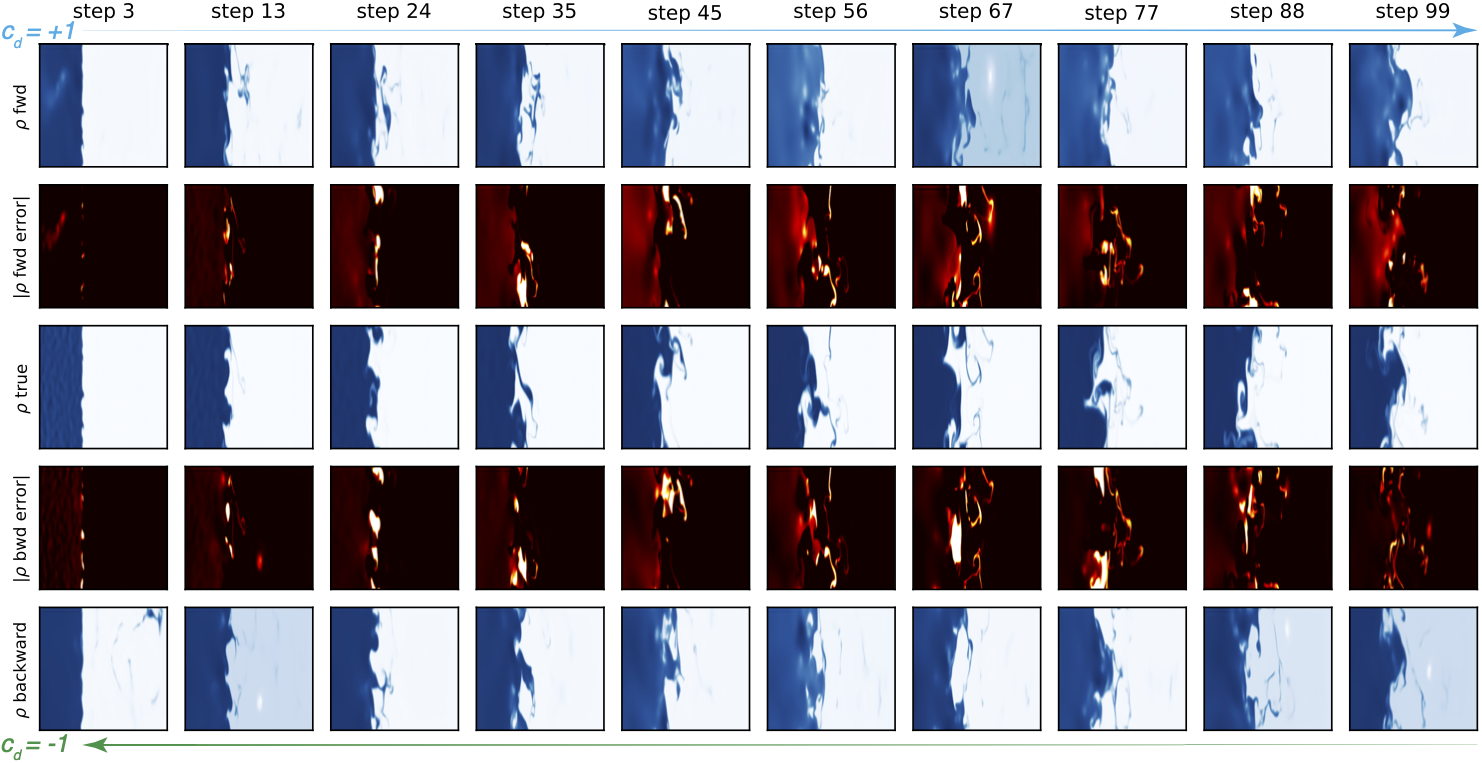}
   \includegraphics[width=0.75\textwidth]{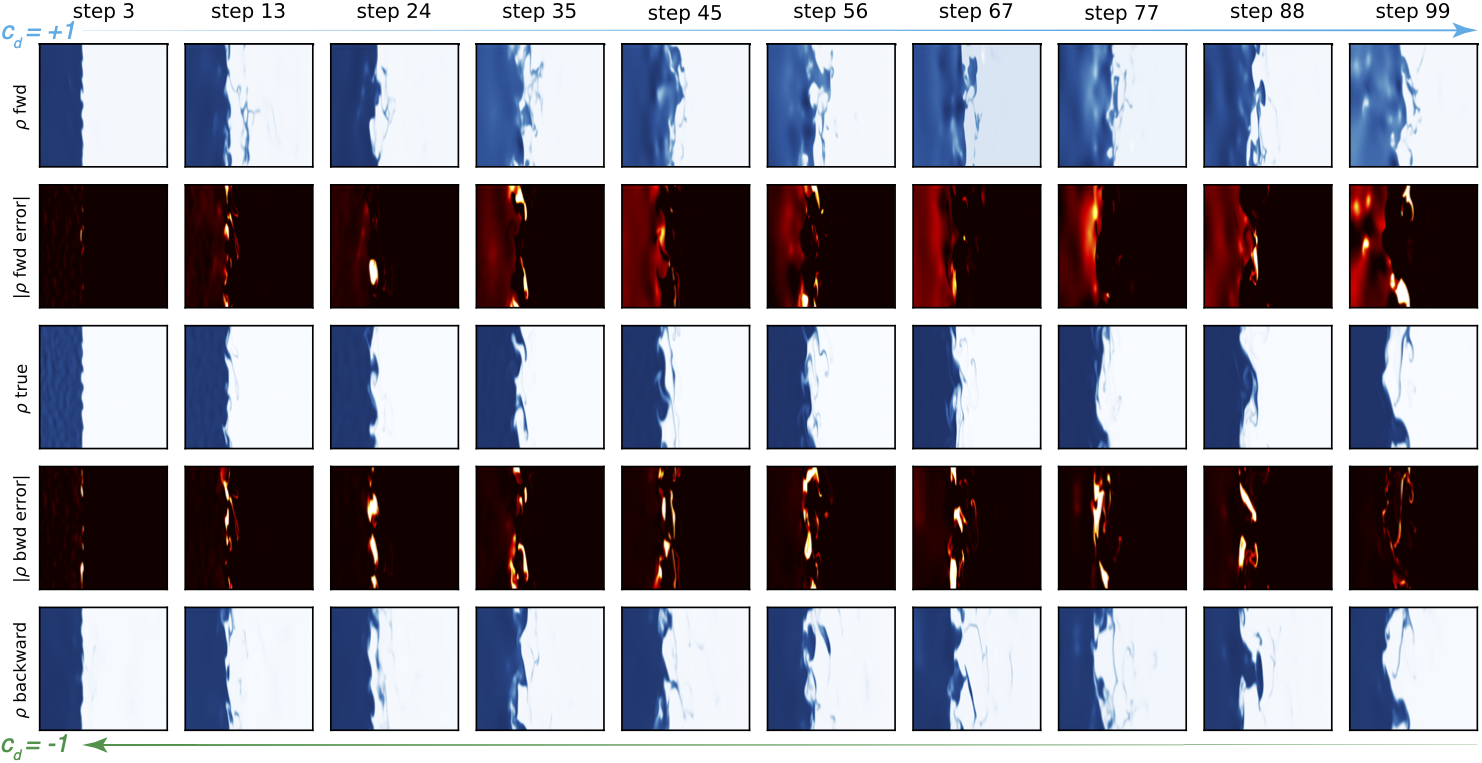}
  \caption{Comparing forward and backward rollout of the denisity fields for 99 steps for test trajectories 7, 8, and 9. For each of the 3 images, the 1st row shows forward autoregressive generation using $c_d=+1$, the 2nd row shows the absolute difference between forward generated density and the true density, the 3rd row shows the true density at each time step, the 4th row shows the absolute difference between backward generated density and the true density, and finally the 5th row shows the backward autoregressive generation using $c_d=-1$.}
  \label{fig:fwd_bck_trb_678}
\end{figure*}

\begin{figure}[t]
\centering
\includegraphics[width=1.0\columnwidth]{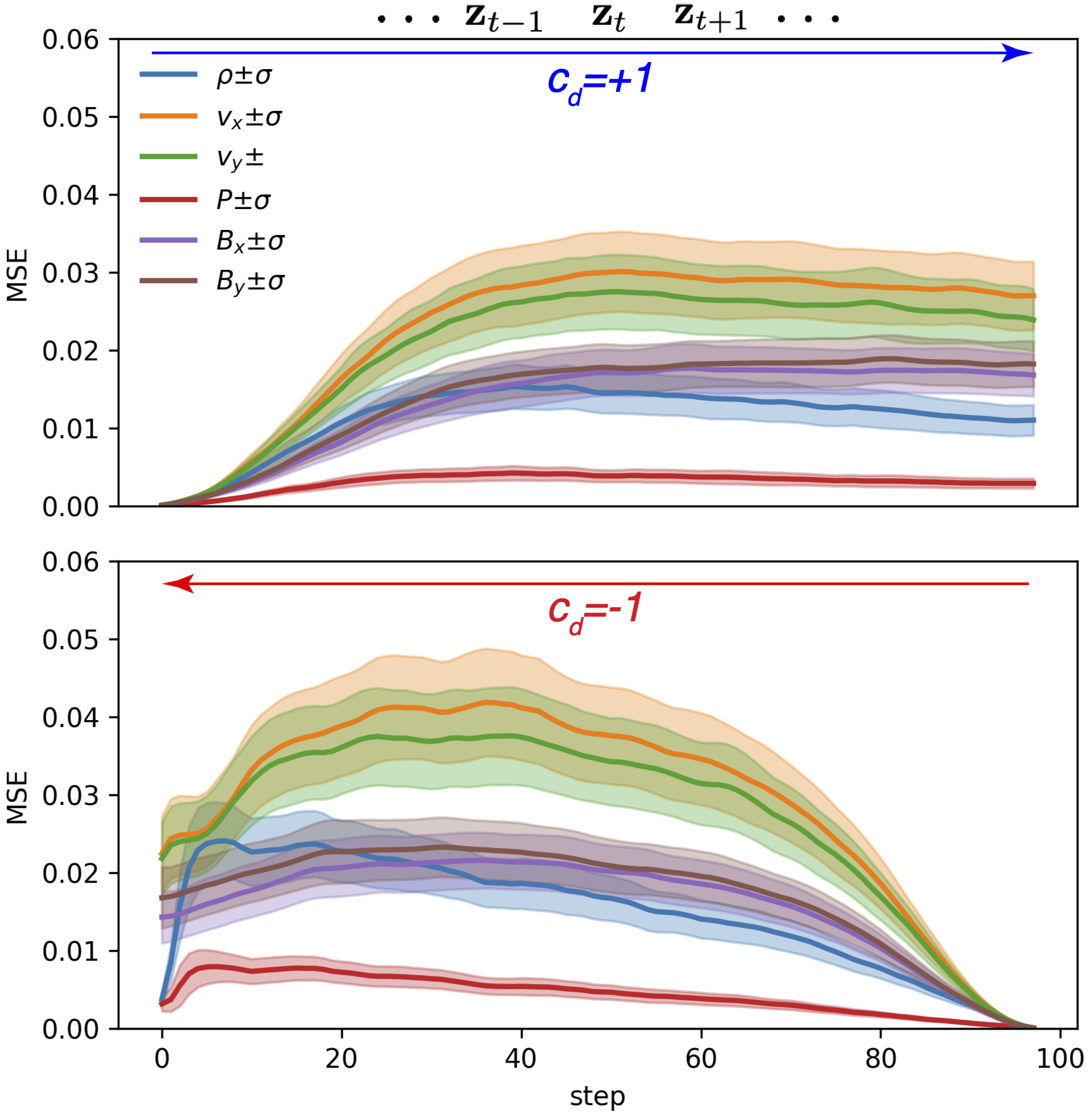}
\caption{MHD forward (top) and backward (bottom) autoregressive rollout error accumulation.}
\label{fig:inv-ci}
\end{figure}

\begin{figure*}[t]
\centering
\includegraphics[width=0.85\textwidth]{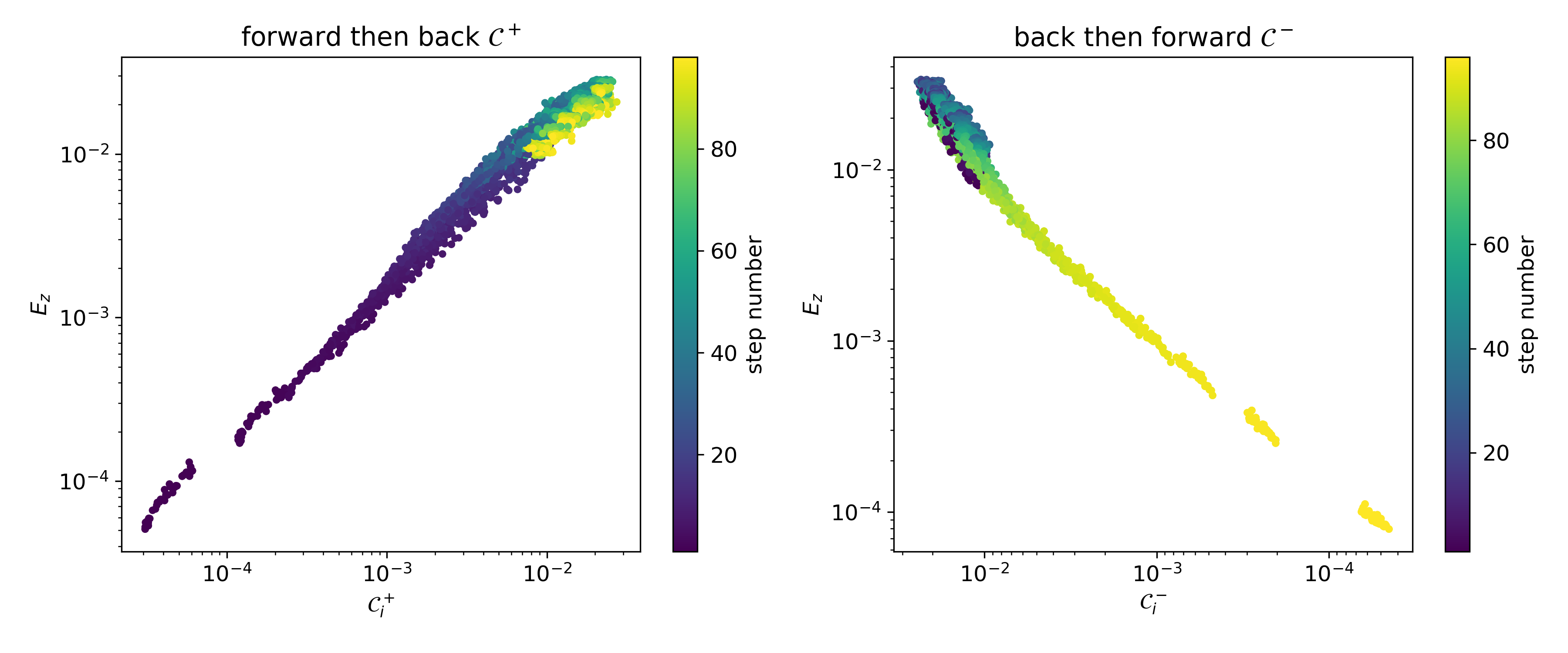}
\caption{\textbf{Round-trip consistency forward or backward in time.}
Forward rollouts cycled backward (the default $\cyc{i}$, left) and
backward rollouts cycled forward (the mirror image $\cyc{i}^{\,-}$,
right), each plotted against the corresponding true rollout error on
the 50 held-out MHD validation trajectories, colored by steo number. The backward cycle tracks backward error when the model is deployed as an inverse solver. The x-axis of the backward rollouts is reversed.}
\label{fig:inv-ci_step}
\end{figure*}

\end{document}